%% file: main_arxiv.tex
\documentclass[journal,twoside,web]{ieeecolor}
\usepackage{amsthm}
\usepackage{generic}
\usepackage{cite}
\usepackage{amsmath,amssymb,amsfonts}
\usepackage{algorithmic}
\usepackage{graphicx}
\usepackage{algorithmic}

\usepackage[urlcolor=blue]{hyperref}

\usepackage{textcomp}
\usepackage[T1]{fontenc}
\usepackage{color}
\usepackage{bm}
\usepackage{algorithm2e}
\usepackage[normalem]{ulem}
\usepackage{xcolor}
\usepackage{xspace}
\usepackage{wrapfig}
\usepackage{booktabs, multirow}
\let\labelindent\relax
\usepackage{enumitem}
\newcommand{\citep}{\cite}
\newcommand{\citet}{\cite}
\usepackage{Definitions}

\allowdisplaybreaks

\newcommand{\AlgName}{{Spectral Dynamics Embedding Control}\xspace}
\newcommand{\algabb}{{SDEC}\xspace}

\newcommand{\revision}[1]{{\color{black} #1}}
\newcommand{\newrevision}[1]{{\color{black} #1}}

\newcounter{relctr} 
\everydisplay\expandafter{\the\everydisplay\setcounter{relctr}{0}} 

\newcommand\labelrel[2]{%
  \begingroup
    \refstepcounter{relctr}%
    \stackrel{\textnormal{(\alph{relctr})}}{\mathstrut{#1}}%
    \originallabel{#2}%
  \endgroup
}
\AtBeginDocument{\let\originallabel\label} 

\def\BibTeX{{\rm B\kern-.05em{\sc i\kern-.025em b}\kern-.08em
    T\kern-.1667em\lower.7ex\hbox{E}\kern-.125emX}}
\begin{document}
\title{Stochastic Nonlinear Control via Finite-dimensional Spectral Dynamics Embedding}
\author{Zhaolin Ren$^{*,1}$, Tongzheng Ren$^{*,2}$, Haitong Ma$^{1}$, Na Li$^{\dagger, 1}$, and Bo Dai$^{\dagger, 3}$
    \thanks{$^*$Equal contributions, $^\dagger$Co-corresponding authors.  $^1$ Z. Ren, H. Ma and N. Li are with Harvard University, Email: {zhaolinren@g.harvard.edu, haitongma@g.harvard.edu, nali@seas.harvard.edu}; $^{1}$T. Ren is with University of Texas, Austin, Email: { tongzheng@utexas.edu}; $^3$B. Dai is with Google DeepMind and Georgia Tech, Email: {bodai@\{google.com, cc.gatech.edu\}}.}
    \thanks{ This work is done when T. Ren was a student researcher at Google DeepMind.
    This work is funded by NSF AI institute: 2112085, NSF ECCS: 2328241, NSF CNS: 2003111, NIH R01LM014465, ONR: N000142512173, NSF ECCS: 2401391, and NSF IIS: 2403240.}
    \\[-8.0ex]
}

\maketitle

\begin{abstract}
    This paper proposes an approach,~\AlgName~(\algabb), to optimal control for nonlinear stochastic systems. This method reveals an infinite-dimensional feature representation induced by the system's nonlinear stochastic dynamics, enabling a linear representation of the state-action value function. For practical implementation, this representation is approximated using finite-dimensional truncations, specifically via two prominent kernel approximation methods: random feature truncation and Nystrom approximation.   To characterize the effectiveness of these  approximations,  we provide an in-depth theoretical analysis to characterize the \emph{approximation error} arising from the finite-dimension truncation and \emph{statistical error} due to finite-sample approximation in both policy evaluation and policy optimization. 
  Empirically, our algorithm performs favorably against existing stochastic control algorithms on several benchmark problems.
    

\end{abstract}


\input{intro}
\input{prelim}
\input{spectral}
\input{analysis}
\input{simulations}
\input{conclusion}
\newpage
\section*{References}
\vspace{-1em}
\bibliography{ref.bib}
\bibliographystyle{IEEEtran}
\newpage
\input{appendix}

\end{document}

%% file: intro.tex
\section{Introduction}\label{sec:intro}
 
Stochastic optimal nonlinear control---namely, the problem of finding an optimal feedback policy to maximize cumulative rewards for a stochastic nonlinear system---has been a long-standing challenging problem in control literature~\cite{khalil2015nonlinear,fleming2012deterministic}.  
Various control techniques have been developed for nonlinear control, including gain scheduling~\cite{rugh2000research}, feedback linearization~\cite{charlet1989dynamic}, iterative linear-quadratic regulator~\citep{sideris2005efficient}, sliding mode control~\cite{edwards1998sliding}, geometric control~\cite{brockett2014early}, back stepping~\cite{kokotovic1992joy}, control Lyapunov functions~\cite{primbs1999nonlinear}, model-predictive control \cite{rawlings2017model}, and tools that leverage inequality approximation and optimization-based methods like sum-of-squares (SOS) programming~\cite{prajna2005sostools}. Nonlinear control often focuses on the stability of closed-loop systems, while control optimality analysis is often heuristic or limited to special classes of systems. Moreover, these methods either lead to highly suboptimal solutions, can only be applied to a subclass of nonlinear systems that satisfy special conditions, or require a large amount of computation and thus cannot  handle large-size problems.  
 
To take advantage of the rich theory and tools developed for linear systems, kernel-based linearization has recently regained attraction and led to a few representative approaches. 
For instance, the Koopman operator~\cite{korda2018linear,kutz2016dynamic}
lifts states into an infinite-dimensional space of known measurement functions, 
where the dynamics become linear in the new space. However, depending on the system itself, possible coupling nonlinear terms between state and control would still appear.  
Alternatively, KNR~\cite{kakade2020information} studies the \textit{sample} complexity of learning control for an unknown nonlinear system whose dynamics are assumed to lie in a \emph{pre-given} RKHS whose dimension can be infinite, and hence are \emph{linear} with respect to the feature maps for states and actions but neglect the \textit{computation} challenges induced by the infinite dimension and the nonlinear features.
\cite{grunewalder2012modelling} considers Markov decision processes (MDP) and represents the conditional transition probability of a MDP in a \emph{pre-defined} RKHS so that calculations involved in solving the MDP could be done via inner products in the infinite-dimensional RKHS.
\cite{kamalapurkar2021occupation} introduces the Liouville’s equation in occupation kernel space to represent the trajectories, with which finding the optimal value can be reformulated as linear programming in the infinite-dimensional space under some strict assumptions.

Although kernelized linearization has brought a promising new perspective to nonlinear control, these representative approaches fall short in both computational and theoretical respects. 
\emph{Computationally}, control in an infinite-dimensional space is intractable. Hence, a finite-dimensional approximation is necessary. 
Though data-driven computational procedures for kernel selection and RKHS reparametrization have been proposed in these existing methods,
they are inefficient in the sense that {\bf i)}, the dynamics or trajectories are \emph{presumed} to be lying in some RKHS with a \emph{pre-defined} finitely-approximated kernel, which is a very strict assumption;  in fact, 
finding good kernel representations for the dynamics itself is a challenging task; and {\bf ii)}, the dynamics information for kernelization is \emph{only} exploited through samples, and other structure information in the dynamics is ignored even when the dynamics formula is known explicitly. 
Meanwhile, \emph{theoretically}, the optimality of control with finite-dimensional approximations---\ie, the policy value gap between the finite-dimensional approximation and optimal policy in an infinite-dimensional RKHS---
has largely been ignored and not been rigorously analyzed.

Recently,~\cite{ren2022free} provided a novel kernel linearization method, \emph{spectral dynamics embedding}, by establishing the connection between stochastic nonlinear control models and linear Markov Decision Processes (MDPs), which exploits the random noise property to factorize the transition probability operator, and induce an infinite-dimensional space for \emph{linearly} representing the \textit{state-action value function} for \emph{arbitrary} policy.
Spectral dynamics embedding bypasses the drawbacks of the existing kernel linearization methods in the sense that {\bf i)} the kernel is \textit{automatically induced} by the system dynamics, which avoids the difficulty in deciding the kernel features, \revision{utilizes the knowledge of the dynamics and noise model} and eliminates the modeling approximation induced by a predefined kernel; and {\bf ii)} the kernel linearization and its finite feature approximation is \emph{computational-friendly} with well-studied performance guarantees~\citep{rahimi2008weighted,hayakawa2023sampling}.
The superiority of spectral dynamics embedding has been justified empirically in the reinforcement learning setting~\cite{ren2022free} where system dynamics are unknown. However, there still lacks any end-to-end control algorithm with theoretical guarantees that utilizes the finite-dimensional approximation of spectral dynamics embedding, 
motivating the present work. 

\noindent\textbf{Our contributions.} Building on the spectral dynamics embedding framework discussed above, we develop an end-to-end stochastic nonlinear control algorithm with theoretical guarantees, filling a key gap in prior work. Spectral dynamics embedding provides a powerful kernel-based linearization of nonlinear stochastic systems by factorizing the transition dynamics into an infinite-dimensional feature space that linearly represents the state-action value function for any policy. Our main contributions are as follows:

\begin{itemize}[leftmargin=*]
    \item We design a computationally tractable control algorithm, \emph{\AlgName~(\algabb)}, based on a finite-dimensional truncation of the spectral dynamics embedding. Specifically, we approximate the infinite-dimensional embedding using two practical methods—Monte Carlo random features and Nystr\"om features—as described in~\secref{subsec:spectral_feature}. We then perform value function estimation using least-squares policy evaluation, and improve the policy using natural policy gradients. The full algorithm is presented in~\secref{subsec:finite_approx}.

    \item Our formulation naturally integrates with modern deep RL pipelines such as Soft Actor-Critic~\citep{haarnoja2018soft} by replacing the critic network with our dynamics-informed representation. A key novelty of \algabb\ is that it leverages known system dynamics to construct a principled, task-adaptive feature space for policy evaluation and improvement—offering a bridge between classical control and learning-based methods.

    \item We provide rigorous theoretical analysis of the \textit{approximation error} introduced by finite-dimensional truncation and the \textit{statistical error} due to finite sample estimation. To our knowledge, this is the first end-to-end analysis of such errors in the context of nonlinear stochastic control. We show that the performance gap between the learned and optimal policies decreases polynomially with the number of features and the sample size.

    \item We validate our approach on several nonlinear control tasks, including 2D Drones, Pendubot balancing, and Cartpole (\secref{sec:simulations}), and benchmark against several nonlinear control methods  and state-of-the-art RL algorithms. Our method achieves strong empirical performance, consistent with the theoretical analysis.
    \end{itemize}

We would like to note that compared to our earlier conference paper~\cite{ren2023stochastic}, this work substantially extends with several new contributions: (i) expanded theoretical exposition and proofs, (ii) incorporation and analysis of Nystr\"om feature approximations, including a novel high-probability error bound (\lemref{prop:kernel_approx_main_paper}), (iii) formal justification of key technical assumptions (Assumptions 2 and 3), and (iv) a significantly enhanced simulation section covering more benchmarks and stronger baselines.


We have also developed an open-source toolbox based on the methods proposed in this paper, available at \href{https://github.com/CoNG-harvard/repr\_control}{https://github.com/CoNG-harvard/repr\_control}. The toolbox enables users to define their own system dynamics, with step-by-step documentation provided to guide the implementation of custom models.

\vspace{-1em}
\subsection{Related Work in Reinforcement Learning (RL)} 


Presumably, recently developed model-free deep RL methods for unknown systems, such as ~\citep{schulman2017proximal} and \citep{haarnoja2018soft}, could be also applied to optimal nonlinear control by treating the nonlinear dynamics as a generative or simulation model to generate data. The success of deep RL is largely attributed to the flexibility of deep neural networks for policy or value function modeling. However, such flexibility also brings difficulty in control by making the optimization extremely non-convex, and thus, may lead to sub-optimal policy, wasting the modeling power. Our proposed \algabb shares some similar aspects of actor-critic family of algorithms in RL, in the sense that~\algabb also exploits policy improvement (actor) with state-action value function estimation (critic). However, we emphasize the major difference is that \algabb takes advantage of the spectral dynamics embedding for both state-value function and policy, instead of an arbitrary network \revision{parameterization}. 
This not only reduces sample complexity by exploiting the known nonlinear dynamics structure but also bypasses the non-convexity difficulty in optimization, allowing us to provide rigorous theoretical guarantees with only mild assumptions. 
\revision{We note that while our work builds on the concept of linear MDPs~\cite{jin2020provably}, such linear MDPs merely refer to MDPs whose transition admits a factorization to linearly representing value functions, which is distinct from the linear programming formulation of solving MDPs~\cite{hernandez1998approximation}. }

%% file: prelim.tex
\section{Problem Setup and Preliminaries}\label{sec:prelim}



In this section, we introduce the  {stochastic nonlinear control} problem that will be studied in this paper and reformulate it as a MDP. We will also  briefly introduce the background knowledge about positive-definite kernels. 

\vspace{-1.2em}
\subsection{Stochastic Nonlinear Control Problem in MDPs}\label{subsec:control_mdp}


We consider the standard discrete-time nonlinear control model with $\gamma$-discounted infinite horizon, defined by
\begin{align}\label{eq:nonlinear_control}
    s_{t+1} = f(s_t, a_t) + \epsilon_t, \quad\text{where}\quad \epsilon_t \sim \mathcal{N}(0, \sigma^2 I_d),
\end{align}
such that $\gamma\in (0, 1)$, $s\in \mathcal{S}=\mathbb{R}^d$ is the state, $a\in\mathcal{A}$ is the control action,
and $\cbr{\epsilon_t}_{t=1}^\infty$ are independent Gaussian noises.
The function $f\rbr{\cdot, \cdot}: \Scal\times \Acal\rightarrow \Scal$ describes the general nonlinear dynamics, and $r:\mathcal{S}\times \mathcal{A} \to \mathbb{R}$ gives the reward function on the state and action pair. 
Without loss of generality, we assume there is a fixed initial state $s_0$. Given a stationary policy $\pi: \mathcal{S} \to \Delta(\mathcal{A})$ with $\Delta(\Acal)$ as the space of probability measures over $\Acal$, the accumulated reward over infinite horizon is given by 
\vspace{-2mm}
\small
\begin{align}
    J^\pi = \revision{\mathbb{E}_{P,\pi}\left[\sum_{t=0}^{\infty} \gamma^t r(s_{t}, a_{t})\right]},
    \label{eq:J_def}
\end{align}
\normalsize
where the expectation is w.r.t. the \revision{stochastic dynamics $P\rbr{s(t+1) \mid s(t), a(t)}$ and the (possibly random) policy $\pi$ which we use to choose the actions $a(t) \sim \pi\rbr{\cdot \mid s(t)}$}. In this paper, we study the optimal \textit{control/planning} problem which is to seek a policy $\pi^*$ that maximizes~\eqref{eq:J_def}, given the dynamics $f$ and the reward function $r$. Note that the nonlinearity of $f$ and $r$, as well as the stochasticity from $\epsilon$ makes this optimal control problem difficult as reviewed in the introduction. 

The above stochastic nonlinear optimal control problem can also be described via Markov Decision Process. Consider an episodic homogeneous-MDP, denoted by $\mathcal{M} = \langle \mathcal{S}, \mathcal{A}, P, r, \gamma\rangle $, where
$P\rbr{\cdot|s, a}:\mathcal{S} \times \mathcal{A} \to \Delta(\mathcal{S})$ describes the state transition distribution, where $\Delta(\mathcal{S})$ denotes the space of probability measures on the set $\mathcal{S}$. Then, the stochastic nonlinear control model~\eqref{eq:nonlinear_control} can be recast as an MDP with transition dynamics
\small
\begin{align}\label{eq:transition}
    P(s^\prime|s, a) & \propto \exp\left(-\frac{\|f(s, a) - s^\prime\|_2^2}{2\sigma^2}\right).
\end{align}
\normalsize
Meanwhile, given a policy $\pi:\mathcal{S} \to \Delta(\mathcal{A})$, the corresponding $Q^\pi$-function is given by
\small
\begin{align}
\textstyle
    Q^\pi(s, a) = \mathbb{E}_{P, \pi}\left[\sum_{t = 0}^{\infty} \gamma^t r(s_{t}, a_{t})\bigg|s_0 = s, a_0 = a\right].
    \label{eq:q_def}
\end{align}
\normalsize
It is straightforward to show the Bellman recursion for $Q^\pi$ or optimal $Q^*$-function, respectively, 
\small
\begin{align}
    Q^\pi(s, a) = & r(s, a) + \gamma\mathbb{E}_{P} \sbr{V^\pi(s^\prime)}\\
    Q^*(s, a) = & r(s, a) + \gamma\mathbb{E}_{P} \sbr{\max_{a^\prime\in\Acal}Q^*(s^\prime, a^\prime)},
    \label{eq:bellman}
\end{align}
\normalsize
with $ V^\pi(s) := \mathbb{E}_{\pi}\left[Q^\pi(s, a)\right]$. Equivalently, the accumulated reward $J^\pi$ is $V^\pi(s_0)$. 
The goal can be reformulated as seeking the optimal policy $\pi^* = \argmax_{\pi}\,V^\pi\rbr{s_0}$. For an MDP with finite states and actions, optimal control can be obtained by solving dynamic program for the tabular $Q$-function via the Bellman relation~\eqref{eq:bellman}.  
However, for continuous states and actions, representing the $Q$-function and conducting dynamic programming (DP) in function space becomes a major issue. In this paper, we introduce spectral dynamics embedding, a novel kernelized linearization, to linearly represent the $Q$-function, and develop a practical and provable method to learn and approximate the optimal control policy. 

 \vspace{-1em}
\subsection{Positive definite kernels, its decompositions, and connection to RKHS}\label{subsec:rkhs}

In this section, we will introduce the necessary background to understanding our proposed kernel linearization. We first define what a positive definite (PD) kernel is, and introduce two decompositions of PD kernels.
\begin{definition}[(Positive-Definite) Kernel \cite{fasshauer2011positive}]\label{def:kernel}

A symmetric function $k:\mathcal{X} \times \mathcal{X} \to \mathbb{R}$ is a positive-definite kernel on the non-empty set $\mathcal{X}$ if $\;\forall n\geq 1$, $\;\forall \{a_i\}_{i\in [n]} \subset \mathbb{R}$ and mutually distinct sets $\{x_i\}_{i\in [n]} \subset \mathcal{X}$,
$\sum_{i\in[n]}\sum_{j\in[n]} a_i a_j k(x_i, x_j) \geq 0$, where the inequality is strict unless $\{a_i\}_{i \in [n]}$'s are all zero. 
\end{definition}
A bounded, continuous PD kernel $k(\cdot, \cdot)$ admits two important decompositions, namely the Bochner \cite{devinatz1953integral} and the Mercer~\cite{mercer1909functions} decompositions. We begin with the Bochner decomposition.

\begin{lemma}[Bochner~\cite{devinatz1953integral}]\label{thm:bochner}
If $k(x, x^\prime) = k(x - x')$ is a positive definite kernel, then there exists a set $\Omega$,  a measure $\rho(\omega)$ on $\Omega$, and Fourier random feature $\psi_\omega\rbr{x}:\Xcal \rightarrow \CC$ such that 
\small
\begin{equation}\label{eq:bochner}
    k(x, x^\prime) = \EE_{\rho\rbr{\omega}}\sbr{ \psi_\omega\rbr{x}\psi_\omega^*\rbr{x^\prime}}.
\end{equation}
\end{lemma}
\normalsize

We note that the $\psi_\omega\rbr{\cdot}$ in Bochner decomposition may not be unique.  The Bochner decomposition provides the random feature~\cite{rahimi2007random}, which applies a Monte-Carlo approximation for~\eqref{eq:bochner}, leading to a finite-dimension approximation for kernel methods.
For example, for the Gaussian kernel, $k(x, x^\prime) = \exp\rbr{-\frac{\left\|x - x^\prime\right\|^2}{2\sigma^2}}$, the corresponding $\rho(\omega)$ is a Gaussian proportional to $\exp\rbr{-\frac{\sigma^2\nbr{\omega}^2}{2}}$ with $\psi_\omega\rbr{x} = \exp\rbr{-\ib\omega^\top x}$ where $\ib$ is the imaginary unit; for the Laplace kernel, the $\rho\rbr{\omega}$ is a Cauchy distribution with the same $\psi_\omega(\cdot)$. More examples can be found in Table 1 of~\cite{dai2014scalable}. An alternative but also useful decomposition for positive definite kernels is Mercer's decomposition, stated below\revision{\footnote{\revision{We note that under the conditions set forth in Lemma \ref{thm:mercer} (especially the finiteness of the measure $\mu$), it can be shown using \cite{sun2005mercer} that the Mercer expansion also holds for unbounded $\mathcal{X} \subset \mathbb{R}^d$.}}}.
\begin{lemma}[Mercer\cite{mercer1909functions,sun2005mercer}]\label{thm:mercer}
Let $\mathcal{X}\subset \mathbb{R}^d$, $\mu$ be a strictly positive \revision{finite} Borel measure on $\mathcal{X}$, and $k(x, x')$ a bounded continuous positive definite kernel. Then there exists a countable orthonormal basis $\{e_i\}_{i=1}^{\infty}$ of $\mathcal{L}_2(\mu)$ 
with corresponding eigenvalues $\{\sigma_i\}_{i=1}^{\infty}$,\footnote{Given a probability measure $\mu$ defined on $\mathcal{X} \subseteq \mathbb{R}^d$, we denote $L_2(\mu)$ as the set of functions $f:\mathcal{X} \to \mathbb{R}$ such that $\int_{\mathcal{X}} f^2(x) d\mu(x) < \infty.$} such that 
\small
\begin{align}\label{eq:kernel_spectral}
    k(x, x^\prime) = \sum_{i=1}^{\infty} \sigma_i e_i(x) e_i(x^\prime),
\end{align}
\normalsize
where the convergence is absolute and uniform
for all $(x, x^\prime) \in \mathcal{X} \times \mathcal{X}$. We assume $\sigma_1 \geq \sigma_2 \geq \cdots > 0.$
\end{lemma}
While analytical formulas of the Mercer decomposition exist for special kernels such as the Gaussian kernel \cite{fasshauer2011positive}, for general kernels they can be difficult to obtain analytically. Thus, it is common in the literature to approximate the Mercer features numerically, \eg, Nystr\"om method \cite{williams2001using}. In addition, there is a one-to-one correspondence between PD kernels and Reproducing Kernel Hilbert Spaces (RKHSs). 
We first formally define the RKHS.
\begin{definition}[Reproducing Kernel Hilbert Space \cite{aronszajn1950theory}]
\vspace{-0.5em}
\label{def:RKHS}
The Hilbert space $\mathcal{H}$ of an $\;\mathbb{R}$-valued function\footnote{A Hilbert space is an inner product space that is also a complete metric space with respect to the distance induced by its inner product $\langle \cdot,\rangle_{\mathcal{H}}$.} defined on a non-empty set $\mathcal{X}$ is said to be a reproducing kernel Hilbert space (RKHS) if for all $x \in \mathcal{X}$, the (linear) evaluation functional given by $L_x:f \to f(x)$ for any $f \in \mathcal{H}$ is continuous.
\end{definition}
\vspace{-1mm}
Next, we state the following classical result, which states a one-to-one correspondence between RKHS and PD kernels.
\vspace{-2mm}
\begin{lemma}[RKHS and PD kernels\cite{aronszajn1950theory}]
    \label{thm:RKHS_reproducing_kernel_moore_aronszajn}
    Suppose $\mathcal{H}$ is a RKHS of $\{f: \mathcal{X} \to \mathbb{R}\}$. Then, there exists a unique positive definite kernel $k: \mathcal{X} \times \mathcal{X} \to \mathbb{R}$ where $k(x,\cdot) \in \mathcal{H}$ for every $x \in \mathcal{X}$, that satisfies the \emph{reproducing} property, i.e. for every $f \in \mathcal{H}$, $f(x) = \langle f, k(x,\cdot)\rangle_{\mathcal{H}}$, where $\langle \cdot,\rangle_{\mathcal{H}}$ denotes the inner product in $\mathcal{H}$. Conversely, for any positive definite kernel $k :\mathcal{X} \times \mathcal{X} \to \mathbb{R}$, there exists a unique RKHS $\mathcal{H}_k$ of functions $f: \mathcal{X} \to \mathbb{R}$ for which $k$ is the reproducing kernel. Moreover, given any $x,y \in \mathcal{X}$, $\langle k(x,\cdot), k(y,\cdot) \rangle_{\mathcal{H}_k} = k(x,y)$. 
\end{lemma}
\revision{We emphasize that any inner product, unless taking the form $\langle \cdot, \cdot \rangle_{\Hcal_k}$ for some RKHS $\Hcal_k$, does not refer to an RKHS inner product. Indeed, we only need the concept of RKHS when proving the approximation error for Nystrom features.}

%% file: spectral.tex
\vspace{-2mm}
\section{Control with spectral dynamics Embedding}\label{sec:sdec}
\vspace{-1mm}

In this section, we first introduce spectral dynamics embedding~\cite{ren2022free}, a novel kernelized linearization, by which the $Q$-function for arbitrary policy can be represented \emph{linearly}, therefore, the policy evaluation can be conducted within the linear space. This is significantly different from existing kernel linearization methods, which are designed for linearizing the dynamic model. 
We then propose the corresponding finite-dimensional approximated linear space for tractable optimal control. Specifically, we execute dynamic programming in the linear function space constructed by the spectral dynamics embedding for policy evaluation, upon which we improve the policy with natural policy gradient. This leads to~\emph{\AlgName~(\algabb)} in~\algtabref{alg:spectral_control}.



 \vspace{-1em}
\subsection{Spectral Dynamics Embedding}\label{subsec:spectral_feature}


As we discussed in~\secref{sec:prelim}, we have recast the stochastic nonlinear control model as an MDP.
By recognizing the state transition operator~\eqref{eq:transition} as a Gaussian kernel, we can further decompose the transition dynamics and reward in a \textit{linear representation}, in two ways. The first utilizes the Bochner decomposition, as presented below.
\begin{lemma}[Bochner linear representation, cf. \cite{devinatz1953integral}]\label{prop:linear_mdp}
Consider any $\alpha \in [0,1)$. Denote $\theta_r = [0,  1]^\top$, 
\small
\begin{align}
    &\phi_{\revision{\omega,b}}\rbr{s, a} = \left[\psi_{\revision{\omega,b}}(s,a), r(s,a) \right], \mbox{ where } \\
    & \psi_{\omega,\revision{b}}(s,a) = \frac{g_\alpha\!\rbr{f\rbr{s, a}}}{\alpha^d}\!\cos\!\rbr{\frac{\omega^\top f(s, a)}{\sqrt{1-\alpha^2}}\revision{+b}}\!\!,
\end{align}
\[
\mu_{\omega,\revision{b}}\rbr{s^\prime} = \revision{p_\alpha(s')[{\cos(\sqrt{1-\alpha^2} \omega^\top s^\prime + b)}, 0]^\top},
\]
\normalsize
where $g_{\alpha}\rbr{f(s, a)} \defeq \exp\left(\frac{\alpha^2 \|f(s, a)\|^2}{2(1-\alpha^2)\sigma^2}\right)$, $\omega\sim\mathcal{N}(0, \frac{1}{\sigma^2}I_d)$, \revision{$b \sim U([0,2\pi]) := \mathrm{Unif}([0,2\pi])$} and $p_\alpha(s^\prime) = \frac{\alpha^d}{(2\pi\sigma^2)^{d/2}}\exp\left(-\frac{\|\alpha s^\prime\|^2}{2\sigma^2}\right)$ is a Gaussian distribution for $s^\prime$ with standard deviation $\frac{\sigma}{\alpha}$. Then,
\small
\begin{subequations}\label{eq:rff}
\begin{align}
    P(s^\prime|s, a) =& \EE_{\omega\sim\Ncal\rbr{0, \frac{1}{\sigma^2}I_d}, b\sim U([0,2\pi])}\sbr{\phi_{\omega,b}(s, a)^\top\mu_{\omega,b}(s^\prime)} \nonumber\\
    \defeq& \inner{\phi_{\omega,b}(\revision{s, a})}{\mu_{\omega,b}(s^\prime)}_{\Ncal\rbr{0, \frac{1}{\sigma^2}I_d}\times U([0,2\pi])}, \label{eq:rff1}\\
    r(s, a) =& \mathbb{E}_{\omega\sim\Ncal\rbr{0, \frac{1}{\sigma^2}I_d},b\sim U([0,2\pi])}\sbr{\phi_{\omega,\revision{b}}(s, a)^\top \theta_r}\nonumber\\
   \defeq& \inner{\phi_{\omega,\revision{b}}(\revision{s, a)}}{\theta_r}_{\Ncal\rbr{0, \frac{1}{\sigma^2}I_d}\times U([0,2\pi])}, \label{eq:rff2}
   \end{align}
\end{subequations} 
\normalsize
\end{lemma}
\begin{proof}
The representation of reward $r(s,a)$ through $\phi_{\omega,\revision{b}}(s, a)$ is straightforward since we explicitly include $r(\cdot)$ in $\phi_{\omega,\revision{b}}(\cdot)$. For $P(s^\prime|s, a)$,
we first notice that $\forall \alpha \in (0, 1)$, we have
\footnotesize
\begin{align}
    &P(s^\prime|s, a) \propto \exp\left(-\frac{\left\|s^\prime - f(s, a)\right\|^2}{2\sigma^2}\right)\nonumber
    \\
    &
    = \exp\left(-\frac{\|\alpha s^\prime\|^2}{2\sigma^2}\right) 
    \resizebox{0.47\hsize}{!}{$\exp\rbr{-\frac{\left\|(1-\alpha^2) s^\prime -  f(s, a)\right\|^2}{2\sigma^2(1-\alpha^2)}}
    $}\nonumber\\
    & \resizebox{0.32\hsize}{!}{$\exp\left(\frac{\alpha^2 \|f(s, a)\|^2}{2(1-\alpha^2)\sigma^2}\right).$}\label{eq:equivalent_transition}
\end{align}
\normalsize
The factorization of the transition $P(s^\prime|s, a)$ in \eqref{eq:rff1} can thus be derived by applying~\lemref{thm:bochner} to the second term in~\eqref{eq:equivalent_transition},
with the $\phi_{\omega,\revision{b}}(\cdot)$ and $\mu_{\omega,\revision{b}}(\cdot)$ as the random Fourier features of the Gaussian kernel. 
\end{proof}
The $\phi_{\omega,b}(\cdot)$ is named as the \emph{Bochner Spectral Dynamics Embedding}. 
We emphasize the decomposition in ~\citet{ren2022free} is a special case of in~\eqref{eq:rff} with $\alpha=0$. The tunable $\alpha$ provides benefits for the analysis and may also be used to improve empirical performance, \revision{as seen in our empirical simulations}.

Another linear representation $\phi_M(s,a) \in \ell_2$,\footnote{\revision{The $\ell_2$ space is the set of square-summable infinite-dimensional vectors $v$ indexed by the positive integers, such that $\|v\|_{\ell_2}^2 = \sum_{i=1}^\infty v_i^2 < \infty$.}} which we name as the \emph{Mercer Spectral Dynamics Embedding}, is also possible via the Mercer decomposition, as described below.

\begin{lemma}[Mercer linear representation, cf. \cite{mercer1909functions}]\label{prop:linear_mdp_Mercer}
Let $\mu$ denote a strictly positive Borel measure on $\mathcal{S}$, and denote $k_\alpha(x,x') = \exp\left(-\frac{(1-\alpha^2)\|x-x'\|^2}{2\sigma^2} \right)$ for any $0 \leq \alpha < 1$. By Mercer's theorem, $k_\alpha$ admits a decomposition for any $(x,x') \in \mathcal{X}$:
\small
\begin{align}
    k_\alpha(x,x') \! =\! \sum_{i=1}^\infty\! \sigma_i e_i(x)e_i(x'), \  \{e_i \} \mbox{ a basis for } \mathcal{L}_2(\mu).
\end{align}
\normalsize
We denote $k_\alpha(x,x') = \langle \tilde{e}(x), \tilde{e}(x')\rangle_{\ell_2}$ where $\tilde{e}_i(x) = \sqrt{\sigma_i} e_i(x)$ for each positive integer $i$.
Denote $\theta_{r,M} = [1,0,0,\dots]^\top \in \ell_2$ (and the same $g_{\alpha}\rbr{f(s, a)}$ and $p_\alpha(s^\prime)$ as in \lemref{prop:linear_mdp}), and
\small
\begin{align}
    &\phi_M \rbr{s, a} = [r(s,a),\psi_M \rbr{s,a}], \mbox{ where } \\ 
    &\quad \quad \psi_M(s,a) =\frac{g_{\alpha}\rbr{f(s, a)} }{\alpha^d} \tilde{e}\left( \frac{f(s,a)}{1-\alpha^2}\right)^\top, \label{eq:psi_M_first_defn}\\
    &\mu_M(s') = p_\alpha(s') \left[0, \tilde{e}\left(s'\right)\right]^\top  
\end{align}
\normalsize
Then,
\small
\begin{subequations}\label{eq:rff_Mercer}
\begin{align}
    P(s^\prime|s, a) =& \langle \phi_M (s,a), \mu_M(s') \rangle_{\ell_2} \\
    r(s, a) =& \langle \phi_M(s,a), \theta_{r,M} \rangle_{\ell_2}
   \end{align}
\end{subequations} 
\normalsize
\end{lemma}
\begin{proof}
    The proof is analogous to the proof of \lemref{prop:linear_mdp}, where here we apply Mercer's theorem to decompose the middle term in terms of the kernel $k_\alpha$ 
    \footnotesize
    \begin{align*}
        k_\alpha\left(\!s',\frac{f(s,a)}{1-\alpha^2}\!\right) \!\!=\!\! & \  \exp\rbr{\!-\frac{\left\|(1-\alpha^2) s^\prime -  f(s, a)\right\|^2}{2\sigma^2(1-\alpha^2)}\!}.
    \end{align*}
\end{proof}
    \normalsize
\vspace{-1.0em}
The MDP with the factorization either in~\eqref{eq:rff} or in \eqref{eq:rff_Mercer} shows that we can represent the stochastic nonlinear control problem as an instantiation of the \emph{linear MDP}~\citep{jin2020provably,yang2020reinforcement}, where the state transition and reward satisfy
$$P(s'|s,a)=\inner{\phi(s,a)}{\mu(s')}, \ r(s,a)=\inner{\phi(s,a)}{\theta}$$
with the feature functions $\phi(s,a)$ be either $\cbr{\phi_{\omega,\revision{b}}(s, a)}$ with $\omega\sim \Ncal\rbr{0, \revision{\frac{1}{\sigma^2}}I_d}, b \sim U([0,2\pi])$ or $\phi_M(s,a) \in \ell_2$.  
The most significant benefit of linear MDP is the property that it induces a function space composed by $\phi(s,a)$ 
where the $Q$-function for arbitrary policy can be linearly represented.
\begin{lemma}[cf. Proposition 2.3 in ~\cite{jin2020provably}]\label{prop:q_linear}
For any policy, there exist weights $\{\theta_{\omega,\revision{b}}^\pi\}$ (where $\omega \sim \mathcal{N}(0,\frac{1}{\sigma^2}I_d), \revision{b \sim U([0,2\pi])}$) and $\theta_M^\pi \in \ell_2$ such that the corresponding state-action value function satisfies\footnote{\revision{We note that these inner products refer to the inner product defined in Equation \ref{eq:rff}} and the standard $\ell_2$ space inner product respectively.}
\small
\begin{align*}
    Q^\pi(s, a) =  & \ \inner{ \phi_{\omega,\revision{b}}(s, a)}{\theta_{\omega,\revision{b}}^\pi}_{\mathcal{N}(0,\frac{1}{\sigma^2}I_d) \times U([0,2\pi])} \\
    = & \  \inner{ \phi_M(s, a)}{\theta_M^\pi}_{\ell_2}.
\end{align*} 
\end{lemma}
\normalsize
\begin{proof}
    Our claim can be verified easily by applying the decomposition in~\eqref{eq:rff} to Bellman recursion:
    \small
\begin{align}
    &Q^\pi(s, a) =  r(s, a) + \gamma \mathbb{E}_{P} \sbr{V^\pi(s^\prime)} \label{eq:q_linear}\\
    &=  \Big\langle\!\phi_{\omega,\revision{b}}(s,\! a), \underbrace{\theta_r \!+\! \gamma\int_{\mathcal{S}} \mu_{\omega,\revision{b}}(\!s^\prime\!) V^\pi(s^\prime)dp_\alpha(s^\prime)}_{\theta_{\omega,\revision{b}}^\pi}\!\Big\rangle_{\!\!\mathcal{N}(0,\frac{1}{\sigma^2}I_d) \times U([0,2\pi])},  \nonumber \\
    &= \Big\langle\! \phi_M(s, a), \underbrace{\theta_{r,M} \!+\!{\gamma} \int_{\mathcal{S}} \mu_M(s^\prime) V^\pi(s^\prime) dp_\alpha(s^\prime)}_{\theta_M^\pi}\!\!\Big\rangle_{\!\!\ell_2}\!\!,  \nonumber
\end{align}
\normalsize
where the second equation comes from the random feature decomposition and the third equation comes from the Mercer decomposition.
\end{proof}
Immediately, for the stochastic nonlinear control model~\eqref{eq:nonlinear_control} with an arbitrary dynamics $f(s, a)$, the spectral dynamics embeddings $\phi_{\omega,\revision{b}}(\cdot)$ and $\phi_M(\cdot)$ provides a linear space, in which we can conceptually conduct dynamic programming for policy evaluation and optimal control with \emph{global optimal} guarantee. Compared to the existing literature of \emph{linear MDP} (\citep{jin2020provably,yang2020reinforcement}) that assumes a given linear factorization with finite dimension $\phi$,  both the Bochner and Mercer Spectral Dynamics embedding decomposes general stochastic nonlinear model explicitly through an infinite-dimensional kernel view, motivating us to investigate the finite-dimensional approximation. 



 \vspace{-1em}
\subsection{Control with Finite-dimensional Approximation}\label{subsec:finite_approx}


\RestyleAlgo{ruled}
\SetKwComment{Comment}{}{}
\begin{algorithm}[t]
\LinesNumbered
\caption{\AlgName~(\algabb)}\label{alg:spectral_control}
\KwData{Transition Model $s^\prime = f(s, a) + \varepsilon$ where $\varepsilon \sim \mathcal{N}(0, \sigma^2 I_d)$, Reward Function $r(s, a)$, Number of Random/Nystr\"om Feature $m$, Number of Nystr\"om Samples $n_{\mathrm{Nys}} \geq m$, Nystr\"om Sampling Distribution $\mu_{\mathrm{Nys}}$, Number of Sample $n$, Factorization Scale $\alpha$, Learning Rate $\eta$, Finite-Truncation-Method $\in \{\mbox{'Random features', 'Nystrom'}\}$}
\KwResult{$\pi$}
\Comment{\color{blue}Spectral Dynamics Embedding Generation}
\label{algline:spectral_embedding}
Generate $\phi(s,a)$ using Algorithm \ref{alg:spectral_feat_generation}  \\
\Comment{\color{blue}Policy evaluation and update}
Initialize $\theta_0 = 0$ and $\pi_0(a|s) \propto \exp(\phi(s, a)^\top \theta_0)$.

\For{$k=0, 1, \cdots, K$}
{
    \Comment{{\color{blue}Least Square Policy Evaluation}}
    Sample \iid~ $\{(s_i, a_i, s_i^\prime), a_i^\prime\}_{i\in [n]}$ where $(s_i, a_i) \sim \nu_{\pi_k}$, $s_i^\prime = f(s_i, a_i) + \varepsilon$ where $\nu_{\pi_k}$ is the stationary distribution of $\pi_k$, $\varepsilon \sim \mathcal{N}(0, \sigma^2 I_d)$, $a_i^\prime \sim \pi_k(s_i^\prime)$. \label{algline:policy_eval_start}
     
    Initialize $\hat{w}_{k, 0} = 0$.
    
    \For{$t=0, 1, \cdots, T-1$}
    {
        Solve 
        \begin{align}
        & \hat{w}_{k, t+1} = \mathop{\arg\min}_{w} \nonumber\\
        & \resizebox{1.0\hsize}{!}{$
        \left\{ \sum_{i\in [n]} \left(\phi(s_i, a_i)^\top w - r(s_i, a_i) - \gamma \phi(s_i^\prime, a_i^\prime)^\top \hat{w}_{k, t}\right)^2\right\}$
        }
        \label{eq:projected_least_square}
        \end{align}
    }
    \label{algline:policy_eval_end}
    \Comment{{\color{blue} Natural Policy Gradient for Control}}
    Update $\theta_{k+1} = \theta_k + \eta\hat{w}_{k, T}$ and 
    \begin{equation}\label{eq:natural_pg}
    \pi_{k+1}(a|s) \propto \exp(\phi(s, a)^\top \theta_{k+1}).  
    \end{equation}
    \label{algline:policy_update}
}
\end{algorithm}

\RestyleAlgo{ruled}
\SetKwComment{Comment}{}{}
\begin{algorithm}
\LinesNumbered
\caption{Spectral Dynamics Embedding Generation}\label{alg:spectral_feat_generation}
\KwData{Same data as Algorithm \ref{alg:spectral_control}}
\KwResult{$\phi(\cdot,\cdot)$}
\uIf{Finite-Truncation-Method = 'Random features'}
{
Sample \iid~ $\{\omega_i\}_{i\in [m]}$ and \revision{$\{b_i\}_{i \in [m]}$} where $\omega_i \sim \mathcal{N}(0, \frac{1}{\sigma^2} I_d)$, $b_i \sim U([0,2\pi])$ and construct the feature 
\label{algline:random_feature}
\begin{align}
    \begin{split}
\psi_{\mathrm{rf}}(s, a)\!\! =\!\! \frac{g_\alpha(f(s,a))}{\alpha^d}\left[\cos\left(\omega_i^\top \frac{f(s, a)}{\sqrt{1-\alpha^2}} + b_i) \right)\right]_{i\in [m]}
\label{eq:alg_phi_rf_defn}
\end{split}
\end{align}
Set $\phi(s,a) = [\psi_{\mathrm{rf}}(s,a), r(s,a)]$}
\ElseIf{Finite-Truncation-Method = 'Nystrom'}
{
Sample $n_{\mathrm{nys}}$ random samples $\{x_1,\dots,x_{\mathrm{nys}}\}$ independently from $\mathcal{S}$ following the distribution $\mu_{\mathrm{Nys}}$. 

Construct $n_{\mathrm{nys}}$-by-$n_{\mathrm{nys}}$ Gram matrix given by $K_{i,j}^{(n_{\mathrm{Nys}})} = k_\alpha(x_i,x_j)$

Compute the Eigendecomposition $K^{(n_{\mathrm{nys}})} U = \Lambda U$, with $\lambda_1 \geq \dots \geq \lambda_{n_{\mathrm{nys}}}$ denoting the corresponding eigenvalues. 

Construct the feature
\label{algline:nystrom_feature}
\begin{align}
\label{eq:alg_phi_nystrom_defn}
    \psi_{\mathrm{nys}}(s,a) \!\!=\!\! \left[\frac{g_{\alpha}\!(\!f(s,a)\!)\!}{\alpha^d \sqrt{\lambda_i}} \!\sum_{\ell = 1}^{n_{\mathrm{nys}}}\!U_{i,\ell} k_\alpha\left(x_{\ell},\! \frac{f(s,a)}{1\!-\!\alpha^2}\right)\!\right]_{i \in [m]}
\end{align}
Set $\phi(s,a) = [\psi_{\mathrm{nys}}(s,a), r(s,a)]$
} \label{algline:end_feat_generation}
\end{algorithm}

Although the spectral dynamics embeddings in~\lemref{prop:linear_mdp} and \lemref{prop:linear_mdp_Mercer} provide linear spaces to represent the family of $Q$-function, there is still a major challenge to be overcome for practical implementation. Specifically, for the Bochner random feature embedding, the dimension of $\phi_{\omega,\revision{b}}(\cdot)$ is \emph{infinite} with $\omega\sim \Ncal\rbr{0, \frac{1}{\sigma^2}I_d}$ and \revision{$b \sim U([0,2\pi])$}, which is computationally intractable.  Similarly, for the Mercer embedding, the dimension of $\phi_M(\cdot) \in \ell_2$ is also infinite. \newrevision{Thus, despite the existence of a linear factorization $P(s' \mid s,a) = \langle \phi(s,a), \mu(s')\rangle p_\alpha(s')$, since the $\phi(s,a)$ and $\mu(s')$ are infinite-dimensional, utilizing the Bellman equation, representing $Q^\pi(s,a) = r(s,a) + \gamma\langle \phi(s,a), \int_{s'} \mu(s') V^\pi(s') p_\alpha(s') ds'\rangle$ requires learning infinite-dimensional weights which is intractable. At a high-level, our analysis will demonstrate that we can use finite-dimensional feature approximation methods (both random features and Nystrom features) to approximate the Gaussian transition kernel $P(s' \mid s,a)$ as $P(s' \mid s,a) \approx \hat{\phi}(s,a)^\top \hat{\mu}(s') p_\alpha(s')$ for some finite-dimensional features $\hat{\phi}(s,a)$ and $\hat{\mu}(s')$, such that 
\small
\begin{align*}
    Q^\pi(s,a) \approx & \  r(s,a) + \gamma \hat{\phi}(s,a)^\top \left(\int_{s'} \hat{\mu}(s') V^{\pi}(s') p_\alpha(s') ds'\right)\\
    = & \  \begin{bmatrix}
    r(s,a) \\
    \hat{\phi}(s,a)
    \end{bmatrix}^\top \underbrace{\begin{bmatrix}
        1 \\
        \gamma \int_{s'} \hat{\mu}(s') V^{\pi}(s') p_\alpha(s') ds'
    \end{bmatrix}}_{\theta^\pi}.
\end{align*}
\normalsize
This demonstrates that $Q^\pi(s,a)$ can be approximately represented as a linear combination of $r(s,a)$ and the features $\hat{\phi}(s,a)$, with finite-dimensional linear coefficients $\theta^\pi$. 
Our algorithm first generates these features $\hat{\phi}(s,a)$, and then learn the weights $\theta^\pi$ in a data-driven fashion.}
For the random feature embedding, \citet{ren2022free} suggested the finite-dimensional random feature, which is the Monte-Carlo approximation for the Bochner decomposition~\cite{rahimi2007random}, and demonstrated strong empirical performances for reinforcement learning. For the Mercer embedding, a well-known method to derive a finite-dimensional approximation of the Mercer expansion is the Nystr\"om method \cite{williams2001using}, which seeks to approximate the subspace spanned by the top eigenfunctions in the Mercer expansion by using the eigendecomposition of an empirical Gram matrix. Compared to random features, it is known that the Nystr\"om method can yield tighter approximation of the kernel when the kernel eigenvalues satisfy certain decay assumptions \cite{yang2012nystrom}, motivating us to study and compare both finite-dimensional truncations. 

In this section, we first formalize the \AlgName~(\algabb) algorithm, implementing the dynamic programming efficiently in a principled way as shown in~\algtabref{alg:spectral_control}, whose theoretical property will be analyzed in the next section. In \AlgName~(\algabb), there are three main components, 
\begin{enumerate}[leftmargin=*]
    \item \textit{Generating spectral dynamics embedding (Line~\ref{algline:spectral_embedding} in Algorithm \ref{alg:spectral_control}, which calls Algorithm \ref{alg:spectral_feat_generation} as a subroutine.}) To generate the spectral dynamics embedding, we consider either using random features as a finite-dimensional truncation of the Bochner random features, or the Nystr\"om method as a finite-dimensional truncation of the Mercer features. 
    \begin{itemize}[leftmargin = *]
        \item {\bf Random features.} Following~\lemref{prop:linear_mdp}, we construct finite-dimensional $\phi$ by Monte-Carlo approximation, building $\psi_{\mathrm{rf}} = [\psi_{\omega_i,\revision{b_i}}]_{i=1}^m$, where $\omega_i\sim\Ncal\rbr{0, \frac{1}{\sigma^2} I}$ and \revision{$b_i \sim U([0,2\pi])$}. We concatenate $\psi_{\mathrm{rf}}$ with $r$ to form $\phi$.
        
        \item {\bf Nystr\"om features.} At a high level, we use the Nystr\"om method \cite{williams2001using} to obtain the Nystr\"om feature decomposition of $k_\alpha(\cdot,\cdot) \approx \varphi_{\mathrm{nys}}(\cdot)^\top \varphi_{\mathrm{nys}}(\cdot)$ for some $\varphi_{\mathrm{nys}}(\cdot) \in \mathbb{R}^m$ by drawing $n_{\mathrm{nys}}$ random samples \revision{iid from a sampling distribution $\mu_{\mathrm{nys}}$ over the state space $\mathcal{S}$, where the sampling distribution can be specified by the user, and in our analysis is chosen to be $p_\alpha$} (note we always choose $m \leq n_{\mathrm{nys}}$)  and then working with the eigendecomposition of the corresponding Gram matrix; we derive and motivate the Nystr\"om method in greater detail in \appref{appendix:derivation_of_nystrom}. Then, noting that 
        \footnotesize
        $$k_\alpha\rbr{s', \frac{f(s,a)}{1-\alpha^2}} \approx \varphi_{\mathrm{nys}}\rbr{\frac{f(s,a)}{1-\alpha^2}}^\top \varphi_{\mathrm{nys}}\rbr{s'}, $$
        \normalsize
        echoing the Mercer decomposition in \lemref{prop:linear_mdp_Mercer}, we set 
        \footnotesize
        $$\psi_{\mathrm{nys}}(s,a) = \frac{g_\alpha(f(s,a))}{\alpha^d}\varphi_{\mathrm{nys}}\rbr{\frac{f(s,a)}{1-\alpha^2}},$$
        \normalsize
        which can be viewed as a finite-dimensional truncation of the Mercer feature in \eqref{eq:psi_M_first_defn}. We specify the exact form of $\varphi_{\mathrm{nys}}(\cdot)$ in \eqref{eq:varphi_m_exact_form} of \appref{appendix:derivation_of_nystrom}, which results in the feature $\psi_{\mathrm{nys}}$ in  \eqref{eq:alg_phi_nystrom_defn} of Algorithm \ref{alg:spectral_feat_generation}.

    \end{itemize}

            \item \textit {Policy evaluation (Line~\ref{algline:policy_eval_start} to~\ref{algline:policy_eval_end} in Algorithm \ref{alg:spectral_control}).} We conduct least square policy evaluation for estimating the state-action value function  of current policy upon the generated finite-dimensional truncation features. Specifically, we sample\footnote{\revision{Our analysis also assumes that we sample exactly from $\nu_\pi(s,a)$, but in practice, we can only approximately sample from this distribution, which may incur another source of error.}} 
     from the stationary distribution $\nu_{\pi}(s, a)$ from dynamics under current policy $\pi$ and solve a series of least square regression~\eqref{eq:projected_least_square} to learn a $Q^\pi(s, a) = \phi\rbr{s, a}^\top w^\pi$, with $w^\pi\in \RR^{m\times 1}$ by minimizing a Bellman recursion type loss. 
    \item \textit{Policy update (Line~\ref{algline:policy_update} in Algorithm \ref{alg:spectral_control}).} Once we have the state-value function for current policy estimated well in step 2), we will update the policy by natural policy gradient or mirror descent~\citep{agarwal2021theory}. 
    Specifically, we have the functional gradient of accumulated reward w.r.t. a direct parameterization policy $\pi$ as $    \nabla_{\pi} J^\pi = \frac{1}{1 - \gamma}\nu^\pi\rbr{s}Q^\pi\rbr{s, a},$
    where $\nu^\pi\rbr{s, a}$ denotes the stationary distribution of $\pi$. At $t$-iteration, we estimated the $Q^{\pi_t}\rbr{s, a}$ with $\phi(s, a)^\top \revision{w}^{\pi_t}$, we have the mirror descent update~\citep{nemirovskij1983problem} for $\pi$ as 
    \small
    \begin{align}
    & \pi_{t+1} = \hspace{-2mm}\argmax_{\pi\rbr{\cdot|s}\in\Delta\rbr{\Acal}}\,\, 
     \inner{\pi}{\nu^{\pi_t}\rbr{s}\phi(s, a)^\top w^{\pi_t}}
    + \frac{1}{\eta} KL\rbr{\pi||\pi_t} \nonumber  
    \end{align}
    \normalsize
    leading to the following closed-form solution, 
    \small
    \begin{align}
    \pi_{t+1}(a|s)&\propto\pi_t(a|s)\exp\rbr{\phi\rbr{s, a}^\top \eta \revision{w}^{\pi_t}}\nonumber\\ 
    &= \exp\rbr{\phi\rbr{s, a}^\top \theta_{t+1}} \nonumber
    \end{align}
    \normalsize
    where $\theta_{t+1} = {\sum_{i=0}^t\eta \revision{w}^{\pi_i}}$, giving the update rule in~(\ref{eq:natural_pg}). 
\end{enumerate}
Once the finite-dimensional spectral dynamics embedding has been generated in step 1), the algorithm alternates between step 2) and step 3) to improve the policy. 

\begin{remark}[{\bf Policy evaluation and update}] \label{rem:sdec-policy}
   We note that although \algabb uses the least square policy evaluation and natural policy gradient method for the policy update, the proposed dynamic spectral embedding is compatible with other policy evaluation and policy update methods, including cutting-edge deep reinforcement learning methods, such as Soft Actor-Critic~\citep{haarnoja2018soft}, by using our proposed representation to approximate the critic function. In other words, one of the novelties of \algabb, is that it exploits the known nonlinear dynamics to obtain a natural, inherent representation space that could be adopted by various dynamical programming or policy gradient-based methods. 
\end{remark}


\begin{remark}[{\bf Beyond Gaussian noise}]
Although \algabb mainly focuses the stochastic nonlinear dynamics with Gaussian noise, the Nystr\"om feature is clearly applicable for any transition operator. 
Moreover, we can easily extend the random feature for more flexible noise models. Specifically, we consider the exponential family condition distribution for the transition~\citep{wainwright2008graphical}, \ie, 
\small
\begin{equation}\label{eq:exp_family}
    P(s'|s, a) \propto p(s')\exp\rbr{f\rbr{s, a}^\top \zeta\rbr{s^\prime}},
\end{equation}
\normalsize
The transition~\eqref{eq:exp_family} generalizes the noise model in the transition in two-folds: {\bf i)}, the noise model is generalized beyond Gaussian to exponential family, which covers most common noise models, \eg, exponential, gamma, Poisson, Wishart, and so on; and {\bf ii)}, the nonlinear transformation over $\zeta\rbr{s'}$ also generalize the dynamics. The random features for~\eqref{eq:exp_family} are derived in~\appref{appendix:rf_ebm}.  

\end{remark}

%% file: analysis.tex
\section{Theoretical Analysis}\label{sec:analysis}

The major difficulty in analyzing the optimality of the policy obtained by~\algabb is the fact that after finite-dimensional truncation, the transition operator constructed by the random or Nystr\"om features is no longer a valid distribution, \ie, it lacks non-negativity and normalization, which induces a \emph{pseudo-MDP}~\citep{yao2014pseudo} as the approximation.  For instance, in general, for the random features, the term $\Phat(s^\prime|s, a)\defeq \frac{1}{m}\sum_{i=1}^m\phi_{\omega_i,\revision{b_i}}\rbr{f(s, a)} \mu_{\omega_i,\revision{b_i}}\rbr{s'} $ with $\cbr{\omega_i,\revision{b_i}}_{i=1}^m\sim \Ncal\rbr{0, \sigma^{-2}I_d} \times U([0,2\pi])$ may be negative for some values of $s^\prime$ and $(s,a)$. As a consequence, the value function for the pseudo-MDP is not bounded, as the normalization condition of the transition operator no longer holds. Then, the vanilla proof strategy used in majority of the literature since~\cite{kearns2002near}, \ie, analyzing the optimality gap between policies through simulation lemma, is no longer applicable. 

In this section, we bypass the difficulty from pseudo-MDP, and 
provide rigorous investigation of the impact of the approximation error for policy evaluation and optimization in~\algabb, filling the long-standing gap. We first specify the assumptions, under which we derive our theoretical results below.
These assumptions are commonly used in the literature~\citep{yu2008new,abbasi2019politex}.


    
    
    


\begin{assumption}[Regularity Condition for Dynamics and Reward] 
$\|f(s, a)\|\leq c_f$, and $r(s, a) \leq c_r$ for all $(s,a) \in \mathcal{S} \times \mathcal{A}$. For simplicity, we omit polynomial dependencies on $c_f$ and $c_r$. 
\end{assumption}
\begin{assumption}[Regularity Condition for Feature] 
The features are linearly independent.
\end{assumption}
\begin{assumption}
[Regularity Condition for Stationary Distribution~\cite{abbasi2019politex}] 
\revision{Any policy $\pi$ has full support over $\mathcal{A}$. In addition, for any policy $\pi$, the discounted stationary state distribution $d^{\pi}$ has full support over any region in $\mathcal{S}$ on which an optimal policy $\pi^*$ is supported, and the discounted stationary state-action distribution $\nu^\pi$ satisfies the following conditions}: 
\begin{align}
    & \resizebox{0.55\hsize}{!}{$
    \sigma_{\min}\left(\mathbb{E}_{\nu_{\pi}} \left[\phi(s, a) \phi(s, a)^\top\right]\right) \geq \Upsilon_1,$} \label{eq:assumption_3_cond1}\\
    & \resizebox{0.8\hsize}{!}{$\sigma_{\min}\left(\mathbb{E}_{\nu_{\pi}} \left[\phi(s, a) \left(\phi(s, a) - \gamma \mathbb{E}_{\nu_\pi} \phi(s^\prime, a^\prime)\right)^{\top}\right]\right) \geq \Upsilon_2, 
    $} \label{eq:assumption_3_cond2}
\end{align}
where $\Upsilon_1, \Upsilon_2 > 0$, $\sigma_{\min}(\cdot)$ denotes smallest singular value.
\end{assumption}



\revision{
Assumption 1 is an assumption that holds whenever the state space is compact.
Assumptions 2 and 3 are also typically challenging to verify. Nonetheless, for our proposed features $\phi(s,a) = [\psi(s,a), r(s,a)]$, where $\psi_(s,a)$ can either denote $\psi_{rf}(s,a)$ or $\psi_{\mathrm{nys}}(s,a)$, we show that the transition features $\psi(s,a)$ are linearly independent under very mild conditions, which then implies that Assumption 2 holds when the reward function $r(s,a)$ does not lie exactly in the span of the transition feature functions $\psi(s,a)$; this again is a mild condition that is for instance satisfied for any (locally) quadratic reward function when random features are used. When this happens, we then show that the two key conditions (\ref{eq:assumption_3_cond1}) and (\ref{eq:assumption_3_cond2}) in Assumption 3 holds for all but finitely many $0 < \lambda < 1$. 
\begin{lemma}
\label{lemma:assumption_2_holds_rf}
The random features $\psi_{rf}(s,a)$ are linearly independent almost surely.
\end{lemma}

\begin{lemma}
\label{lemma:assumption_2_holds_nys}
Consider any feature dimension $0 < m \leq n_{\mathrm{nys}}$. Suppose the Nystrom Gram matrix $K^{(n_{\mathrm{nys}})}$ has rank at least $m$. Then, the Nystrom features $\psi_{\mathrm{nys}}(s,a)$ are linearly independent.
\end{lemma}

\begin{lemma}
\label{lemma:we_satisfy_assumption_3}
Suppose the features $\phi(s,a)$ are linearly independent over the interior of the support of $\nu^\pi$ for any policy $\pi$.  Then,  for all but finitely many $0 \leq \lambda < 1$, the two conditions (\ref{eq:assumption_3_cond1}) and (\ref{eq:assumption_3_cond2}) in Assumption 3 holds.
\end{lemma}
Proofs of Lemmas 7,8 and 9 are deferred to Appendix \ref{appendix:when_assumptions_hold}.
}
Throughout the analysis, \revision{for notational simplicity} it will be helpful to define the term 
\small
\begin{align}
    \tilde{g}_\alpha := \sup_{s,a} \left(g_\alpha(f(s,a))\right)\alpha^{-d}
\end{align}
\normalsize
We provide a brief outline of our overall proof strategy. First, in~\secref{subsec:error_eval} we study the policy evaluation error. To do so, for any policy $\pi_k$ encountered during the algorithm, let $\hat{Q}^{\pi_k}$ denote the learned Q-value at the end of the policy evaluation step in Line~\ref{algline:policy_eval_end} of Algorithm \ref{alg:spectral_control}. Then, by the triangle inequality, the approximation error satisfies
\small
$$\|Q^{\pi_k} - \hat{Q}^{\pi_k}\|_{\nu_{\pi_k}} \leq \underbrace{\|Q^{\pi_k} - \tilde{Q}^{\pi_k}\|_{\nu_{\pi_k}}}_{\mbox{\textit{approximation error}}} + \underbrace{\|\tilde{Q}^{\pi_k} - \hat{Q}^{\pi_k}\|_{\nu_{\pi_k}}}_{\mbox{\textit{statistical error}}},$$
\normalsize
where $\tilde{Q}^{\pi_k}$ denotes the solution to a projected Bellman equation where the projection is onto the either the finite-dimensional 
random features ((Line~\ref{algline:random_feature} in Algorithm \ref{alg:spectral_feat_generation})) or Nystr\"om features (Line~\ref{algline:nystrom_feature} in Algorithm \ref{alg:spectral_feat_generation}); we explain this in greater detail in \secref{subsec:error_eval}. Essentially, the approximation error corresponds to the truncation error incurred by using finite-dimensional features, while the statistical error is the error incurred by using finitely many samples to approximate $\tilde{Q}^{\pi_k}$. Bounding these two sources of errors gives us a total error for policy evaluation bound in Theorem \ref{thm:total_error_evaluation}. Second, in \secref{subsec:control_analysis}, we study the policy optimization error of the natural policy gradient in Line \ref{algline:policy_update} of Algorithm \ref{alg:spectral_control}.
We note that the policy optimization error includes a component inherited from the policy evaluation error. 
This culminates in Theorem \ref{thm:opt_policy}, which provides the overall performance guarantee for the control optimality of \algabb. 
\subsection{Error Analysis for  Policy Evaluation}\label{subsec:error_eval}
Our analysis starts from the error for policy evaluation (Line~\ref{algline:policy_eval_start} to~\ref{algline:policy_eval_end} in~\algtabref{alg:spectral_control}). We decompose the error into two parts, one is the approximation error due to the limitation of our basis (\ie, finite $m$ in~Line \ref{algline:random_feature} and Line \ref{algline:nystrom_feature} of \algtabref{alg:spectral_feat_generation}), and one is the statistical error due to the finite number of samples we use (\ie, finite $n$ in Line~\ref{algline:policy_eval_start} of \algtabref{alg:spectral_control}). For notational simplicity, we omit $\pi$ and use $\nu$ to denote the stationary distribution corresponding to $\pi$ in this section.

\textbf{Approximation Error} We first provide the bound on the approximation error in terms of representing the $Q$-function for arbitrary policy due to using an imperfect finite-dimensional basis. When restricted to using a finite-dimensional basis, the best possible $\tilde{Q}$ approximation is the solution to a projected Bellman equation (cf. \cite{yu2008new}), defined as follows: given any (possibly finite-dimensional) feature map $\Phi := \{\phi(s,a)\}_{(s,a) \in \mathcal{S} \times \mathcal{A}}$, we define the approximation $\tilde{Q}_{\Phi}^\pi$ as
\begin{align}
\label{eq:tilde_Q_phi_general_defn}
    \tilde{Q}_{\Phi}^\pi = \Pi_{\nu, \Phi}(r + P^\pi \tilde{Q}_{\Phi}^\pi),
\end{align}
where $P^\pi$ is defined as
\begin{align}
    \label{eq:P_pi_defn}
    (P^\pi \revision{h})(s,a) = \mathbb{E}_{(s',a') \sim P(s,a) \times \pi} \revision{h}(s',a').
\end{align}
and $\Pi_{\nu,\Phi}$ is the projection operator as defined as 
\begin{align}
\label{eq:pi_nu_phi_projector_operator}
    \Pi_{\nu,\Phi} Q = \mathop{\arg\min}_{\revision{h}\in\mathrm{span}(\Phi)} \mathbb{E}_{\nu}\left(Q(s, a) - \revision{h}(s, a)\right)^2.
\end{align}
Our interest in $\tilde{Q}_{\Phi}^\pi$ stems from the fact that the least-squares policy evaluation step in Algorithm \ref{alg:spectral_control} (see equation \eqref{eq:projected_least_square}) recovers $\tilde{Q}_{\Phi}^\pi$ if the number of samples, $n$, goes to infinity. We will address the statistical error from using a finite $n$ later. We are now ready to introduce our bound for the approximation error in using finite-dimensional random and Nystr\"om features. We begin with the result for random features. To do so, we first need the following technical result in \cite{rahimi2008weighted}.
\begin{lemma}[cf. Lemma 1 from \cite{rahimi2008weighted}]
    \label{lemma:lemma_1_from_kitchen_sink}
    Let $p$ be a distribution on a space $\Omega$, and consider a mapping $\phi(x;\omega) \in \revision{\mathbb{R}}$. Suppose $f^*(x) = \int_{\Omega} p(\omega) \beta(\omega) \phi(x;\omega) d\omega$
    for some \revision{scalar $\beta(\omega) \in \mathbb{R}$} where $\sup_{x,\omega} \left|\beta(\omega) \phi(x;\omega)\right| \leq C$ for some $C > 0$. Consider $\{\omega_i\}_{i=1}^k$ drawn iid from $p$, and denote the sample average of $f^*$ as $\hat{f}(x) = \frac{1}{K} \sum_{k=1}^K \beta(\omega_k) \phi(x;\omega_k)$. Then, for any $\delta > 0$, w.p. at least $1 - \delta$ over the random draws of $\{\omega_i\}_{i=1}^k$, 
    \small
$$\sqrt{\int_{\mathcal{X}} \left(\hat{f}(x) - f^*(x)\right)^2 d\mu(x)} \leq \frac{C}{\sqrt{K}}\left( 1 + \sqrt{2\log\frac{1}{\delta}}\right).$$ 
\normalsize
\end{lemma}

\begin{proposition}[$Q$-Approximation error with random features]\label{thm:approx_error}
We define the feature map $\Phi_{\mathrm{rf}}$ for random features as follows
\small
\begin{align}
    \label{eq:random_features_phi_definition}
    \Phi_{\mathrm{rf}} = \{[\psi_{\mathrm{rf}}(s,a), r(s,a)]\}_{(s,a) \in \mathcal{S} \times \mathcal{A}}, 
\end{align}
\normalsize
where $\psi_{\mathrm{rf}}(s,a)$ is defined in \eqref{eq:alg_phi_rf_defn} in Algorithm \ref{alg:spectral_control}. Then, for any $\delta > 0$, w.p. at least $1-\delta$, we have that
\small
\begin{align}
    \left\|Q^\pi - \tilde{Q}_{\Phi_{\mathrm{rf}}}^\pi\right\|_{\nu} = \tilde{O}\left(\frac{\gamma \tilde{g}_{\alpha}}{(1-\gamma)^2\sqrt{m}} \right),
\end{align}
\normalsize
where $ \left\|f\right\|_{\nu} := \int f^2 d \nu$, and $\tilde{Q}_{\Phi_{\mathrm{rf}}}^\pi$ is defined in \eqref{eq:tilde_Q_phi_general_defn}. 
\end{proposition}
\begin{proof}
Using the contraction property and results in \cite{yu2008new} gives
\footnotesize
\begin{align}\label{eq:thm7_step1}
    \left\|Q^\pi - \tilde{Q}_{\Phi_{\mathrm{rf}}}^\pi\right\|_{\nu} \leq \frac{1}{1-\gamma} \left\|Q^\pi - \Pi_{\nu,\Phi_{\mathrm{rf}}} Q^\pi\right\|_{\nu},
\end{align}
\normalsize
where $\Pi_{\nu,\Phi_{\mathrm{rf}}}$ is defined as in \eqref{eq:pi_nu_phi_projector_operator}. 
By definition $\Pi_{\nu,\Phi_{\mathrm{rf}}}$ is contractive under $\|\cdot\|_{\nu}$. Note that 
\footnotesize
\begin{align*}
    & Q^\pi(s, a) = r(s, a) \\
    & + \gamma \mathbb{E}_{\omega \sim \mathcal{N}(0, \frac{1}{\sigma^2} I_d), \revision{b \sim U([0,2\pi])}} \left[\phi_{\omega,\revision{b}}(s, a)^{\top} \int_{\mathcal{S}}p_\alpha(s^\prime)\mu_{\omega}(s^\prime) V^\pi(s^\prime) d s^\prime\right].
\end{align*}
\normalsize
With H\"older's inequality, as well as $\|\mu_{\omega,\revision{b}}(s^\prime)\|= O(1)$ by definition, $\left\|V^\pi\right\|_{\infty} = O((1-\gamma)^{-1})$, and using the notation $\beta_{\omega,\revision{b}}^\pi := \int_{\mathcal{S}}p_\alpha(s^\prime)\mu_{\omega,\revision{b}}(s^\prime) V^\pi(s^\prime) d s^\prime$,
we have for every $(\omega,\revision{b})$ that
$\left\|\beta_{\omega,\revision{b}}^\pi\right\| = O((1-\gamma)^{-1}).$
In addition, recalling that $\phi_{\omega,\revision{b}}(s,a) = [\psi_{\omega,\revision{b}}(s,a), r(s,a)]$,
we have 
\footnotesize
$$\sup_{((s,a),(\omega,\revision{b}))} |\phi_{\omega,\revision{b}}(s,a)\beta_{\omega,\revision{b}}^\pi| \leq \sup_{((s,a),(\omega,\revision{b}))} |\psi_{\omega,\revision{b}}(s,a)|\left|\beta_{\omega,\revision{b}}^\pi\right|$$
\normalsize
since the coordinate in $\beta_{\omega,\revision{b}}^\pi$ corresponding to the reward coordinate of $\phi_{\omega,\revision{b}}(s,a)$ is 0.
Since 
\footnotesize
$$\sup_{((s,a),(\omega,\revision{b}))} |\psi_{\omega,\revision{b}}(s,a)| \leq  2 \sup_{(s,a)} \frac{g_\alpha(f(s,a))}{\alpha^d},$$ 
\normalsize
this implies that
\footnotesize
\begin{align*}
\sup_{((s,a),(\omega,\revision{b}))} |\phi_{\omega,\revision{b}}(s,a)\beta_{\omega,\revision{b}}^\pi| =  O\rbr{ \frac{\sup_{(s,a)} g_\alpha(f(s,a))}{\alpha^d(1-\gamma)}}     
\end{align*}
\normalsize
Thus, applying \lemref{lemma:lemma_1_from_kitchen_sink}, 
we have that
\footnotesize
\begin{align}\label{eq:thm7_step2}
    \left\|Q^\pi - \Pi_{\nu,\Phi_{\mathrm{rf}}} Q^\pi\right\|_{\nu} = \tilde{O}\left(\frac{\gamma \sup_{s,a} g_\alpha(f(s,a))}{\alpha^d(1-\gamma)\sqrt{m}} \right).
\end{align}
\normalsize
Substituting (\ref{eq:thm7_step2}) into (\ref{eq:thm7_step1}) finishes the proof.
\end{proof}
Proposition \ref{thm:approx_error} shows with high probability, the approximation error for random features shrinks at a rate of $O\rbr{m^{-1/2}}$ where $m$ is the number of random features. Next is our approximation error result for Nystr\"om features. A key step for doing so is the following proposition, which gives a high-probability approximation error bound for approximating a smooth kernel using Nystr\"om features.


\begin{lemma}[Kernel approximation error using Nystr\"om features]
    \label{prop:kernel_approx_main_paper}
Consider the Mercer decomposition (on $\mathcal{S}$) of $k_\alpha(\cdot,\cdot)$:
\begin{align}
    k_\alpha(x,x') = \sum_{i=1}^\infty \sigma_i e_i(x)e_i(x'),
\end{align}
where $\{e_i\}_{i=1}^\infty$ forms a countable orthonormal basis for $L_2(\mu_{\mathrm{nys}})$ with corresponding \revision{eigenvalues} $\{\sigma_i\}_{i=1}^\infty$.
Let $X^{n_{\mathrm{nys}}} = \{x_i\}_{i=1}^{n_{\mathrm{nys}}}$ be an i.i.d $n_{\mathrm{nys}}$-point sample from $\mu_{\mathrm{nys}}$. In addition, let $\lambda_1 \geq \lambda_2 \geq \dots \geq \lambda_{n_{\mathrm{nys}}}$ denote the eigenvalues of the (unnormalized) Gram matrix $K^{(n_{\mathrm{nys}})}$ in its eigendecomposition $K^{(n_{\mathrm{nys}})}U = \Lambda U$ where $U^\top U = \revision{U}U^\top = I$. Suppose that $\sigma_j, \lambda_j/n_{\mathrm{nys}} \lesssim \exp(-\beta j^{1/h})$ for some $\beta > 0$ and $h > 0$. Suppose that $n_{\mathrm{nys}} \geq 3$ and that $ \lfloor (2\log n_{\mathrm{nys}})^h/\beta^h \rfloor \leq m \leq n_{\mathrm{nys}}$. Consider the rank-$m$ kernel approximation $\hat{k}_{\alpha,m}^{n_{\mathrm{nys}}}$ constructed using Nystr\"om features, defined as follows: 
\begin{align}
\label{eq:k_hat_m_n_varphi_inner_prod_defn}
\hat{k}_{\alpha,m}^{n_{\mathrm{nys}}}(s,t) = \varphi_{\mathrm{nys}}(s)^\top \varphi_{\mathrm{nys}}(t),   
\end{align}
where $\varphi_{\mathrm{nys}}(\cdot) \in \mathbb{R}^m$ and is defined as 
\begin{align}
\label{eq:varphi_nys_defn}
    (\varphi_{\mathrm{nys}})_i(\cdot) := \frac{1}{\sqrt{\lambda_i}} u_i^\top k_\alpha(X^{n_{\mathrm{nys}}}, \cdot), \quad \forall i \in [m],
\end{align}
where $u_i$ denotes the $i$-th column of $U$, and $k_\alpha(X^{n_{\mathrm{nys}}},\cdot)$ denotes a $n_{\mathrm{nys}}$-dimensional vector where $(k_\alpha(X^{n_{\mathrm{nys}}},\cdot))_{\ell} = k_\alpha(x_\ell,\cdot)$; for details on deriving $(\varphi_{\mathrm{nys}})_i$, see \appref{appendix:derivation_of_nystrom}. Then, $\forall \delta > 0$, w.p. at least $1 - \delta$,
\small
\begin{align}
&\int_\mathcal{S} \sqrt{\left(k_\alpha - \hat{k}_{\alpha,m}^{n_{\mathrm{nys}}}\right)(x,x)} d\mu_{\mathrm{nys}}(x)  =  \tilde{O}\left(\frac{1}{n_{\mathrm{nys}}} \right).\label{eq:high_prob_kernel_approx_error_asssuming_decay_main_paper} 
\end{align}
\normalsize
\end{lemma}
We defer the proof of \lemref{prop:kernel_approx_main_paper} to Appendix \ref{appendix:nystrom}. 
We note here that the $\Otil(1/n_{\mathrm{nys}})$ high probability error decay rate is a novel result, outperforming the $\Otil(1/\sqrt{n_{\mathrm{nys}}})$ rate in \cite{hayakawa2023sampling}.
We achieve this improvement by combining Bernstein's inequality and results in local Rademacher complexity theory \cite{bartlett2005local}. 
Using Lemma \ref{prop:kernel_approx_main_paper}, we can then demonstrate the following approximation error for the Nystr\"om features.
\begin{proposition}[$Q$-Approximation error with Nystr\"om features]
\label{thm:approx_error_of_nystrom}
Suppose all the assumptions in \lemref{prop:kernel_approx_main_paper} hold. Suppose also that we pick the sampling distribution $\mu_{\mathrm{nys}}$ such that $\mu_{\mathrm{nys}}(x) = p_\alpha(x).$
We define the feature map $\Phi_{\mathrm{nys}}$ for the Nystr\"om features as follows: 
\begin{align}
    \label{eq:nystrom_features_phi_definition}
    \Phi_{\mathrm{nys}} = \{[\psi_{\mathrm{nys}}(s,a), r(s,a)]\}_{(s,a) \in \mathcal{S} \times \mathcal{A}}, 
\end{align}
where $\psi_{\mathrm{nys}}(s,a) \in \mathbb{R}^m$ is defined in \eqref{eq:nystrom_features_phi_definition} in Algorithm \ref{alg:spectral_control}. Then, for any $\delta > 0$, with probability at least $1 - \delta$,
\small
\begin{align*}
& \ \left\|Q^\pi - \tilde{Q}_{\mathrm{nys}}^\pi\right\|_{\nu} \leq \  \tilde{O}\left(\frac{\gamma \tilde{g}_{\alpha}}{(1-\gamma)^2 n_{\mathrm{nys}}} \right) \leq  \tilde{O}\left(\frac{\gamma \tilde{g}_{\alpha}}{(1-\gamma)^2 m} \right), 
\end{align*}
\normalsize
where $\|\cdot\|_{\nu}$ is the $L_2$ norm defined as $ \left\|f\right\|_{\nu} = \int f^2 d \nu$, $\tilde{Q}_{\Phi_{\mathrm{nys}}}^\pi$ is defined in \eqref{eq:tilde_Q_phi_general_defn}, and $\tilde{g}_{\alpha} := \sup_{s,a} \frac{g_{\alpha}(f(s,a))}{\alpha^d}$. 
\end{proposition}
\begin{proof}
    Similar to the analysis at the start of the proof of \propref{thm:approx_error}, with the contraction property and results from \cite{yu2008new}, we have that 
    \footnotesize
    \begin{align}\label{eq:thm_nystrom_step1}
    \left\|Q^\pi - \tilde{Q}_{\Phi_{\mathrm{nys}}}^\pi\right\|_{\nu} \leq \frac{1}{1-\gamma} \left\|Q^\pi - \Pi_{\nu,\Phi_{\mathrm{nys}}} Q^\pi\right\|_{\nu},
\end{align}
\normalsize
where $\Pi_{\nu,\Phi_{\mathrm{nys}}}$ is defined as in \eqref{eq:pi_nu_phi_projector_operator}. It remains for us to bound $\left\|Q^\pi - \Pi_{\nu,\Phi_{\mathrm{nys}}} Q^\pi\right\|_{\nu}$. First, 
by \eqref{eq:k_hat_m_n_varphi_inner_prod_defn}, recalling the definition of $\varphi_{\mathrm{nys}}(\cdot)$ in \eqref{eq:varphi_nys_defn}, we have
\footnotesize
\begin{align}
& \ \hat{k}_m^{n_{\mathrm{nys}}}\left(s',\frac{f(s,a)}{1 - \alpha^2}\right) =  \varphi_{\mathrm{nys}}(s')^\top \varphi_{\mathrm{nys}}\rbr{\frac{f(s,a)}{(1-\alpha^2)}},
\end{align}
\normalsize
as our Nystr\"om approximation of $k_{\alpha}\rbr{s',\frac{f(s,a)}{1-\alpha^2}}$. 
Since we have
$\psi_{\mathrm{nys}}(s,a) = g_\alpha(f(s,a))\varphi_{\mathrm{nys}}\rbr{\frac{f(s,a)}{1-\alpha^2}},$
where $\psi_{\mathrm{nys}}(\cdot)$ is defined as in \eqref{eq:alg_phi_nystrom_defn}, this motivates us to consider the following $Q$-value approximation:
$$\resizebox{0.95\hsize}{!}{$\hat{Q}_{\mathrm{nys}}^\pi(s,a) \! :=\! r(s,a) + \gamma \psi_{\mathrm{nys}}(s,a)^\top \!\left(\!\frac{\int_{\mathcal{S}} \varphi_{\mathrm{nys}}(s')V^\pi(s') p_\alpha(s')
ds'}{\alpha^d}\!\right)\!,$}$$
which we note this is a feasible solution to the objective 
\footnotesize
\begin{align*}
\mathop{\arg\min}_{f\in\mathrm{span}(\Phi_{\mathrm{nys}})} \mathbb{E}_{\nu}\left(Q^{\pi}(s, a) - f(s, a)\right)^2.
\end{align*}
\normalsize 
Thus, we have 
\footnotesize
\begin{align*}
& \ \left\|\Pi_{\nu,\Phi_{\mathrm{nys}}}(Q^\pi) - Q^\pi\right\|_\nu \leq \left\|\hat{Q}_{\mathrm{nys}}^{\pi} - Q^\pi\right\|_\nu\\
= & \ \gamma \left\|\psi_{\mathrm{nys}}(s,a)^\top \left( \frac{\int_{\mathcal{S}} \varphi_{\mathrm{nys}}(\sqrt{1-\alpha^2}s')V^\pi(s') p_\alpha(s')
ds'}{\alpha^d}\right) \right. \\
& \quad \quad  - \left. \int_{\mathcal{S}} P(s' \mid s,a) V^\pi(s') ds' \right\|_\nu   \\
= & \  \gamma\!\left\|\left(\frac{g_{\alpha}(\!f(s,a)\!)}{\alpha^d}\right)\left(\int_{\mathcal{S}} \varphi_{\mathrm{nys}}\!\!\left(\!\!\frac{f(s,a)}{1\!-\!\alpha^2}\!\right)^\top\!\!\!\varphi_{\mathrm{nys}}(s') V^\pi(s') dp_\alpha(s') \right. \right.\\
&\quad \quad  \left. \left. - \int_{\mathcal{S}} k_\alpha(s', \frac{f(s,a)}{1\!-\!\alpha^2}) V^\pi(s') dp_\alpha(s') \right)   \right\|_\nu  \\
\leq & \ \gamma \left(\frac{V_{\max} g_{\alpha,\mathrm{sup}}}{\alpha^d}\right) \times \\
& \quad  \underbrace{\left\| \int_{\mathcal{S}} \!\left|\hat{k}_m^{n_{\mathrm{nys}}}\!\left(s',\frac{f(s,a)}{1\!-\!\alpha^2}\right) \!-\! k_\alpha\left(s' , \frac{f( s,a)}{1\!-\!\alpha^2}\right)\!\right| p_\alpha(s') ds' \!\right\|_\nu}_{T_1}\!,\stepcounter{equation}\tag{\theequation}\label{eq:Khat-K_integral_under_p}
\end{align*}
\normalsize
where $g_{\alpha,\mathrm{sup}} := \sup_{s,a} g_\alpha (f(s,a))$. By our choice of $\mu_{\mathrm{nys}} = p_\alpha(x)$, continuing from \eqref{eq:Khat-K_integral_under_p}, we have
\footnotesize
\begin{align*}
    &\ T_1 = \ \left\| \int_{\mathcal{S}} \left|\hat{k}_m^{n_{\mathrm{nys}}}\left(s',\frac{f(s,a)}{1-\alpha^2}\right) - k_\alpha\left(s' , \frac{f( s,a)}{1-\alpha^2}\right)\right| p_\alpha(s') ds' \right\|_\nu \\
    = & \ \!\left\|\!\int_{\mathcal{S}} \left|\left\langle \!k_\alpha\left(\!\frac{f(s,a)}{1\!-\!\alpha^2},\cdot\right), \hat{k}_m^{n_{\mathrm{nys}}}(s',\cdot \!)\! -\! k_\alpha(s',\cdot)\!\right\rangle_{\!\!\!\mathcal{H}_{k_\alpha}} \! \!\right| dp_\alpha(s') \!\right\|_\nu \\
    \leq & \ \!\left\|\int_{\mathcal{S}} \!\sqrt{k_\alpha\!\left(\!\frac{f(s,a)}{1-\alpha^2},\frac{f(s,a)}{1\!-\!\alpha^2}\!\right)}  \!\sqrt{(k_\alpha \!-\! \hat{k}_m^{n_{\mathrm{nys}}})(s',s')} dp_\alpha(s') \!\right\|_\nu \\
    \leq & \ \left\|\int_{\mathcal{S}}  \sqrt{(k_\alpha - \hat{k}_m^{n_{\mathrm{nys}}} )(s',s')} dp_\alpha(s') \right\|_\nu
\end{align*}
\normalsize
Note to move from the second last line to the last line, we used the fact that $k_{\alpha}(\cdot,\cdot) \leq 1.$
Using Proposition \ref{prop:kernel_approx_main_paper}, by utilizing the decay assumption on the eigenvalues of both the Mercer expansion and empirical gram matrix, we thus have that for any $\delta > 0$, with probability at least $1 - \delta$,
\footnotesize
\begin{align*}
T_1  \leq  \ \left(\int_{\mathcal{S}}  \sqrt{(k_\alpha - \hat{k}_m^{n_{\mathrm{nys}}})(s',s')} \mu(s') ds' \right) = \tilde{O}\left(\frac{1}{n_{\mathrm{nys}}} \right).
\end{align*}
\normalsize
We conclude by combining \eqref{eq:Khat-K_integral_under_p} and the bound on $T_1$.
\end{proof}
\begin{remark}
    Proposition \ref{thm:approx_error_of_nystrom} shows that using the Nystr\"om method improves the approximation error to $O(\rbr{n_{\mathrm{nys}}^{-1}})$, where $n_{\mathrm{nys}}$ is the number of samples we sample to construct the Gram matrix used to build the Nystr\"om features. \revision{We emphasize that an important condition for this to hold is the assumption in the above proposition that the number of features $m$ satisfy $ m \geq \frac{(2\log n_{\mathrm{nys}})^h}{\beta^h}$.} The rate $O(\rbr{n_{\mathrm{nys}}^{-1}})$ is always an improvement over $O(1/m)$ where $m$ is the number of features, since by design \algabb chooses $m \leq n_{\mathrm{nys}}$. The only assumption required for this improvement is that the eigenvalues of the empirical Gram matrix and Mercer expansion satisfy certain decay assumptions, which is a natural assumption to make in the kernel literature (cf. Theorem 2 in \cite{belkin2018approximation}). 
\end{remark}


\textbf{Statistical Error:} We now provide the bound on the statistical error due to using a finite number of samples $n$. We note that this \revision{ is a general result that applies to any features $\phi(s,a)$, inclding both Nystr\"om and random features, that satisfies Assumptions 1 to 3.}\footnote{\revision{Similar analysis of statistical error of policy evaluation with linear function approximation also exist in the literature~\cite{abbasi2019politex}. 
}}
\begin{proposition}\label{thm:stat_error}
For each policy $\pi$ encountered in the algorithm, let $\hat{Q}_{\Phi,T}^\pi$ denote the policy given by $ \hat{Q}_{\Phi, T}^\pi(s,a) = \phi(s,a)^\top \hat{w}_T,$ where $\hat{w}_T$ is defined as in \eqref{eq:projected_least_square}, and $\Phi$ can either be the reward concatenated with the Nystr\"om or random features, i.e. either $\Phi_{\mathrm{nys}}$ or $\Phi_{\mathrm{rf}}$. \revision{Let $\phi_{\sup} := \sup_{s,a}\|\phi(s,a)\|$}. Then, for sufficiently large $n$, there exists an universal constant $C>0$ independent of $m$, $n$, $T$ and $(1-\gamma)^{-1}$, such that with probability at least $1-\delta$, we have
\footnotesize
    \begin{align}
        \left\|\tilde{Q}_\Phi^\pi - \hat{Q}^\pi_{\Phi,  T}\right\|_{\nu} \leq \gamma^T \left\|\tilde{Q}_\Phi^\pi\right\|_{\nu} + \frac{C \phi_{\sup}^6 \mathrm{polylog}(m, T/\delta)}{(1-\gamma) \Upsilon_1^2\Upsilon_2 \sqrt{n}},
    \end{align}
    \normalsize
    where we recall $\tilde{Q}_\Phi^\pi$ is defined as in \eqref{eq:tilde_Q_phi_general_defn}.
\end{proposition}

\begin{proof}
We define $\tilde{w}$ which satisfies the condition
\footnotesize
\begin{align}
    & \tilde{w} \!=\!  \left(\mathbb{E}_{\nu} \left[\phi(s, a) \phi(s, a)^\top\right]\right)^{-1} \label{eq:tilde_w_defn} \\
    & 
    \resizebox{0.8\hsize}{!}{$
    \left(\mathbb{E}_{\nu}\left[\!\phi(s, a) \!\left(r(s, a)\! +\! \gamma\! \mathbb{E}_{(s^\prime, a^\prime) \sim P(s, a) \times \pi} \!\left[\phi(s^\prime,\! a^\prime)^\top \tilde{w}\right]\!\right)\!\right]\right),$}  \nonumber
\end{align}
\normalsize
and let $\tilde{Q}_\Phi^\pi(s, a) = \phi(s, a)^\top \tilde{w}$. It is straightforward to see that $\tilde{w}$ is the fixed point of the population (\ie, $n\to\infty$) projected least square update \eqref{eq:projected_least_square}. 
For notational convenience, we drop the $\Phi$ in the subscript of $\Pi_{\nu,\Phi}$. 
With the update \eqref{eq:projected_least_square}, we have that 
\footnotesize
\begin{align*}
    \left\|\Phi(\tilde{w} - \hat{w}_{t+1})\right\|_{\nu} \leq &  \gamma \left\|\Phi(\tilde{w} - \hat{w}_{t})\right\|_{\nu} + \left\|\left(\Pi_{\nu} - \hat{\Pi}_{\nu}\right) r\right\|_{\nu} \nonumber\\
    & + \gamma \left\|\left(\Pi_{\nu} P^{\pi} - \hat{\Pi}_{\nu} \hat{P}^\pi\right) \Phi \hat{w}_t\right\|_{\nu},
\end{align*}
\normalsize
where we use the contractivity under $\|\cdot\|_{\nu}$. Telescoping gives
\footnotesize
\begin{align*}
    \left\|\Phi(\tilde{w}\! -\! \hat{w}_{T})\right\|_{\nu} \leq & \gamma^T \left\|\Phi(\tilde{w} - \hat{w}_{0})\right\|_{\nu}\! + \!\frac{1}{1-\gamma}\left\|\left(\Pi_{\nu}\! -\! \hat{\Pi}_{\nu}\right)\! r\right\|_{\nu} \\
    & + \frac{\gamma}{1-\gamma} \max_{t\in [T]}\left\|\left(\Pi_{\nu} P^{\pi} - \hat{\Pi}_{\nu} \hat{P}^\pi\right) \Phi \hat{w}_t\right\|_{\nu}.
\end{align*}
\normalsize
With the concentration of the second and third terms as shown in~\appref{appendix:stat_error}, we conclude the proof.
\end{proof}
Proposition \ref{thm:stat_error} shows the statistical error of the linear parts can be decomposed into two parts. The first part $\gamma^T \left\|\tilde{Q}^\pi\right\|_{\nu}$ is due to the fact that we start from $\hat{w}_0 = 0$ and shrinks as the number of least square policy evaluation iterations $T$ increases. The second part $\frac{C \tilde{g}_\alpha^6 m^3 \mathrm{polylog}(m, T/\delta)}{(1-\gamma) \Upsilon_1^2\Upsilon_2 \sqrt{n}}$ denotes the statistical error due to the finite sample and shrinks as the number of samples $n$ increases. We can balance the two parts with $T = \Theta(\log n)$ and obtain an estimation error that shrinks as $O(n^{-1/2})$ with high probability.

\textbf{Total Error for Policy Evaluation:} Combining the approximation error in~\propref{thm:approx_error} and~\propref{thm:approx_error_of_nystrom} and the statistical error in~\propref{thm:stat_error}, we have the following bound on the error for least square policy evaluation:
\begin{theorem}
    \label{thm:total_error_evaluation}
    Let $T = \Theta(\log n)$. With probability at least $1-\delta$, we have that for the random features $\Phi_{\mathrm{rf}}$,
    \footnotesize
    \begin{align}
    \label{eq:policy_eval_error_rf}
        \left\|Q^\pi - \hat{Q}_{\Phi_{\mathrm{rf}}, T}^\pi\right\|_{\nu} = \tilde{O} \left(\underbrace{\frac{\tilde{g}_{\alpha}}{(1\!-\!\gamma)^2\sqrt{m}}}_{\mathrm{approx.\ error}} + \underbrace{\frac{\tilde{g}_\alpha^6 m^3}{(1\!-\!\gamma) \Upsilon_1^2\Upsilon_2 \sqrt{n}}}_{\mathrm{stat.\ error}}\right), 
    \end{align}
    \normalsize
    and for the Nystr\"om features
    \footnotesize
    \begin{align}
        &\left\|Q^\pi - \hat{Q}_{\Phi_{\mathrm{nys}}, T}^\pi\right\|_{\nu} = \tilde{O} \left(\underbrace{\frac{\tilde{g}_{\alpha}}{(1-\gamma)^2 n_{\mathrm{nys}}}}_{\mathrm{approx.\ error}} + \underbrace{\frac{\tilde{g}_\alpha^6 \revision{n_{\mathrm{nys}}^3/\lambda_m^3}}{(1-\gamma) \Upsilon_1^2\Upsilon_2 \sqrt{n}}}_{\mathrm{stat.\ error}}\right).  \label{eq:policy_eval_error_nystrom}
    \end{align}
    \normalsize
\end{theorem}
\begin{proof}
    This directly comes from the triangle inequality, \ie,
    \footnotesize
    \[
    \nbr{Q^\pi - \hat{Q}_{\Phi,T}^\pi}_\nu \le 
    \nbr{Q^\pi - \tilde{Q}_\Phi^\pi}_\nu +  \nbr{\tilde{Q}_\Phi^\pi - \hat{Q}_{\Phi,T}^\pi}_\nu. 
    \]
    \normalsize
    The first approximation error term is bound using \propref{thm:approx_error} and \propref{thm:approx_error_of_nystrom} for random and Nystr\"om features respectively, while the second statistical error term can be bound using~\propref{thm:stat_error}, and \revision{applying the result from Appendix G that for random features, $\sup_{s,a}\|\phi(s,a)\| = O(\tilde{g}_\alpha \sqrt{m})$, and that for Nystrom features, $\sup_{s,a} \| \phi(s,a)\| = O(\tilde{g}_\alpha \sqrt{\frac{n_{\mathrm{nys}}}{\lambda_m}})$, where $\lambda_m$ is the $m$-th largest eigenvalue of the Gram matrix $K^{(n_{\mathrm{nys}})}$}.
\end{proof}
Theorem~\ref{thm:total_error_evaluation} provides the estimation error of $Q$ with least square policy evaluation under the $\|\cdot\|_{\nu}$ norm for both the random and Nystr\"om features. 
Meanwhile, it also performs as a cornerstone of control optimality analysis, as we will show in the next section. The bound also reveals a fundamental tradeoff between the approximation error and statistical error: for random features, a larger number $m$ of features will make the finite kernel truncation capable of approximating the original infinite-dimensional function space better (cf. the $\tilde{O}\left( \frac{1}{\sqrt{m}}\right)$ approximation error term in the RHS of \eqref{eq:policy_eval_error_rf}) but it also requires a larger number of learning samples $n$ in the policy evaluation step in order to train the weights well (cf. the $\tilde{O}\left( \frac{m^3}{\sqrt{n}}\right)$ statistical error term in the RHS of \eqref{eq:policy_eval_error_rf}). While the statistical error term in Nystr\"om features remains the same, the approximation error term here (cf. \eqref{eq:policy_eval_error_nystrom}) improves to $\tilde{O}(1/n_{\mathrm{nys}}) \leq \tilde{O}(1/m)$ (which holds since the number of features we pick, $m$, is always less than $n_{\mathrm{nys}}$ which is the number of samples used to obtain the Nystr\"om features), smaller than the $\tilde{O}(1/\sqrt{m})$ approximation error term in \eqref{eq:policy_eval_error_rf}.  
\subsection{Error Analysis of Natural Policy Gradient for Control}\label{subsec:control_analysis}

We now provide the error analysis for policy optimization. We start with the following "regret lemma" that is adapted from Lemma 34 in \cite{agarwal2021theory}. We remark that in the following proof we assume the action space $\revision{\Acal}$ is finite for simplicity, however, the proposed~\algabb is applicable for infinite action space without any modification. 


\begin{lemma}[\revision{cf. Lemma 34 in \cite{agarwal2021theory}}]
\label{lem:regret}
\revision{Let $\phi_{\sup} := \sup_{s,a} \|\phi(s,a)\|$.} For a fixed comparison policy $\tilde{\pi}$, we have
\begin{align}
\label{eq:v_min_gap}
\resizebox{0.88\hsize}{!}{$
    \min\limits_{k < K} \left\{V^{\tilde{\pi}} - V^{\pi_k}\right\} = O\left(\frac{1}{1-\gamma}\left(\frac{\log|\mathcal{A}|}{\eta K} + \frac{\eta \revision{\phi_{\sup}^4}}{\Upsilon_2^2} + \sum_{k=0}^{K-1}\frac{\mathrm{err}_k}{K}\right)\right),$}
\end{align}
where (noting that $A^{\pi_k}(s,a) := Q^{\pi_k}(s,a) - V^{\pi_k}(s)$, and $\hat{w}_{k,T}$ is defined in \eqref{eq:projected_least_square}) $\mathrm{err}_k$ is defined as
\begin{align*}
\resizebox{1.0\hsize}{!}{$
    \mathrm{err}_k \defeq \mathbb{E}_{s\sim d^{\tilde{\pi}}, a\sim \tilde{\pi}(\cdot|s)}\left[A^{\pi_k}(s, a) - \hat{w}_{k, T}^{\top}\left(\phi(s, a) - \mathbb{E}_{a^\prime\sim\pi_k(s)}\left[\phi(s, a^\prime)\right]\right)\right]$}.
\end{align*}
\end{lemma}
\begin{proof}
    We provide a brief sketch here. 
    Note that $\|\phi(s, a)\|_2 = O(\revision{\phi_{\sup}})$. 
    Thus, by Remark 6.7 in \cite{agarwal2021theory}, we know that $\log \pi(a|s)$ is smooth with the smoothness parameter $\beta = O(\revision{\phi_{\sup}^2})$, with which we have
    \footnotesize
    \begin{align*}
        & \mathbb{E}_{s\sim d^{\tilde{\pi}}} \left[\mathrm{KL}(\tilde{\pi}(\cdot|s)\|\pi_k(\cdot|s)) -\mathrm{KL}(\tilde{\pi}(\cdot|s)\|\pi_{k+1}(\cdot|s)) \right]\\
        = & \mathbb{E}_{s\sim d^{\tilde{\pi}}, a\sim \tilde{\pi}(\cdot|s)} \left[\log\frac{\pi_{k+1}(a|s)}{\pi_k(a|s)}\right]\\
        \geq & \eta\mathbb{E}_{s\sim d^{\tilde{\pi}}, a\sim \tilde{\pi}(\cdot|s)}\left[\left(\nabla_{\theta} \log \pi_k(a|s)\right)^\top \hat{w}_{k}^T\right] - \frac{\eta^2 \beta}{2}\left\|\hat{w}_k^T\right\|_2^2\\
        = & \eta\mathbb{E}_{s\sim d^{\tilde{\pi}}, a\sim \tilde{\pi}(\cdot|s)}\left[A^{\pi_k}(s, a)\right] - \frac{\eta^2 \beta}{2}\left\|\hat{w}_k^T\right\|_2^2\\
        & + \eta\mathbb{E}_{s\sim d^{\tilde{\pi}}, a\sim \tilde{\pi}(\cdot|s)}\left[\left(\nabla_{\theta} \log \pi_k(a|s)\right)^\top \hat{w}_{k}^T - A^{\pi_k}(s, a)\right] \\
        = & (1-\gamma) \eta \left(V^{\tilde{\pi}} - V^{\pi_k}\right) - \frac{\eta^2 \beta}{2}\left\|\hat{w}_k^T\right\|_2^2 - \eta \mathrm{err}_k,
    \end{align*}
    \normalsize
    where second inequality comes from smoothness condition and the last step comes from the performance difference lemma (Lemma 3.2 in \cite{agarwal2021theory}). We can then complete the proof by a telescoping sum, and also use the fact that for sufficiently large number of learning samples $n$, the learnt $\hat{w}_{k,T} \approx \tilde{w}$ (where $\tilde{w}$ is defined in \eqref{eq:tilde_w_defn}) and  that $\|\tilde{w}\| = O(\phi_{\sup}\Upsilon_2^{-1})$.
    More details are provided in \appref{appendix:regret}.
\end{proof}
Given an arbitrary policy $\tilde{\pi}$, for $K$ steps of the policy update step in \algabb (yielding $\{V^{\pi_k}\}_{k=1}^K$), this lemma bounds a difference term where the largest of the $\{V^{\pi_k}\}_{k=1}^K$'s is subtracted from $V^{\tilde{\pi}}$. We note that the first term of the bound in \eqref{eq:v_min_gap} decreases at the rate of $O\rbr{K^{-1}}$, with the second term in the bound in \eqref{eq:v_min_gap} coming from the smoothness parameter (see \appref{appendix:regret} for details) and the third term in the bound in \eqref{eq:v_min_gap} arising from the estimation error in the policy evaluation step.
We can eventually balance the first and second terms to select the optimal stepsize $\eta$, while the third term is irreducible and inherited from the policy evaluation step. Now, we can provide the performance guarantee for control optimality with natural policy gradient.
\begin{theorem}\label{thm:opt_policy}
    Let \small $\eta = \Theta\left(\Upsilon_2 \tilde{g}_\alpha^{-2} m^{-1} \sqrt{\log|\mathcal{A}|}\right)$ \normalsize \revision{for random features and \small$\eta = \Theta\left(\Upsilon_2 \tilde{g}_\alpha^{-2} n_{\mathrm{nys}}^{-1}\lambda_m \sqrt{\log|\mathcal{A}|}\right)$ \normalsize  for Nystrom features}. With high probability, for the $V_{\Phi_{\mathrm{rf}}}^{\pi_k}$ learnt using random features $\Phi_{\mathrm{rf}}$, we have
    \begin{align}
        & 
        \resizebox{0.4\hsize}{!}{$\min_{k < K} \left\{V^{\pi^*} - V_{\Phi_{\mathrm{rf}}}^{\pi_k}\right\}
        = $}\\
        &\resizebox{0.96\hsize}{!}{$\tilde{O}\left(\frac{\tilde{g}_\alpha^2 m}{\Upsilon_2}\sqrt{\frac{\log|\mathcal{A}|}{K}} 
        + \frac{1}{1-\gamma} \sqrt{\max_{s, a, \pi, k} \frac{d^{\pi^*}(s)\pi(a|s)}{\nu_{\pi_k}(s, a)}} \left(\frac{\tilde{g}_{\alpha}}{(1\!-\!\gamma)^2\sqrt{m}} + \frac{\tilde{g}_{\alpha}^6 m^3}{(1\!-\!\gamma) \Upsilon_1^2\Upsilon_2 \sqrt{n}}\right)\right)$}.\nonumber 
    \end{align}
    Meanwhile for the $V_{\Phi_{\mathrm{nys}}}^{\pi_k}$ learnt using Nystr\"om features $\Phi_{\mathrm{nys}}$, we have
    \begin{align}
        & 
        \resizebox{0.4\hsize}{!}{$\min_{k < K} \left\{V^{\pi^*} - V_{\Phi_{\mathrm{nys}}}^{\pi_k}\right\}
        = $}\\
        &\resizebox{0.96\hsize}{!}{$\tilde{O}\left(\revision{\frac{n_{\mathrm{nys}} \lambda_m^{-1} \tilde{g}_\alpha^2 }{\Upsilon_2}}\sqrt{\frac{\log|\mathcal{A}|}{K}} 
        + \frac{1}{1-\gamma} \sqrt{\max_{s, a, \pi, k} \frac{d^{\pi^*}(s)\pi(a|s)}{\nu_{\pi_k}(s, a)}} \left(\frac{\tilde{g}_{\alpha}}{(1-\gamma)^2 n_{\mathrm{nys}}} + \frac{\tilde{g}_{\alpha}^6 \revision{n_{\mathrm{nys}}^3/\lambda_m^3)}}{(1-\gamma) \Upsilon_1^2\Upsilon_2 \sqrt{n}}\right)\right)$}.\nonumber 
    \end{align}
\end{theorem}
\begin{proof}
This is proved by setting $\tilde{\pi} := \pi^*$ in~\lemref{lem:regret}, selecting the appropriate $\eta$ as our choice in the lemma, and characterizing the term $\mathrm{err}_k$ as in~\appref{appendix:opt_policy}. \revision{Finiteness of the policy distribution ratio is ensured by the coverage conditions in Assumption 3.}
\end{proof}
We emphasize that~\thmref{thm:opt_policy} characterizes the gap between optimal policy and the solution provided by~\algabb, taking in account of finite-step in policy optimization ($K$), finite-dimension ($m$) in the feature space, finite number of samples used to generate Nystr\"om features $(n_{\mathrm{nys}})$ and finite-sample ($n$) approximation in policy evaluation, respectively, which, to the best of our knowledge, is established for the first time. 
As we can see, for the random features, when $m$ increases, we can reduce the approximation error, but the optimization and sample complexity will also increase accordingly. For Nystr\"om features, this tradeoff may be better especially since the approximation term scales as $O(1/n_{\mathrm{nys}})$), and depending on the spectrum eigendecay, $m$ may be chosen to be significantly smaller than $n_{\mathrm{nys}}$. In either case, we can further balance the terms for the optimal dimension of features. 
\newrevision{Finally, we note that the final bound has a polynomial dependence on $\tilde{g}_{\alpha}$, where $\tilde{g}_\alpha := \sup_{s,a}(g_\alpha(f(s,a)))\alpha^{-d}$, where $g_\alpha(f(s,a)) := \exp\left(\frac{\alpha^2 \| f(s,a)\|^2}{2(1-\alpha^2)\sigma^2} \right)$. While this term scales exponentially with $\frac{c_f^2}{\sigma^2}$, where $c_f := \sup_{s,a}\|f(s,a)\|$,
$\alpha$ can be picked to be $\tilde{O}(\sigma/c_f)$ to cancel out this exponential dependence, leading $\tilde{g}_\alpha$ to instead be on the order of $\alpha^{-d} = O\left(\frac{c_f^d}{\sigma^d} \right)$, which for small to medium dimensional problems can be significantly better than the dependence $\Omega\left(\exp\left(\frac{c_f^2}{\sigma^2}\right)\right)$.}

%% file: simulations.tex
 \vspace{-0.5em}
\section{Simulations}\label{sec:simulations}
 \vspace{-0.25em}
\subsection{Details of SDEC implementation}
 \vspace{-0.25em}
In this section, we will empirically implement the SDEC and compare it with other methods. As we discussed in Remark~\ref{rem:sdec-policy}, the advantage of the proposed spectral dynamics embedding is that it is compatible with a wide range of dynamical programming/reinforcement learning methods, allowing us to take advantage of the computational tools recently developed in deep reinforcement learning. Though in \algabb we stick to a particular type of policy evaluation and policy update for the purpose of theory development,  
in our empirical implementation of SDEC, we combine spectral dynamical embedding with Soft Actor-Critic (SAC) \cite{haarnoja2018soft}, a cutting-edge deep reinforcement learning method that has demonstrated strong empirical performance.
Specifically, we use random features or Nystr\"om to parameterize the critic, and use this as the critic in SAC. In SAC, it is necessary to maintain a parameterized function $Q(s,a)$ which estimates the soft $Q$-value (which includes not just the reward but also an entropy term encouraging exploration). For a given policy $\pi$, the soft $Q$-value satisfies the relationship 
$    Q^\pi(s_t,a_t) = r(s_t,a_t) + \gamma \mathbb{E}_{s_{t+1} \sim p}[V(s_{t+1})],$
where 
$    V(s_t) = \mathbb{E}_{a_t \sim \pi} [Q(s_t,a_t) - \revision{\beta} \log \pi(a_t \mid s_t)].$
Based on the spectral dynamics embedding proposed in our paper, we parameterize the $Q^\pi$-function as 
 $   Q^\pi(s,a) = r(s,a) + \tilde{\phi}_{\omega}(s,a)^\top \tilde{\theta}^\pi $
where $\tilde{\phi}_{\omega}(s,a)$ is our feature approximations. We use two types of feature approximations, random features and Nystr\"om features. The random features are composed by 
$\tilde{\phi}_{\omega}(s,a) = \begin{bmatrix}
\cos(\omega_1^\top s'+b_1), \cdots , \cos(\omega_m^\top s' + b_m)
\end{bmatrix}^\top,
$
where $s' = f(s,a)$, with $\{\omega_i\}_{i=1}^m\sim \Ncal(0,I_d)$  and $\{b_i\}$ drawn iid from $\mathrm{Unif}([0,2\pi])$.
The $\tilde{\theta}^\pi \in \mathbb{R}^m$ is updated using (mini-batch) gradient descent for~\eqref{eq:projected_least_square}. For the Nystr\"om feature, we first sample $m$ data $s_1,\dots,s_m$ uniformly in the state space to construct the corresponding Gram matrix $K\in\mathbb R^{m\times m}$. Then we do spectral decomposition of $K$ to get orthogonal eigenvectors $U$ and diagonal matrix $\Lambda$. The Nystr\"om features are composed by $
\tilde{\phi}_{\omega}(s,a) = [k(s',s_1)\dots,k(s',s_m)] U \Lambda^{-1/2}$. We then use this $Q$-function for the actor update in SAC. 

\vspace{-1mm}

 \vspace{-2mm}
\subsection{Experimental Setup}\label{subsec:pendulum_comparison}
 \vspace{-1mm}
We implemented the proposed SDEC algorithm on four nonlinear control tasks, Pendulum swingup, pendubot balance, cartpole swingup, and 2D drone hovering. Pendulum and Cartpole are two classic nonlinear control tasks, the tasks are swinging up the pendulum to the upright position. Pendubot is another classic underactuated system where two rods are linked together with an un-acturated joint~\citep{spong1995pendubot}. The task is to balance the two rods in the upright position. The 2D drone task simplifies drones to a 2D plane with two motors. The task is hovering the drone at a fixed position. 

 \vspace{-0.5em}

\subsection{Performance versus other nonlinear controllers}
\vspace{-1mm}

 We compare \algabb with iterative LQR (iLQR)~\citep{sideris2005efficient} and Koopman-based control~\citep{korda2018linear},  Energy-based control~\citep{fantoni2000energy,aastrom2000swinging}, and \revision{soft actor-critic (SAC) \cite{haarnoja2018soft}}. \revision{We note that iLQR, the energy-based controller and SDEC all make use of the knowledge of the environment dynamics. DKUC and SAC do not make use of the environment dynamics.} The first two are well-known alternatives for nonlinear control, and the energy-based control is a Lyapunov-based method to design stabilizing controllers. 
 For iLQR, we used the implementation in \cite{roulet2022iterative}, where we added a log barrier function to account for the actuator constraint.
Meanwhile, for the Koopman-based control, we adapted an implementation called Deep KoopmanU with Control (DKUC) proposed in the paper \cite{shi2022deep}, where we combine the dynamics learned from DKUC with MPC to enforce the input constraint. We modified the energy-based swing-up controller from \cite{aastrom2000swinging}. 
For SDEC and DKUC, we trained using 4 different random seeds, and both algorithms have access to the same amount of environment interaction. 
\begin{table*}[htbp]
\caption{Performance comparison. Bold numbers are the best average performance over all algorithms.}\label{tab:performance}
\begin{center}
\begin{tabular}{p{1.8cm}p{1.5cm}p{1.5cm}p{1.6cm}p{1.6cm}p{1.6cm}p{1.6cm}p{1.5cm}p{1.5cm}}
\toprule
& Pendulum & Pendulum & Pendubot & Pendubot  & 2D Drones & 2D Drones & Cartpole & Cartpole \\
&          & (Noisy)&                   &          (Noisy) &           &           (Noisy)&           &           (Noisy)\\
\midrule  
\multirow{4}{*}{}SDEC(RF)  & -279.0$\pm$175 & -240$\pm$173 & {-0.793$\pm$0.190} & -3.883$\pm$0.283 & -5.074$\pm$2.023 &  -7.423$\pm$2.549&{-212$\pm$116}&-211$\pm$93\\
SDEC(Nystr\"om)  & $\textbf{-245.0$\pm$174}$ & $\textbf{-225$\pm$156}$ & -2.635$\pm$1.072 & {-3.880$\pm$1.192} & -2.315$\pm$0.323 & $\textbf{-4.424$\pm$0.924 }$ &-216$\pm$118&\textbf{-207$\pm$90}\\
\hline
 iLQR & -1084$\pm$544 & -1191$\pm$415 & -429.6$\pm$113.8 & -1339$\pm$605  & $\textbf{-0.973$\pm$0.841}$  & 753.7$\pm$119.7 & -430$\pm$87.3&-705$\pm$59\\
 DKUC & -1090$\pm$436 & -1021$\pm$546 & -446.0$\pm $54.2 & -330.8$\pm $140.1 & -750.8$\pm$427.7 &  -867.0$\pm$94.9 & -486$\pm$189 & -774$\pm$309  \\
 Energy-based &-895.8$\pm $128 & -1604$\pm  $288 & -0.797$\pm $0.588 & -16.63$\pm$5.62 & -322.9$\pm$14.6 & -339.8$\pm$35.3 & -779$\pm$247 & -694$\pm$91\\
 SAC & -250$\pm$173 & -261$\pm$179 & \textbf{-0.295$\pm$0.132} & \textbf{-0.751$\pm$0.458} & -10.87$\pm$12.15 & -20.06$\pm$14.82 & \textbf{-206$\pm$100} & -214$\pm$96 \\
\bottomrule 
\end{tabular}

\label{table:pendulum_comparison}
\end{center}
\vspace{-5 mm}
\end{table*}   

\begin{figure}[htp]
    \centering
    \includegraphics[width=0.495\linewidth]{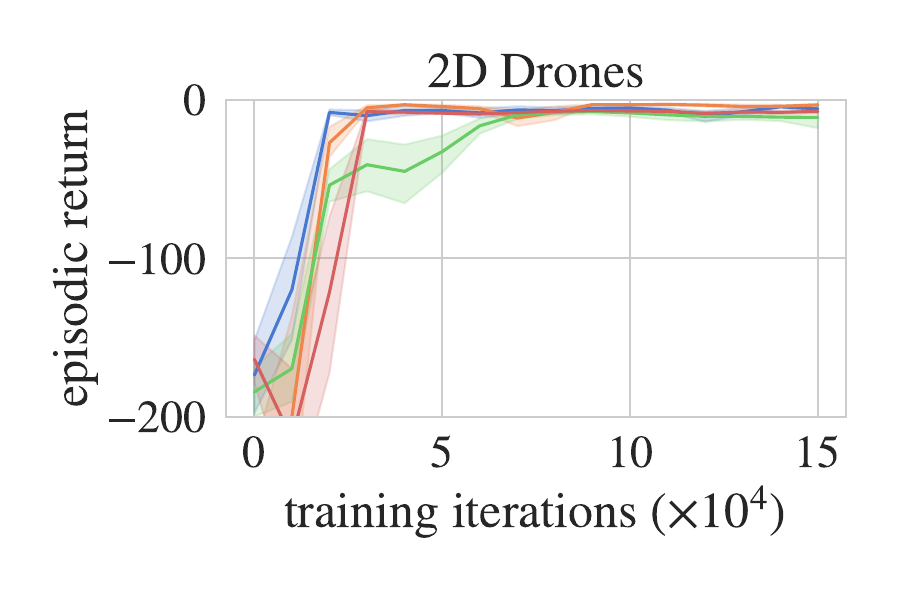}
    \includegraphics[width=0.49\linewidth]{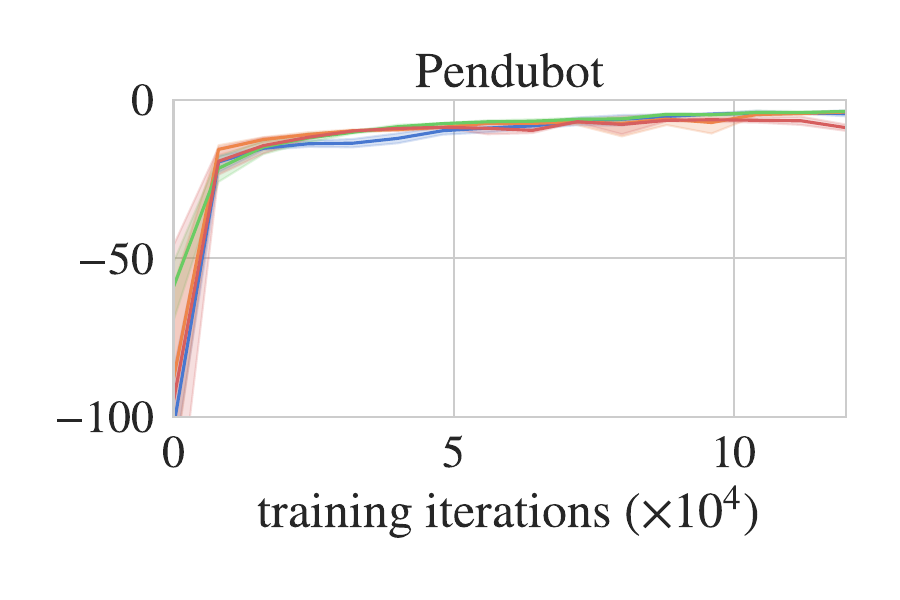}
    \includegraphics[width=0.8\linewidth]{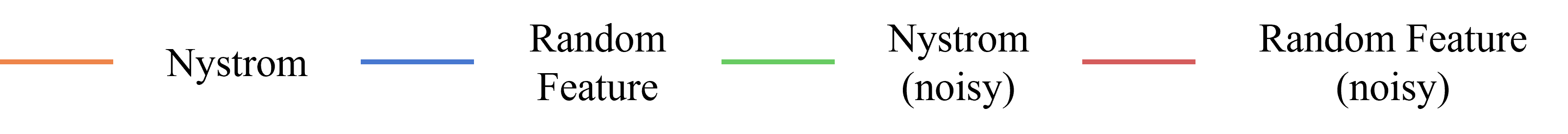}
    \caption{SDEC performances with random features and Nystr\"om on noisy and non-noisy environments. } 
    \label{fig:train_curves}
\end{figure}

\begin{figure}[htp]
\begin{center}
\includegraphics[height = 0.3\linewidth]{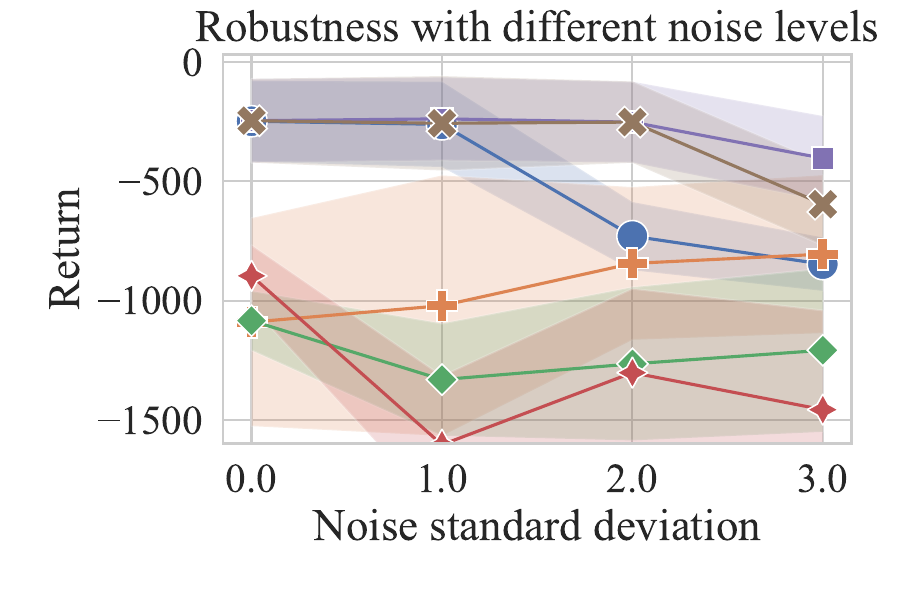}
\includegraphics[height = 0.3\linewidth]{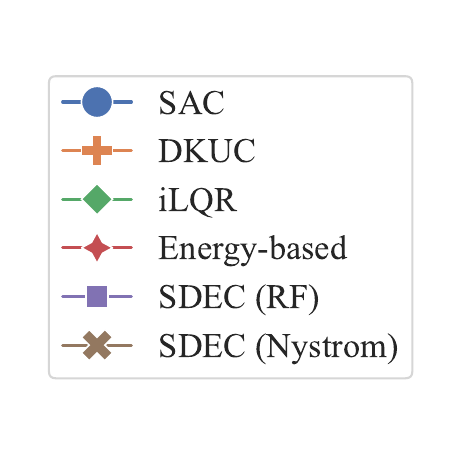}
\caption{Performances of all algorithms on Pendulum with varying noise levels $\sigma \in \{0,1,2,3\}$. 
}
\label{fig:pendulum}
\end{center}
\vspace{-3mm}
\end{figure}
\begin{figure}[h]
    \centering
\includegraphics[width=0.48\linewidth]{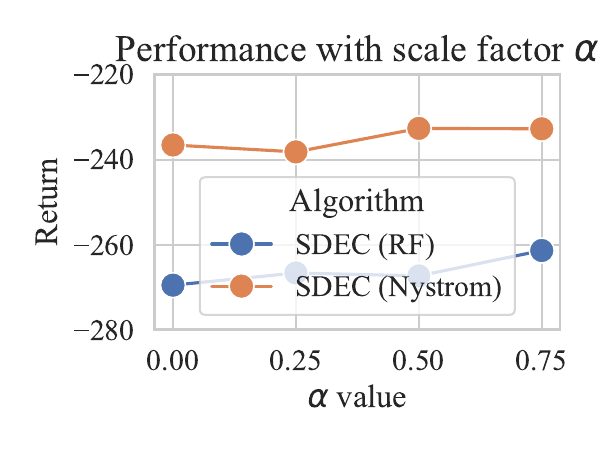}
\includegraphics[width=0.48\linewidth]{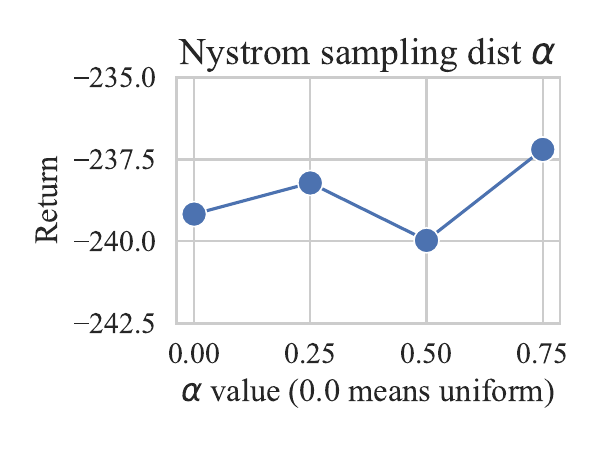}
\vspace{-8pt}
    \caption{SDEC performance on the Pendulum with different $\alpha$ in scaling $f(s, a)$ (left) and in Nystrom sampling distribution (right).}
    \label{fig:tuning_alpha}
\end{figure}
\begin{figure}[h]
    \centering
    \includegraphics[width=\linewidth]{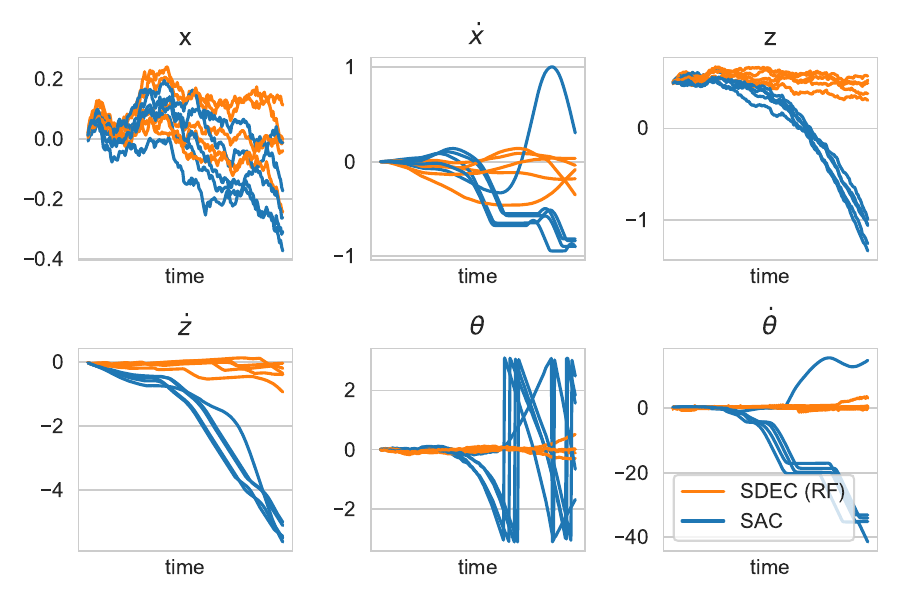}
    \caption{5 trajectories of 2D quadrotor controlled by SDEC and SAC.}
    \label{fig:quad_traj}
\end{figure}

The performance of the various algorithms in both the noiseless and noisy setting (with $\sigma = 1$) can be found in Table~\ref{table:pendulum_comparison}. 
The mean and standard deviation (in brackets) over 4 random seeds are shown. The number of random and Nystr\"om features used is 512 for pendulum, 4096 for pendubot and 2D Drones and 8192 for Cartpole. We observe that the SDEC strongly outperforms other nonlinear controllers in the noisy setting (higher reward is better; in this case the negative sign is due to the rewards representing negative cost). For the noiseless setting, SDEC are comparable to the best nonlinear baseline controllers.
For comparison with SAC, we observe different results depending on whether the system is fully actuated. For the pendulum and 2D drones where the systems are fully actuated, we observe significant advantages of the proposed SDEC over baselines, especially in noisy environments. For the Cartpole and pendubot that are underactuated, the proposed SDEC can still outperform other nonlinear control baselines but achieve comparable performance with SAC.
We further plot the performance of \algabb during the course of its learning in Figure~\ref{fig:train_curves} and \ref{fig:pendulum}. The evaluation is performed every 25 learning episodes on a fresh evaluation set of 100 episodes, and the y-axis represents the average episodic reward on the evaluation set. The shaded regions represent a 1 standard deviation confidence interval (across 4 random seeds). Figure~\ref{fig:train_curves} compares performance with  different features and noise levels. In both the noiseless and noisy settings, the performance of SDEC continuously improves until it saturated after a few episodes. Figure~\ref{fig:pendulum} shows the evaluation performance on the pendulum swingup task during training with noisier environments. The performance still continuously improves and then saturates. Performance with different finite approximation dimensions are compared in Figure~\ref{fig:dim}. As dimension increases, the performance improves. Nystr\"om features are consistently better than random features and benefit more from dimension increase, which aligns with our theoretical results. 

\begin{figure}[htbp]
    \centering
    \includegraphics[width=0.45\linewidth]{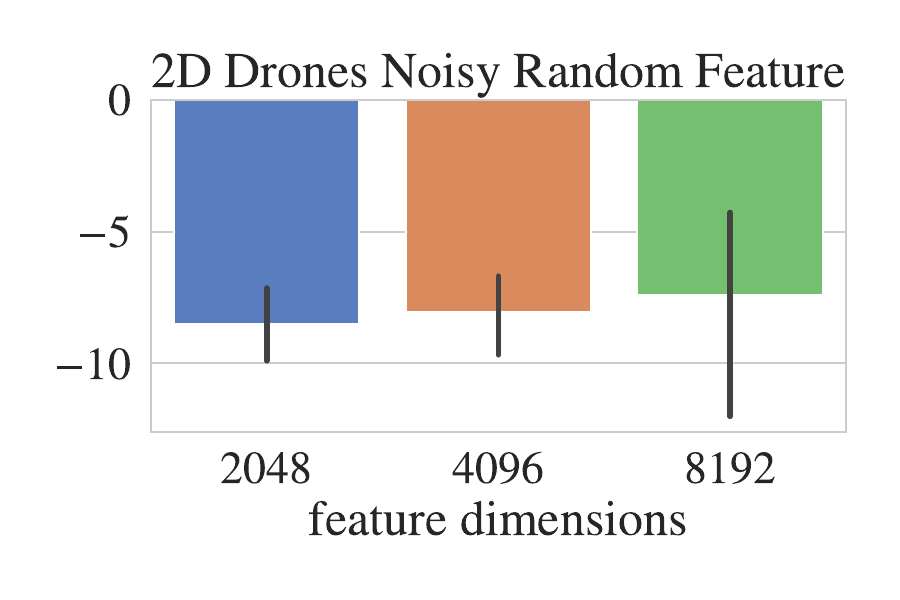}
    \includegraphics[width=0.45\linewidth]{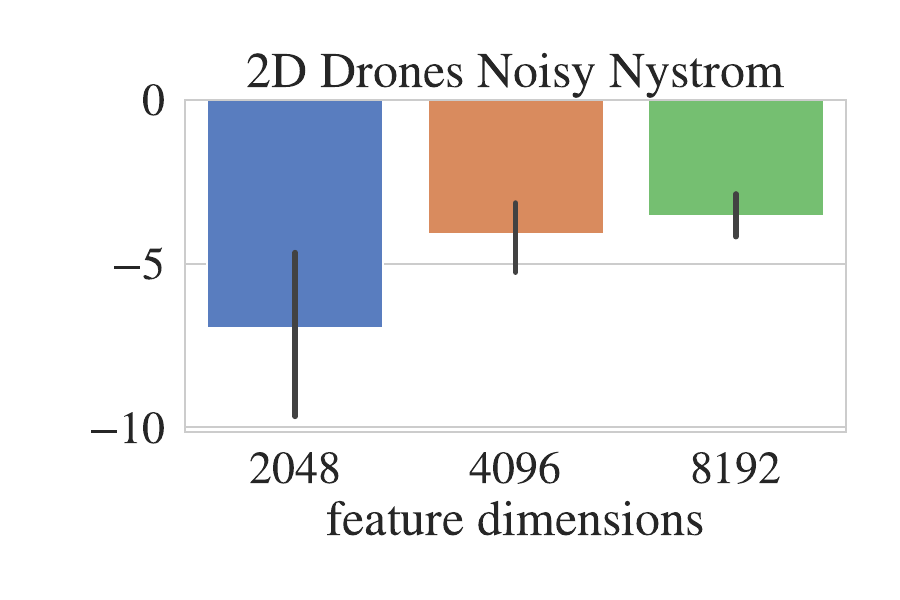}
    \caption{Performance after training with different feature dimensions on noisy 2D drones with random and Nystr\"om features.}
    \label{fig:dim}
\end{figure}

\vspace{-1.5em}
\subsection{Robustness and Stability}

To evaluate the robustness of the proposed algorithm, We show the performance under different noise levels in Figure \ref{fig:pendulum}. We can see that both two SDEC algorithms show good robustness against increasing noise levels, while SAC has significantly worse performance as the noise level increases. As for all the nonlinear control baselines, they are not able to obtain good performance with noise dynamics. We can study stability using the results in Table \ref{tab:performance}. The reward functions are designed as the squared error with respect to the equilibrium state. We see the proposed \algabb shows good stability, especially with noise. We also show 5 trajectories of noisy 2D drones controlled by SDEC and SAC in Figure \ref{fig:quad_traj}. The goal is to stabilize the drone at the initial position $[0.0, 0.5, 0.0]$. The initial state is the origin and the noise level is $\sigma=1.0$. The six plots show six drone states including the 2D position and the roll angle $[x, z,\theta]$, and their time derivatives. The trajectories show that SDEC can stabilize the drone even with the presence of noise. However, with SAC, the drones quickly become unstable and flip over due to the noise.

%% file: conclusion.tex
\vspace{-1em}
\section{Conclusion}


Optimal control for stochastic nonlinear dynamics is a long-standing problem. In this work, we reveal the spectral dynamics embedding, upon which we obtain an practical and provable control algorithm,~\algabb. Meanwhile, we close the gap between the theoretical understanding and empirical success of the algorithm. 
We carefully characterize the policy evaluation and optimization error for finite-dimensional random and Nystrom features, theoretically validating the practical usage of~\algabb. Along the way, we have also developed new analysis tools and results that may be of independent interest, e.g. the novel $O(1/n_{\mathrm{nys}})$ high-probability convergence rate for Nystr\"om kernel approximation. 


%% file: appendix.tex
\section*{{Appendix}}
\setcounter{subsection}{0}
\subsection{Random Feature for Exponential Family Transition}\label{appendix:rf_ebm}

With the transition operator in~\eqref{eq:exp_family}, we have
\footnotesize
\begin{align*}
    &P(s'|s, a) = \exp\rbr{A\rbr{s, a}}\exp\rbr{f\rbr{s, a}^\top \zeta\rbr{s^\prime}}p(s')\\
    =&\underbrace{\exp\rbr{A\rbr{s, a}}\exp\rbr{\frac{\nbr{f(s, a)}^2}{2}}}_{h\rbr{s, a}}\\
    &\cdot
    \exp\rbr{-\frac{\nbr{f\rbr{s, a} -  \zeta\rbr{s^\prime}}^2}{2}} 
    \underbrace{\exp\rbr{\frac{\nbr{\zeta\rbr{s'}}^2}{2}}p(s')}_{g\rbr{s'}}\\
    =& \EE_{\omega}\sbr{\inner{h\rbr{s, a}\psi_\omega\rbr{f(s, a)}}{\mu_\omega\rbr{s'}g(s')}},
\end{align*}
\normalsize
where $\psi_{\omega,\revision{b}}\rbr{f(s, a)} \!=\! \cos\rbr{\omega^\top f\rbr{s, a} + \revision{b}}$ and $\mu_{\omega,\revision{b}} = \cos\rbr{\omega^\top \zeta\rbr{s'} + \revision{b}}$
with $\omega\!\!\sim\!\!\Ncal\rbr{0, \Ib}$ and \revision{$b \sim U([0,2\pi])$}.

\subsection{Derivation of Nystr\"om method}
\label{appendix:derivation_of_nystrom}

Consider $k_\alpha(x,y) = \exp\left(\frac{-(1-\alpha^2)\|x-y\}^2}{2\sigma^2} \right)$, and note that it is a bounded and continuous positive definite kernel on a compact space $\mathcal{S}$.\footnote{The Nystr\"om derivation here applies for any general kernel $k$ satisfying these conditions.} Consider any probability measure $\mu$ on $\mathcal{S}$, e.g. the $\mu_{\mathrm{nys}}$ considered in Algorithm \ref{alg:spectral_control}. As a consequence of Mercer's theorem, there exist eigenvalues $\{\sigma_j\}_{j=1}^\infty$ and (orthonormal) eigenvectors $\{ e_j\}_{j=1}^\infty \subset \mathcal{L}_2(\mu)$ such that for any $j \in \mathcal{N}$ and any $x \in \mathcal{S}$, 
\begin{align}
    \int_{\mathcal{S}} k_\alpha(x,y) e_j(y) d\mu(y) = \sigma_j e_j(x). \label{eq:eigenfunction_problem}
\end{align}
The Nystr\"om method is a way to approximate the Mercer eigendecomposition $k_\alpha(x,y) = \sum_{j=1}^\infty \sigma_j e_j(x) e_j(y)$ using a finite sum 
$$\hat{k}_{\alpha,m}^{n}(x,y) =\sum_{i=1}^m \hat{\sigma}_i \hat{e}_i(x) \hat{e}_i(y) \label{eq:k_hat_nystrom_decomp}$$ for some positive integer $m$, where the $(\hat{\sigma}_i, \hat{e}_i)$ pairs are found through a
numerical approximation of the eigenfunction problem in  \eqref{eq:eigenfunction_problem}. Let us now describe the procedure.

Recall that $n_{\mathrm{nys}}$ denotes the number of samples used to create the Nystr\"om features. For notational convenience, in the sequel, throughout the parts relevant to the Nystr\"om approximation results in the Appendix, unless otherwise specified, we denote $n:= n_{\mathrm{nys}}$. We emphasize this is not be to confused with the $n$ used to denote the number of samples used in statistical learning of the $Q$-function. Suppose we draw $n$ independent samples, $X^n = \{x_s\}_{s=1}^n$ from $\mathcal{S}$ following the distribution $\mu$. For any eigenfunction $e_j$, a numerical approximation of the eigenfunction problem in \eqref{eq:eigenfunction_problem} is then
\begin{subequations}
\begin{align}
    &\frac{1}{n}\sum_{s=1}^n k_\alpha(x,x_s) e_j(x_s) \approx \sigma_j e_j(x), \label{eq:numerical_eigen_decomp_1}\\
    &\frac{1}{n} \sum_{s=1}^n (e_j(x_s))^2 \approx 1. \label{eq:numerical_eigen_decomp_2}
\end{align}
\end{subequations}
Let $K^{(n)} \in \mathbb{R}^{n \times n}$ denote the Gram matrix where $(K^n)_{rs} = k_\alpha(x_r,x_s)$, and let $K^{(n)} U = \Lambda U$ denote the eigendecomposition of $K^{(n)}$ whereby $U$ is orthogonal and $\Lambda$ is diagonal. For \eqref{eq:numerical_eigen_decomp_1} to hold, it should hold for any $x \in X^n$. So, for any $r \in [n]$, we would expect
\begin{align}
    \frac{1}{n}\sum_{s=1}^n k_\alpha(x_r,x_s) e_j(x_s) \approx \sigma_j e_j(x_r)
\end{align}
for any eigenfunction/value pair $(e_j(\cdot), \sigma_j)$ of $k_\alpha(\cdot,\cdot)$.
For any eigenvector $u_i$ of $K^{(n)}$, 
we note that 
$$\frac{1}{n}\sum_{s=1}^n k_\alpha(x_r,x_s) (u_i)_s = \frac{\lambda_i}{n} (u_i)_r.$$
Thus, it is natural to extend $u_i \in \mathbb{R}^n$ to be an eigenfunction $\tilde{e}_i$ in the following way, by enforcing $\tilde{e}_i(x_s) = (u_i)_s$ for any $x_s \in X^n$, and setting (for all other $x$)
$$\tilde{e}_i(x) := \frac{1}{\lambda_i}\sum_{s=1}^n k_\alpha(x,x_s) (u_i)_s,$$
with associated eigenvalue $\frac{\lambda_i}{n}$.
Recalling the orthonormality constraint in \eqref{eq:numerical_eigen_decomp_2}, since 
$$ \frac{1}{n} \sum_{s=1}^n \tilde{e}_i(x_s)^2 = \frac{1}{n} \sum_{s=1}^n (u_i)_s^2 = \frac{1}{n},$$
we scale 
$$\hat{e}_i(x) := \sqrt{n} \tilde{e}_i(x)= \frac{\sqrt{n}}{\lambda_i}\sum_{s=1}^n k_\alpha(x,x_s) (u_i)_s,$$
and consider the Nystr\"om approximation $\hat{k}_\alpha^{n}$ of the original $k_\alpha$, where
\begin{align}
\hat{k}_\alpha^{n}(x,y) = & \  \sum_{i=1}^n \frac{\lambda_i}{n} \hat{e}_i(x) \hat{e}_i(y) \nonumber \\
= & \  k_\alpha(x,X^n) \left(\sum_{i=1}^n \frac{1}{\lambda_i}u_iu_i^\top \right) k_\alpha(X^n,x), \label{eq:k_hat_nystrom_approx}
\end{align}
where the last equality can be verified by plugging in the definition of $\hat{e}_i(\cdot)$, and using the notation that $k_\alpha(x,X^n)$ denotes a row vector whose $s$-th component is $k_\alpha(x,x_s),$ and $k_\alpha(X^n,y)$ denotes a column vector whose $s$-th component is $k_\alpha(y,x_s)$. Based on \eqref{eq:k_hat_nystrom_approx}, for any $m \leq n$, we then define $\hat{k}_{\alpha,m}$ to be the rank-$m$ approximation
\begin{align}
\label{eq:k_hat_nystrom_approx_m}
\hat{k}_{\alpha,m}^n(x,y) = k_\alpha(x,X^n) \left(\sum_{i=1}^m \frac{1}{\lambda_i}u_iu_i^\top \right) k_\alpha(X^n,x),
\end{align}
which is the Nystr\"om kernel approximation that we will consider in the paper. Since it can be shown that $\hat{k}_{\alpha,m}^n(x,y) = \varphi_{\mathrm{nys}}(x)^\top \varphi_{\mathrm{nys}}(y)$, where
\begin{align}
(\varphi_{\mathrm{nys}})_i(\cdot) = \frac{1}{\sqrt{\lambda_i}}u_i^\top k_\alpha(X^n,\cdot), \quad \forall i \in [m], \label{eq:varphi_m_exact_form}
\end{align}
we may view $\varphi_{\mathrm{nys}}(\cdot) \in \mathbb{R}^m$ as the (rank-$m$) Nystr\"om features associated with the Nystr\"om kernel approximation.

\subsection{Detailed Proof for Proposition \ref{prop:kernel_approx_main_paper} (Kernel approximation error of Nystr\"om features)}
\label{appendix:nystrom}

By \lemref{thm:RKHS_reproducing_kernel_moore_aronszajn}, in the RKHS corresponding to $k_{\alpha}$, which we denote as $\mathcal{H}_{k_\alpha}$, given any $x,y \in \mathcal{S}$, $\langle k_\alpha(x,\cdot), k_\alpha(y,\cdot) \rangle_{\mathcal{H}_{k_\alpha}} = k_\alpha(x,y)$. Let $\phi(\cdot) \in \mathbb{R}^m$ be as defined in \eqref{eq:varphi_m_exact_form}. Note that for each $i \in [m]$, we can view $(\varphi_{\mathrm{nys}})_i(\cdot)$ as an element in $\mathcal{H}_{k_\alpha}$. This and several other useful properties of the features $\{(\varphi_{\mathrm{nys}})_i(\cdot)\}_{i=1}^m$ is shown in this next result (cf. \cite{hayakawa2023sampling}). 


\begin{lemma}
    \label{lemma:Khat_inner_product}
    Let $(\varphi_{\mathrm{nys}})_i(\cdot) := \frac{1}{\sqrt{\lambda_i}} u_i^\top k_\alpha(X^n,\cdot)$ for each $i \in [m]$. Then, the following statements hold.
    \begin{enumerate}[labelindent=0pt]
        \item The set $\left\{(\varphi_{\mathrm{nys}})_i(\cdot) \right\}_{i=1}^m$ is orthonormal in $\mathcal{H}_{k_\alpha}$.
        \item Define an orthogonal projection map $P_m^{X^n}: \mathcal{H}_{k_\alpha} \to \mathcal{H}_{k_\alpha}$ onto the span of $\left\{(\varphi_{\mathrm{nys}})_i(\cdot) \right\}_{i=1}^m$ , such that 
        \footnotesize
        $$P_m^{X^n}(h(\cdot)) = \sum_{i=1}^m \langle h(\cdot),(\varphi_{\mathrm{nys}})_i(\cdot) \rangle_{\mathcal{H}_{k_\alpha}} (\varphi_{\mathrm{nys}})_i(\cdot).$$
        \normalsize
        Then, $P_m^{X^n} k_\alpha(x,\cdot) = \hat{k}_{\alpha,m}^n(x,\cdot)$ for every $x \in \mathcal{S}$. Thus, for every $(x,y) \in \mathcal{S}$, recalling the definition of $\hat{k}_{\alpha,m}^n$ in \eqref{eq:k_hat_nystrom_approx_m}, we have
        \footnotesize
        \begin{align*}
        \hat{k}_{\alpha,m}^n(x,y) = &  \langle k_\alpha(x,\cdot), \hat{k}_{\alpha,m}^n(\cdot,y) \rangle_{\mathcal{H}_{k_\alpha}} \\
        = &  \langle P_m^{X^n} k_\alpha(x,\cdot), P_m^{X^n} k_\alpha(y,\cdot) \rangle_{\mathcal{H}_{k_\alpha}}
        \end{align*}
        \normalsize
        \item Let $P_{m,\perp}^{X^n}$ denote the orthogonal projection onto the complement of $\{(\varphi_{\mathrm{nys}})_i(\cdot)\}_{i=1}^m$ in $\mathcal{H}_{k_\alpha}$. Then, 
        \footnotesize
        \begin{align}
            \|P_{m,\perp}^{X^n} k_\alpha(x,\cdot)\|_{\mathcal{H}_{k_\alpha}} = \sqrt{(k_\alpha - \hat{k}_{\alpha,m}^n)(x,x)} 
\label{eq:K_minus_Khat_equal_P_orthogonal_complement_norm}
        \end{align}
        \normalsize
    \end{enumerate}
\end{lemma}

\begin{proof}
\begin{enumerate}
\item We first prove the orthogonality of the $\{(\varphi_{\mathrm{nys}})_i\}_{i=1}^m$'s. First given any $i,j$ pair, we have 
\footnotesize
    \begin{align*}
        \langle\! (\varphi_{\mathrm{nys}})_i(\!\cdot\!),\! (\varphi_{\mathrm{nys}})_j(\!\cdot\!) \!\rangle_{\!\mathcal{H}_{k_\alpha}}\!\! 
        = &  \frac{1}{\sqrt{\lambda_i \lambda_j}} \langle\! u_i^\top k_\alpha(X^n,\cdot\!), u_j^\top k_\alpha(X^n,\cdot \!)\!\rangle_{\!\mathcal{H}_{k_\alpha}} \\
        = &  \frac{1}{\sqrt{\lambda_i \lambda_j}} u_i^\top k_\alpha(X^n,X^n) u_j  = \delta_{ij}.
    \end{align*}
    \normalsize
\item Next we prove the second stated property. We have
\footnotesize
\begin{align*}
    &P_m^{X^n} k_\alpha(x,\cdot) =   \sum_{i=1}^m \langle k_\alpha(x,\cdot), (\varphi_{\mathrm{nys}})_i(\cdot)\rangle_{\mathcal{H}_{k_\alpha}} (\varphi_{\mathrm{nys}})_i(\cdot) \\
    &=   \sum_{i=1}^m \left\langle k_\alpha(x,\cdot), \frac{u_i^\top k_\alpha(X^n,\cdot)}{\sqrt{\lambda_i}}\right\rangle_{\mathcal{H}_{k_\alpha}}\!\!\!\frac{1}{\sqrt{\lambda_i}}u_i^\top k_\alpha(X^n,\cdot) \\
     & = k_\alpha(x,X^n) \sum_{i=1}^m \frac{1}{\lambda_i} u_i u_i^\top k_\alpha(X^n,\cdot) = \hat{k}_{\alpha,m}^n(x,\cdot),
\end{align*}
\normalsize
as desired. Thus, 
\footnotesize
\begin{align*}
\hat{k}_{\alpha,\!m\!}^n(\!x,y\!) \!\!=\!\!   \langle k_\alpha(\!x,\cdot\!), \hat{k}_{\alpha,\!m\!}^n(\cdot,y) \rangle_{\!\mathcal{H}_{k_\alpha}\!\!}\!\!\!=\!\!  \langle P_m^{X^n} k_\alpha(\!x,\cdot\!), P_m^{X^n} k_\alpha(\!y,\cdot\!) \rangle_{\!\mathcal{H}_{k_\alpha}}\!\!. 
\end{align*}
\normalsize

\item This follows from the second part of the lemma, since 
\small
\begin{align*}
    &\sqrt{(k_\alpha\! -\! \hat{k}_{\alpha,m}^n)(x,x)} =   \|(k_\alpha \!-\! \hat{k}_{\alpha,m}^n)(x,\cdot)\|_{\mathcal{H}_{k_\alpha}} \\
     & =\! \|k_\alpha(x,\!\cdot\!)\! \!-\!\! P_m^{X^n}(k_\alpha(x,\!\cdot\!)\|_{\mathcal{H}_{k_\alpha}} \!\!=\!\! \|P_{m,\perp}^{X^n} k_\alpha(x,\!\cdot\!)\|_{\!\!\mathcal{H}_{k_\alpha}}\!.
\end{align*}
\normalsize
\end{enumerate}
\end{proof}

We next bound the approximation error of the Nystr\"om method, i.e. find a bound for $\int_{\mathcal{S}} \sqrt{(k_\alpha - \hat{k}_{\alpha,m}^n)(x,x)} d\mu(x)$, which by \eqref{eq:K_minus_Khat_equal_P_orthogonal_complement_norm} is equal to $\int_{\mathcal{S}} \|P_{m,\perp}^{X^n}k_\alpha(x,\cdot)\|_{\mathcal{H}_K} d\mu(x)$. To do so, we need a few preliminary results. We first introduce some additional notation. For any distance $d$ and class of functions $\mathcal{F}$, we define 
\footnotesize
\begin{align*}
&\mathcal{N}(\epsilon, \mathcal{F},d) := \min_{G \in \mathcal{G}} |G|, \ \ \mathcal{G} := \{G: \forall f \in \mathcal{F}, \exists g \in G \mbox{ s.t. } d(f,g) \leq \epsilon \},
\end{align*}
\normalsize
and $|G|$ denotes the cardinality of the set $G$.
In addition, for any function $f \in \mathcal{L}_2(\mu)$, we denote $\mu(f) := \int_{\mathcal{S}} f(x) d\mu(x)$, and $\mu_{X^n}(f) = \frac{1}{n}\sum_{i=1}^n f(x_i)$. Given any $f, g \in \mathcal{L}_2(\mu)$, we also define $d_{L_2(\mu_{X^n})}(f,g) := \sqrt{\frac{1}{n} \sum_{i=1}^n (f(x_i) - g(x_i))^2}$.  The first result below, \revision{which utillizies local Rademacher complexity~\cite{bartlett2005local}}, is based on Corollary 3.7 in \cite{bartlett2005local}.

\begin{lemma}
    \label{lemma:corollary_3.7_but_real_value}
    Let $\mathcal{F}$ be a class of $[0,1]$-valued functions such that for any $\epsilon > 0$
    \footnotesize
    $$\log \mathcal{N}(\epsilon, \mathcal{F}, d_{L_2(\mu_{X^n})}) \leq S\log(1 + \revision{4}/\epsilon).$$
    \normalsize
 Then, there exists a constant $c > 0$ such that for any $\delta > 0$ and any $C > 0$, with probability at least $1 - \delta$, for any $f \in \mathcal{F}$,
 \footnotesize
 $$ \mu(f) \leq \frac{C}{C-1}\mu_{X^n}f + cC\left(\frac{S \log(n/S)}{n} + \frac{\log(1/\delta)}{n} \right).$$
 \normalsize
\end{lemma}
\begin{proof}
    The result follows naturally by using the log covering number bound in the proof of Corollary 3.7 in \cite{bartlett2005local}.
\end{proof}
\revision{We note that local Rademacher complexity is critical in deriving the $O(1/n)$ convergence rate above. Standard Rademacher complexity bounds~\cite{wainwright2008graphical} can only ensure a $O(1/\sqrt{n})$ convergence rate, which is strictly worse.}
The next result we need is this lemma, which builds on Lemma \ref{lemma:corollary_3.7_but_real_value}.

\begin{lemma}
\label{proposition:mu_upp_bound_const_mu_hat_plus_O(1/n)}
    Suppose $\Pi_r$  is an arbitrary deterministic $r$-dimensional orthogonal projection in $\mathcal{H}_{k_\alpha}$.  Then, there exists a constant $c > 0$ such that for any constant $C > 1$ and orthogonal projection $P$ in $\mathcal{H}_{k_\alpha}$ possibly depending on $X^n$, we have that for any $\delta > 0$ , with probability at least $1 - \delta$,
    \footnotesize
    \begin{align*}
        & \mu(\|P\Pi_r k_\alpha(x,\cdot)\|_{\mathcal{H}_{k_\alpha}}) \leq \frac{C}{C-1} \mu_{X^n}(\|P\Pi _r k_\alpha(x,\cdot)\|_{\mathcal{H}_{k_\alpha}})  \\
        & \quad \quad + c C \left(\frac{r^2 \max\{\log(n/r^2),1\}}{n} + \frac{\log(1/\delta)}{n}  \right).
    \end{align*}
    \normalsize
\end{lemma}
\begin{proof}
    First, given any orthogonal projection $P$ in $\mathcal{H}_{k_\alpha}$, let $\{u_i\}_{i \in I}$ denote an orthonormal basis for $P\mathcal{H}_{k_\alpha}$, and let $P_\perp$ denotes the orthogonal projection onto its complement.  In addition, let $\{e_1,\dots,e_r\}$  denote a basis for $\Pi_r \mathcal{H}_{k_\alpha}$.  We note that for any projection $P$ in $\mathcal{H}_{k_\alpha}$, 
    \footnotesize
    \begin{align*}
        & \   \left\|P \left(\sum_{\ell \in [r]} \langle e_\ell, k_\alpha(x,\cdot)\rangle_{\mathcal{H}_{k_\alpha}} e_\ell\right) \right\|_{\mathcal{H}_{k_\alpha}}^2 \\
        & = \left\|\sum_{i \in I} u_i \left\langle u_i, \sum_{\ell \in [r]} e_\ell(x) e_\ell \right\rangle_{\mathcal{H}_{k_\alpha}} \right\|_{\mathcal{H}_{k_\alpha}}^2 \\
        & = \sum_{i \in I} \left(\sum_{\ell \in [r]} e_\ell(x) \langle u_i, e_\ell \rangle_{\mathcal{H}_{k_\alpha}}\right)^2 \\
        & =  \sum_{\ell, s \in [r]} e_{\ell}(x) e_s(x) \sum_{i \in I} \langle u_i, e_\ell \rangle_{\mathcal{H}_{k_\alpha}} \langle u_i, e_s \rangle_{\mathcal{H}_{k_\alpha}} \\
        & =  \sum_{\ell, s \in [r]} e_{\ell}(x) e_s(x) \langle Pe_\ell, Pe_s \rangle_{\mathcal{H}_{k_\alpha}} \\
        = & \ v(x)^\top A_{P, \Pi_r} v(x),
    \end{align*}
    \normalsize
    where $v(x) = (e_1(x),e_2(x),\dots, e_r(x)) \in \mathbb{R}^r$, and $A_{P,\Pi_r} \in \mathbb{R}^{r \times r}$ is a $r$ by $r$ dimensional positive semidefinite matrix where $(A_{P,\Pi_r})_{\ell, s} = \langle Pe_\ell, Pe_s \rangle_{\mathcal{H}_{k_\alpha}}$. We can similarly define a $r \times r$ positive semidefinite matrix $A_{P_\perp,\Pi_r} \in \mathbb{R}^{r \times r}$ where $(A_{P_\perp,\Pi_r})_{\ell,s} = \langle P_\perp e_\ell, P_\perp e_s \rangle_{\mathcal{H}_{k_\alpha}}$. Since
    \footnotesize
\begin{align*}
        \resizebox{0.98\hsize}{!}{$(A_{P,\Pi_r})_{\ell,s} + (A_{P_\perp,\Pi_r})_{\ell,s} = \langle Pe_\ell, Pe_s \rangle_{\mathcal{H}_{k_\alpha}} + \langle P_\perp e_\ell, P_\perp e_s \rangle_{\mathcal{H}_{k_\alpha}} = \langle e_\ell, e_s \rangle_{\mathcal{H}_{k_\alpha}} = \delta_{\ell,s},$}
\end{align*}
\normalsize
    it follows that $A_{P,\Pi_r} \preceq I_{r \times r}$, and hence can be decomposed as $A_{P,\Pi_r}= U^\top U$ for some matrix $U \in \revision{\mathbb{R}^{r \times r}}$ such that $\|U\|_2 \leq 1$. Thus, we can reexpress $\|P\Pi_r \mathcal{H}_{k_\alpha}\|_{\mathcal{H}_{k_\alpha}}$ as
    $\|P\Pi_r k_\alpha(x,\cdot) \|_{\mathcal{H}_{k_\alpha}} = \|Uv(x)\|_2$. Since this holds for any random orthogonal projection $P$, the class of functions $P\Pi_r k_\alpha(x,\cdot)$ is a subset of $F_U := \{f_U(x) = \|Uv(x)\|_2: U \in \mathbb{R}^{r \times r}\}$. We next seek to find a covering number bound for the class $F_U$. We note that \revision{straightforward calculations} show that if we have an \revision{$\epsilon/\sqrt{2}$} covering set $\mathcal{U}_\epsilon$ of $\mathcal{U} := \{U: U \in \mathbb{R}^{r \times r}, \|U\|_2 \leq 1\}$ under the distance $\|\cdot\|_2$, then we have an $\epsilon$-covering $F_\epsilon$ of $F_U$ under the distance $d_{L_2(\mu_{X^n})}$. 
    To see why, note that given any $U \in \mathcal{U}$, if $\|U_\epsilon - U\|_2 \leq \epsilon$, then 
    \footnotesize
\begin{align*}
    &\frac{1}{n} \sum_{i=1}^n (\|Uv(x) \|_2 - \|U_\epsilon v(x)\|_2)^2 \leq \frac{1}{n} \sum_{i=1}^n (\|(U - U_\epsilon)v(x)\|_2^2) \\
    & \leq \epsilon^2 \frac{1}{n}\sum_{i=1}^n 2\|v(x)\|_2^2 \leq  2\epsilon^2 \frac{1}{n} \sum_{i=1}^n k_\alpha(x,x) \leq 2\epsilon^2,
\end{align*}
\normalsize
where for the final inequality we used the fact that $\max_{x} k_\alpha(x,x) \leq 1$. 
Since $\log \mathcal{N}(\revision{\epsilon/\sqrt{2}}, \mathcal{U}, \|\cdot\|_2) \leq r^2 \log(1 + \revision{4}/\epsilon)$ (as $\mathcal{U}$ is the unit ball of $\mathcal{R}^{r^2}$ under the spectral norm; \revision{see Example 5.8 in \cite{wainwright2019high}}), it follows that $\log \mathcal{N}(\epsilon, F_U, d_{L_2(\mu_{X^n})}) \leq r^2 \log(1+\revision{4}/\epsilon)$ as well. Our result then follows by applying Lemma \ref{lemma:corollary_3.7_but_real_value}.
    \end{proof}
We note \lemref{proposition:mu_upp_bound_const_mu_hat_plus_O(1/n)} is a new result, and it is key to obtaining the fast rate of $O\rbr{n_{\mathrm{nys}}^{-1}}$ for Nystr\"om kernel approximation in \lemref{prop:kernel_approx_main_paper}. Now we prove \lemref{prop:kernel_approx_main_paper}. 
\begin{proof}[Proof of \lemref{prop:kernel_approx_main_paper}]
Recall that for any function $f \in \mathcal{L}_2(\mu)$, we denote $\mu(f) := \int_{\mathcal{S}} f(x) d\mu(x)$, and $\mu_{X^n}(f) = \frac{1}{n}\sum_{i=1}^n f(x_i)$. Observe that by \eqref{eq:K_minus_Khat_equal_P_orthogonal_complement_norm}, we have
\footnotesize
\begin{align*}
    \sqrt{(k_\alpha - \hat{k}_{\alpha,m}^n)(x,x)} = \|P_{m,\perp}^{X^n}k_\alpha(x,\cdot)\|_{\mathcal{H}_{k_\alpha}}.
\end{align*}
\normalsize
Thus it suffices for us to bound $\mu(\|P_{m,\perp}^{X^n}k_\alpha(x,\cdot)\|_{\mathcal{H}_{k_\alpha}}) = \int_{\mathcal{S}} \|P_{m,\perp}^{X^n}k_\alpha(x,\cdot)\|_{\mathcal{H}_{k_\alpha}} d\mu(x)$. We follow the idea in the proof of Proposition 9 in \cite{hayakawa2023sampling}, and use the observation that for any $C > 0$ and $f,f_- \in \mathcal{L}_2(\mu)$ where $f_- \leq f$ pointwise, we have
\footnotesize
\begin{align}
    & \mu(f) - C\mu_{X^n}(f) \nonumber \\
    & =  \resizebox{0.95\hsize}{!}{$(\mu(f) - \mu(f_-)) + (\mu(f_-) - C\mu_{X^n}(f_-)) + C(\mu_{X^n}(f_-) - \mu_{X^n}(f))$} \nonumber \\
    & \leq  (\mu(f) - \mu(f_-)) + (\mu(f_-) - C\mu_{X^n}(f_-)), \label{eq:mu_f_mu_f_minus_decomposition}
\end{align}
\normalsize
where we used the fact that $f_- \leq f$ pointwise to derive the inequality above. Now, consider setting 
\footnotesize
\begin{align*}
    & f(x) =\|P_{m,\perp}^{X^n}k_\alpha(x,\cdot)\|_{\mathcal{H}_{k_\alpha}}, \\
    & f_-(x) = \|P_{m,\perp}^{X^n} \Pi_r k_\alpha(x,\cdot)\|_{\mathcal{H}_{k_\alpha}} - \|P_{m,\perp}^{X^n} \Pi_{r,\perp} k_\alpha(x,\cdot) \|_{\mathcal{H}_{k_\alpha}},
\end{align*}
\normalsize
where $\Pi_r$ denotes the orthogonal projection onto $\{e_1,e_2,\dots,e_r\} \subset \mathcal{H}_{k_\alpha}$ which are the first $r$ eigenfunctions in the Mercer expansion of $k_\alpha(x,\cdot)$ (so each $e_i$ satisfies $\|e_i\|_{\mathcal{H}_{k_\alpha}} = 1$), and $\Pi_{r,\perp}$ is the orthogonal projection onto the complement of $\{e_1,\dots,e_r\}$ in $\mathcal{H}_{k_\alpha}$.  We first note that $f_-(x) \leq f(x)$ holds pointwise by the triangle inequality. Continuing from \eqref{eq:mu_f_mu_f_minus_decomposition}, picking $C=2$, we have 
\footnotesize
\begin{align}
    & \ \mu(\|P_{m,\perp}^{X^n} k_\alpha(x,\cdot)\|_{\mathcal{H}_{k_\alpha}}) - 2 \mu_{X^n}(\|P_{m,\perp}^{X^n} k_\alpha(x,\cdot)\|_{\mathcal{H}_{k_\alpha}}) \nonumber \\
\leq & \  \mu(\|P_{m,\perp}^{X^n} k_\alpha(x,\cdot)\|_{\mathcal{H}_{k_\alpha}}) \nonumber\\
& \ - \mu\left(\|P_{m,\perp}^{X^n} \Pi_r k_\alpha(x,\cdot)\|_{\mathcal{H}_{k_\alpha}} - \|P_{m,\perp}^{X^n} \Pi_{r,\perp}k_\alpha(x,\cdot)\|_{\mathcal{H}_{k_\alpha}}\right) \nonumber \\
& \ + \mu\left(\|P_{m,\perp}^{X^n} \Pi_r k_\alpha(x,\cdot)\|_{\mathcal{H}_{k_\alpha}} - \|P_{m,\perp}^{X^n} \Pi_{r,\perp}k_\alpha(x,\cdot)\|_{\mathcal{H}_{k_\alpha}}\right) \nonumber \\
& \ - \resizebox{0.9\hsize}{!}{$2 \mu_{X^n}\left(\|P_{m,\perp}^{X^n} \Pi_r k_\alpha(x,\cdot)\|_{\mathcal{H}_{k_\alpha}}- \|P_{m,\perp}^{X^n} \Pi_{r,\perp} k_\alpha(x,\cdot)\|_{\mathcal{H}_{k_\alpha}}\right)$} \nonumber \\
\leq & \  2\mu\left(\|P_{m,\perp}^{X^n} \Pi_{r,\perp}k_\alpha(x,\cdot)\|_{\mathcal{H}_{k_\alpha}}\right)  \nonumber \\
& \ + \!\! \mu\!\left(\!\|P_{m,\perp}^{X^n} \Pi_r k_\alpha(x,\cdot) \|_{\mathcal{H}_{k_\alpha}}\!\!\right) \!\!-\!  2 \mu_{X^n}\!\left(\|P_{m,\perp}^{X^n} \Pi_r k_\alpha(x,\cdot) \|_{\mathcal{H}_{k_\alpha}}\!\!\right)\!\nonumber \\
& \ + 2\mu_{X^n}\left(\|P_{m,\perp}^{X^n} \Pi_{r,\perp} k_\alpha(x,\cdot)\|_{\mathcal{H}_{k_\alpha}}\right) \label{eq:mu_P_minus_2mu_hat_decomp_3_terms}
\end{align}
\normalsize
From \eqref{eq:mu_P_minus_2mu_hat_decomp_3_terms}, we see that we need to bound three terms. The first term we need to bound is $2\mu(\|P_{m,\perp}^{X^n} \Pi_{r,\perp}k_\alpha(x,\cdot)\|_{\mathcal{H}_{k_\alpha}})$. By nonexpansiveness of $P_{m,\perp}^{X^n}$,
\footnotesize
\begin{align}
    &2\mu(\|P_{m,\perp}^{X^n} \Pi_{r,\perp}k_\alpha(x,\cdot)\|_{\mathcal{H}_{k_\alpha}}) \leq 2\mu(\|\Pi_{r,\perp}k_\alpha(x,\cdot)\|_{\mathcal{H}_{k_\alpha}}) \nonumber \\
    &\leq 2\sqrt{\mu(\Pi_{r,\perp}k_\alpha(x,\cdot)\|_{\mathcal{H}_{k_\alpha}}^2)} = 2\sqrt{\sum_{j > r} e_j(x)^2} = 2\sqrt{\sum_{j > r} \sigma_j}.\label{eq:mu_pi_r_perp_bound}
\end{align}
\normalsize
Next we bound 
\footnotesize
$$\left(\mu(\|P_{m,\perp}^{X^n} \Pi_r k_\alpha(x,\cdot) \|_{\mathcal{H}_{k_\alpha}})  - 2 \mu_{X^n}(\|P_{m,\perp}^{X^n} \Pi_r k_\alpha(x,\cdot) \|_{\mathcal{H}_{k_\alpha}})\right).$$ 
\normalsize
To do so, we apply \lemref{proposition:mu_upp_bound_const_mu_hat_plus_O(1/n)} (by picking $C = 2$ in that proposition) , which shows that there exists $c > 0$ such that for any $\delta > 0$, with probability at least $1 - \delta$,
\footnotesize
\begin{align}
    &\left(\mu(\|P_{m,\perp}^{X^n} \Pi_r k_\alpha(x,\cdot) \|_{\mathcal{H}_{k_\alpha}})  - 2 \mu_{X^n}(\|P_{m,\perp}^{X^n} \Pi_r k_\alpha(x,\cdot) \|_{\mathcal{H}_{k_\alpha}})\right) \nonumber \\
    &\leq c\left(\frac{r^2 \max\{\log(n/r^2),1\}}{n} + \frac{\log(1/\delta)}{n} \right).
    \label{eq:mu_minus_mu_hat_local_rademacher_bdd}
\end{align}
\normalsize
Finally, to bound $\mu_{X^n}(\|P_{m,\perp}^{X^n} \Pi_{r,\perp} k_\alpha(x,\cdot)\|_{\mathcal{H}_{k_\alpha}})$, we use Bernstein's inequality~\cite{wainwright2019high}. \revision{It  suffices to bound $\mu_{X^n}(\|\Pi_{r,\perp} k_\alpha(x,\cdot)\|_{\mathcal{H}_{k_\alpha}}$ since $\|P_{m,\perp}^{X^n} \Pi_{r,\perp} k_\alpha(x,\cdot)\|_{\mathcal{H}_{k_\alpha}} \leq \|\Pi_{r,\perp} k_\alpha(x,\cdot)\|_{\mathcal{H}_{k_\alpha}}$}.
Next, observe that 
\footnotesize
\begin{align*}
    &\mathrm{Var}(\|\Pi_{r,\perp} k_\alpha(x,\cdot)\|_{\mathcal{H}_{k_\alpha}}) \leq \mu\left(\|\Pi_{r,\perp} k_\alpha(x,\cdot)\|_{\mathcal{H}_{k_\alpha}}^2\right) =\sum_{j > r} \sigma_j.
\end{align*}
\normalsize
Then, since \small$0\leq \sqrt{(k_\alpha - \hat{k}_{\alpha,m}^n)(x,x)} \leq \sqrt{k_\alpha(x,x)} \leq 1$\normalsize~ for any $x \in \mathcal{X}$, by Bernstein's inequality, there exists a constant $c > 0$ such that for any $\delta > 0$, with probability at least $1 - \delta$, 
\footnotesize
\begin{align*}
    &\mu_{X^n}(\|\Pi_{r,\perp} k_\alpha(x,\cdot)\|_{\mathcal{H}_{k_\alpha}}) - \mu(\|\Pi_{r,\perp} k_\alpha(x,\cdot)\|_{\mathcal{H}_{k_\alpha}}) \\
    & \leq  c \log(1/\delta) \left(\sqrt{\mathrm{Var}(\|\Pi_{r,\perp}k_\alpha(x,\cdot) \|_{\mathcal{H}_{k_\alpha}})}/\sqrt{n} + 1/n\right)  \\
    & \leq  c \log(1/\delta) \left(\sqrt{\sum_{j > r} \sigma_j}/\sqrt{n} + 1/n\right)
\end{align*}
\normalsize
Since $\mu(\|\Pi_{r,\perp} k_\alpha(x,\cdot)\|_{\mathcal{H}_{k_{\alpha}}}) \leq \sqrt{\sum_{j > r}\sigma_j}$ (this follows from \eqref{eq:mu_pi_r_perp_bound}), we thus have (with probability at least $1 - \delta)$,
\footnotesize
\begin{align}
    & \ \mu_{X^n}(\|\Pi_{r,\perp} k_\alpha(x,\cdot)\|_{\mathcal{H}_{k_\alpha}}) \nonumber \\
    \leq & \   \mu(\|\Pi_{r,\perp} k_\alpha(x,\cdot)\}\|_{\mathcal{H}_{k_\alpha}}) + c \log(1/\delta) \left(\sqrt{\sum_{j > r} \sigma_j}/\sqrt{n} + 1/n\right) \nonumber \\
    \leq & \ \sqrt{\sum_{j > r}\sigma_j} + c \log(1/\delta) \left(\sqrt{\sum_{j > r} \sigma_j}/\sqrt{n} + 1/n\right).
\label{eq:mu_hat_perp_bernstein_bdd}
\end{align}
\normalsize
Thus, by plugging \eqref{eq:mu_pi_r_perp_bound}, \eqref{eq:mu_minus_mu_hat_local_rademacher_bdd} and \eqref{eq:mu_hat_perp_bernstein_bdd} into \eqref{eq:mu_P_minus_2mu_hat_decomp_3_terms}, we find that 
\footnotesize
\begin{align}
    &\mu(\|P_{m,\perp}^{X^n} k_\alpha(x,\cdot)\|_{\mathcal{H}_{k_\alpha}}) - 2 \mu_{X^n}(\|P_{m,\perp}^{X^n} k_\alpha(x,\cdot)\|_{\mathcal{H}_{k_\alpha}}) \nonumber \\
    &\leq c \log(1/\delta) \left(\sqrt{\sum_{j > r}\sigma_j} + \frac{r^2 \max\{\log(n/r^2),1\}}{n} \right). \label{eq:mu-2mu_hat_before_hayakawa_remark_1_bdd}
\end{align}
\normalsize
By picking $r \geq \lfloor (2\log n)^h/\beta^h \rfloor$ and the argument in Remark 1 in \cite{hayakawa2023sampling}, using the decay assumption on $\{\sigma_j\}$, we have that
\footnotesize
\begin{align}
    \sqrt{\sum_{j > r}\sigma_j} + \frac{r^2 \max\{\log(n/r^2),1\}}{n} = O\left(\frac{(\log n)^{2h+1}}{n} \right)
\label{eq:hayakawa_tail_eigenvalues_plus_r^2/n_bdd}
\end{align}
\normalsize
Thus, to bound $\mu(\|P_{m,\perp}^{X^n}k_\alpha(x,\cdot)\|_{\mathcal{H}_{k_\alpha}}),$ all it remains is to bound $2 \mu_{X^n}(\|P_{m,\perp}^{X^n} k_\alpha(x,\cdot)\|_{\mathcal{H}_{k_\alpha}}$. To this end, observe that 
\footnotesize
\begin{align}
& \ \mu_{X^n}(\|P_{m,\perp}^{X^n}k_\alpha(x,\cdot)\|) = \frac{1}{n} \sum_{i=1}^n \sqrt{(k_\alpha - \hat{k}_{\alpha,m}^n)(x_i,x_i)} \nonumber \\
    \leq & \  \sqrt{\frac{1}{n} \sum_{i=1}^n \sqrt{(k_\alpha - \hat{k}_{\alpha,m}^n)(x_i,x_i)}^2} = \sqrt{\frac{1}{n} \sum_{i > m} \lambda_i},
\label{eq:Khat_approx_sqrt_bdd_by_tail_evals}
\end{align}
\normalsize
where the first inequality follows from Jensen's, and the second follows from Lemma 2 in \cite{hayakawa2023sampling}.
The result then follows by plugging in \eqref{eq:hayakawa_tail_eigenvalues_plus_r^2/n_bdd} and \eqref{eq:Khat_approx_sqrt_bdd_by_tail_evals} into \eqref{eq:mu-2mu_hat_before_hayakawa_remark_1_bdd}, and using the assumption that $m \geq \lfloor (2\log n)^h/\beta^h \rfloor$ as well as the argument in Remark 1 in \cite{hayakawa2023sampling}, which by the decay assumption on the empirical eigenvalues gives us \small$\sqrt{\sum_{i > m} n^{-1} \lambda_i} = O\left( n^{-1}(\log n)^{h/2}\right).$\normalsize
\end{proof}
\begin{remark}
\label{remark:high_prob_better_than_hayakawa}
    Our result significantly strengthens the high-probability result in Corollary 1 of \cite{hayakawa2023sampling}, since we attain an $O(1/n)$ high-probability rate rather than their $O(1/\sqrt{n})$ rate. 
\end{remark}
\begin{remark}
\label{remark:high_prob_vs_expected_kernel_approx_error}
We note that by Theorem 2 in \cite{belkin2018approximation}, both the (normalized) empirical and Mercer eigenvalues ($\frac{1}{n} \lambda_j$ and $\sigma_j$ respectively) satisfy 
$\lambda_i/n \leq C' \exp(-Ci^{1/d})$ and $\sigma_i \leq C' \exp(-Ci^{1/d})$
for some measure-independent constants $C, C' > 0$, where $d$ is the dimension of $\mathcal{S}$. This justifies our assumption on eigenvalue decay in the statement of \lemref{prop:kernel_approx_main_paper}. To model cases where the decay exponent may be better than $d$, we use a general spectral decay exponent $h$.
\end{remark}
\label{appendix:proof}
\subsection{Detailed Proof for~\propref{thm:stat_error}}\label{appendix:stat_error}
For notational ease, we define $\Phi$ as the concatenation of $\phi(s, a)$ over all $(s, a) \in \mathcal{S}\times \mathcal{A}$, and define the operator $P^\pi$ as
\footnotesize
\begin{align*}
    (P^{\pi} f)(s, a) = \mathbb{E}_{(s^\prime, a^\prime) \sim P(s, a) \times \pi} f(s^\prime, a^\prime).
\end{align*}
\normalsize
We additionally define $\tilde{w}$ which satisfies the condition
\footnotesize
\begin{align*}
    \tilde{w} = & \left(\mathbb{E}_{\nu} \left[\phi(s, a) \phi(s, a)^\top\right]\right)^{-1} \\
    & \left(\mathbb{E}_{\nu}\left[\phi(s, a) \left(r(s, a) + \gamma \mathbb{E}_{(s^\prime, a^\prime) \sim P(s, a) \times \pi} \left[\phi(s^\prime, a^\prime)^\top \tilde{w}\right]\right)\right]\right),
\end{align*}
\normalsize
and let $\tilde{Q}(s, a) = \phi(s, a)^\top \tilde{w}$. It is straightforward to see that $\tilde{w}$ is the fixed point of the population (\ie, $n\to\infty$) projected least square update \eqref{eq:projected_least_square}. Furthermore, 
note that
\footnotesize
\begin{align*}
    \tilde{w} = & 
    \left(\mathbb{E}_{\nu} \left[\phi(s, a) \left(\phi(s, a) - \gamma \mathbb{E}_{(s^\prime, a^\prime) \sim P(s, a) \times \pi} \phi(s^\prime, a^\prime)\right)^{\top}\right]\right)^{-1}\nonumber \\
    & \mathbb{E}_{\nu}\left[\phi(s, a) r(s, a)\right].
\end{align*}
\normalsize
With this characterization, \revision{ letting $\phi_{sup} := \sup_{s,a}\|\phi(s,a)\|$}, we have $\|\tilde{w}\| = O\left(\revision{\phi_{\sup}}\Upsilon_2^{-1}\right)$.
We define the operator $D_{\nu}$ as
\footnotesize
\begin{align}
    \|f\|_{\nu}^2 = \langle f, D_{\nu} f\rangle,
\end{align}
\normalsize
and we omit the subscript when $\nu$ is the Lebesgue measure, and we also use $\hat{\Pi}_{\nu}$ and $\hat{P}^\pi$ to define the empirical counterpart of $\Pi_{\nu}$ and $P^\pi$.
With the update \eqref{eq:projected_least_square}, we have that
\footnotesize
\begin{align*}
    \Phi \hat{w}_{t+1} = \hat{\Pi}_{\nu}(r + \gamma \hat{P}^\pi \Phi\hat{w}_t), \quad
    \Phi \tilde{w} = \Pi_{\nu}(r + \gamma P^\pi \Phi\tilde{w}),
\end{align*}
\normalsize
which leads to
\footnotesize
\begin{align*}
    \Phi(\tilde{w} - \hat{w}_{t+1}) = & (\Pi _\nu- \hat{\Pi}_\nu) r + \gamma (\Pi_\nu P^{\pi})\Phi(\tilde{w} - \hat{w}_t) \nonumber \\
    & + \gamma(\Pi_\nu P^{\pi} - \hat{\Pi}_\nu \hat{P}^\pi) \Phi \hat{w}_t.
\end{align*}
\normalsize
With the triangle inequality, we have that
\footnotesize
\begin{align*}
    \left\|\Phi(\tilde{w} - \hat{w}_{t+1})\right\|_{\nu} \leq &  \gamma \left\|\Phi(\tilde{w} - \hat{w}_{t})\right\|_{\nu} + \left\|\left(\Pi_{\nu} - \hat{\Pi}_{\nu}\right) r\right\|_{\nu} \nonumber\\
    & + \gamma \left\|\left(\Pi_{\nu} P^{\pi} - \hat{\Pi}_{\nu} \hat{P}^\pi\right) \Phi \hat{w}_t\right\|_{\nu},
\end{align*}
\normalsize
where we use the contractivity under $\|\cdot\|_{\nu}$. Telescoping gives
\footnotesize
\begin{align*}
    \left\|\Phi(\tilde{w} - \hat{w}_{T})\right\|_{\nu} \leq & \gamma^T \left\|\Phi(\tilde{w} - \hat{w}_{0})\right\|_{\nu} \!+\! \frac{1}{1-\gamma}\left\|\left(\Pi_{\nu} - \hat{\Pi}_{\nu}\right) r\right\|_{\nu} \\
    & + \frac{\gamma}{1-\gamma} \max_{t\in [T]}\left\|\left(\Pi_{\nu} P^{\pi} - \hat{\Pi}_{\nu} \hat{P}^\pi\right) \Phi \hat{w}_t\right\|_{\nu}.
\end{align*}
\normalsize
Note that $\|\phi(s, a)\| = \revision{O(\phi_{\sup})}$, 
and $\|\tilde{w}\| = O(\revision{\phi_{\sup}}\Upsilon_2^{-1})$. 
We follow the proof of Theorem 5.1 from \cite{abbasi2019politex} to bound the second term and the third term. In particular, as the subsequent analysis makes clear, the third term dominates the second term, so we focus on bounding the third term. For notational simplicity, we denote $w := \hat{w}_t$ in this section.
Note that
\footnotesize
\begin{align*}
    & \left\|\left(\Pi_{\nu} P^{\pi} - \hat{\Pi}_{\nu} \hat{P}^{\pi}\right)\Phi w\right\|_{\nu} 
    \! =\!  \left\|D_{\nu}^{1/2}\left(\Pi_{\nu} P^{\pi} - \hat{\Pi}_{\nu} \hat{P}^{\pi}\right)\Phi w\right\|\\
    \leq & \left\|D_{\nu}^{1/2} \Pi_{\nu}\left(P^\pi - \hat{P}^\pi\right)\Phi w\right\| + \left\|D_{\nu}^{1/2} \left(\Pi_{\nu} - \hat{\Pi}_{\nu}\right)\hat{P}^\pi \Phi w\right\|
\end{align*}
\normalsize
For the first term, as $\Pi_{\nu}$ can be written as $\Pi_{\nu} = \Phi\left(\Phi^\top D_{\nu} \Phi\right)^{-1} \Phi D_{\nu}$, we have
\footnotesize
\begin{align*}
    & \left\|D_{\nu}^{1/2} \Pi_{\nu}\left(P^\pi - \hat{P}^\pi\right)\Phi w\right\|\\
    \leq & \left\|D_{\nu}^{1/2} \Phi\right\| \left\|\left(\Phi^\top D_{\nu} \Phi\right)^{-1}\right\|\left\|\Phi^\top D_{\nu} \left(P^{\pi}-\hat{P}^\pi\right)\Phi w\right\|\\
    = & O\left(\revision{\phi_{\sup}}\Upsilon_1^{-1} \left\|\Phi^\top D_{\nu} \left(P^{\pi}-\hat{P}^\pi\right)\Phi\right\|\|w\| \right).
\end{align*}
\normalsize
For the second term, we have
\footnotesize
\begin{align*}
    & \left\|D_{\nu}^{1/2}\left(\Pi_{\nu} - \hat{\Pi}_{\nu}\right)\hat{P}^\pi \Phi w\right\|\\
    = & \left\|D_{\nu}^{1/2} \Phi\left(\!\left(\Phi^\top D_{\nu}\Phi\!\right)^{-1}\Phi D_{\nu} \!-\! \left(\Phi^\top \hat{D}_{\nu} \!\Phi\right)^{-1} \Phi\hat{D}_{\nu}\right)\hat{P}^\pi \Phi w\!\right\|\\
    \leq & \left\|D_{\nu}^{1/2}\Phi \left(\Phi^\top D_{\nu} \Phi\right)^{-1} \Phi\left(D_{\nu} - \hat{D}_{\nu}\right)\hat{P}^\pi \Phi w\right\|\\
    & +\left\|D_{\nu}^{1/2} \Phi\left(\left(\Phi^\top D_{\nu} \Phi\right)^{-1} - \left(\Phi^\top \hat{D}_{\nu} \Phi\right)^{-1}\right)\Phi \hat{D}_{\nu} \hat{P}^\pi \Phi w\right\|.
\end{align*}
\normalsize
For the former term, we have that
\footnotesize
\begin{align*}
    & \left\|D_{\nu}^{1/2}\Phi \left(\Phi^\top D_{\nu} \Phi\right)^{-1} \Phi\left(D_{\nu} - \hat{D}_{\nu}\right)\hat{P}^\pi \Phi w\right\|\\
    \leq & \left\|D_{\nu}^{1/2} \Phi\right\| \left\|\left(\Phi^\top D_{\nu} \Phi\right)^{-1}\right\| \left\|\Phi(D_{\nu} - \hat{D}_{\nu}) \hat{P}^\pi \Phi w\right\|\\
    = & O\left(\revision{\phi_{\sup}} \Upsilon_1^{-1} \left\|\Phi(D_{\nu} - \hat{D}_{\nu}) \hat{P}^\pi \Phi\right\|\|w\|\right).
\end{align*}
\normalsize
For the latter term, note that
\footnotesize
\begin{align*}
    & \left(\Phi^\top D_{\nu} \Phi\right)^{\!\!-1}\!\! -\!\! \left(\!\Phi^\top \hat{D}_{\nu} \Phi\!\right)^{\!\!-\!1}\!\! =\!\!(\!\Phi^\top D_{\nu} \Phi\!)^{\!\!-\!1} \left(\!\Phi^\top \left(\!\hat{D}_{\nu} \!\!-\!\! D_{\nu}\!\right) \Phi\!\right)(\!\Phi^\top D_{\nu} \Phi\!)^{\!\!-\!1}.
\end{align*}
\normalsize
Hence, we have
\footnotesize
\begin{align*}
    & \left\|D_{\nu}^{1/2} \Phi\left(\left(\Phi^\top D_{\nu} \Phi\right)^{-1} - \left(\Phi^\top \hat{D}_{\nu} \Phi\right)^{-1}\right)\Phi \hat{D}_{\nu} \hat{P}^\pi \Phi w\right\|\\
    = & O\left( \revision{\phi_{\sup}^3}\Upsilon_1^{-1} \left\|\Phi^\top \left(\hat{D}_{\nu} - D_{\nu}\right) \Phi\right\| \left\|\left(\Phi^\top \hat{D}_{\nu} \Phi\right)^{-1}\right\|\|w\|\right)
\end{align*}
\normalsize
Now we bound the remaining terms $\left\|\Phi^\top D_{\nu} \left(P^{\pi}-\hat{P}^\pi\right)\Phi\right\|$, $\left\|\Phi(D_{\nu} - \hat{D}_{\nu}) \hat{P}^\pi \Phi\right\|$ and $\left\|\Phi^\top \left(\hat{D}_{\nu} - D_{\nu}\right) \Phi\right\|$ with standard concentration inequalities. We start with $\left\|\Phi^\top \left(\hat{D}_{\nu} - D_{\nu}\right) \Phi\right\|$. Using the matrix Azuma inequality\cite{tropp2015introduction}, 
\footnotesize
\begin{align*}
    \left\|\Phi^\top \left(\hat{D}_{\nu} - D_{\nu}\right) \Phi\right\| = \tilde{O}\left( \revision{\phi_{\sup}^2} n^{-1/2}\right).
\end{align*}
\normalsize
Hence, for sufficient large $n$, we have that $ \left\|\left(\Phi^\top \hat{D}_{\nu} \Phi\right)^{-1}\right\| = O(\Upsilon_1^{-1})$. Moreover, as $\hat{P}^\pi$ is a stochastic operator, we have
\footnotesize
\begin{align*}
    \left\|\Phi(D_{\nu} - \hat{D}_{\nu}) \hat{P}^\pi \Phi\right\| = \tilde{O}\left(\revision{\phi_{\sup}} n ^{-1/2}\right).
\end{align*}
\normalsize
Now we turn to the term $\left\|\Phi^\top D_{\nu} \left(P^{\pi}-\hat{P}^\pi\right)\Phi\right\| $. Note that
\footnotesize
\begin{align*}
    & \left\|\!\Phi^\top D_{\nu} \left(P^{\pi}\!\!-\!\!\hat{P}^\pi\right)\Phi\!\right\|\! \leq \! \left\|\Phi(D_{\nu}\! -\! \hat{D}_{\nu}) \hat{P}^\pi \Phi\right\| \!+\! \left\|\!\Phi\left(\hat{D}_{\nu} \hat{P}^\pi\! -\! D_{\nu} P^{\pi}\right) \Phi\!\right\|.
\end{align*}
\normalsize
We already have the bound for the first term. For the second term, we still apply the matrix Azuma's inequality and obtain
\footnotesize
\begin{align*}
    \left\|\Phi\left(\hat{D}_{\nu} \hat{P}^\pi - D_{\nu} P^{\pi}\right) \Phi\right\| = \tilde{O}\left(\revision{\phi_{\sup}^2} n ^{-1/2}\right).
\end{align*}
\normalsize
Thus for large enough $n$, we have $\|w\| := \|\hat{w}_t\|=  \Theta(\|\tilde{w}\|)$. Combining the previous results completes the proof.
\qed
\subsection{Detailed Proof for Lemma~\ref{lem:regret}}
\label{appendix:regret}
Note that $\|\phi(s, a)\|_2 = O(\revision{\phi_{\sup}})$ and $\|\tilde{w}\|_2 = O(\revision{\phi_{\sup}\Upsilon_2^{-1}})$. By Remark 6.7 in \cite{agarwal2021theory}, we know that $\log \pi(a|s)$ is smooth with the smoothness parameter $\beta = O(\revision{\phi_{\sup}^2})$. As a result, 
\footnotesize
\begin{align*}
    \log \frac{\pi_{k+1}(a|s)}{\pi_k(a|s)} \geq \eta \left(\nabla_{\theta} \log \pi_k(a|s)\right)^\top \hat{w}_{k,T} - \frac{\eta^2 \beta}{2}\left\|\hat{w}_{k,T}\right\|_2^2.
\end{align*}
\normalsize
With this inequality, we have
\footnotesize
\begin{align*}
    & \mathbb{E}_{s\sim d^{\tilde{\pi}}} \left[\mathrm{KL}(\tilde{\pi}(\cdot|s)\|\pi_k(\cdot|s)) -\mathrm{KL}(\tilde{\pi}(\cdot|s)\|\pi_{k+1}(\cdot|s)) \right]\\
    = & \mathbb{E}_{s\sim d^{\tilde{\pi}}, a\sim \tilde{\pi}(\cdot|s)} \left[\log\frac{\pi_{k+1}(a|s)}{\pi_k(a|s)}\right]\\
    \geq & \eta\mathbb{E}_{s\sim d^{\tilde{\pi}}, a\sim \tilde{\pi}(\cdot|s)}\left[\left(\nabla_{\theta} \log \pi_k(a|s)\right)^\top \hat{w}_{k,T}\right] - \frac{\eta^2 \beta}{2}\left\|\hat{w}_{k,T}\right\|_2^2\\
    = & \eta\mathbb{E}_{s\sim d^{\tilde{\pi}}, a\sim \tilde{\pi}(\cdot|s)}\left[A^{\pi_k}(s, a)\right] - \frac{\eta^2 \beta}{2}\left\|\hat{w}_{k,T}\right\|_2^2\\
    & + \eta \mathbb{E}_{s\sim d^{\tilde{\pi}}, a\sim \tilde{\pi}(\cdot|s)}\left[\left(\nabla_{\theta} \log \pi_k(a|s)\right)^\top \hat{w}_{k,T} - A^{\pi_k}(s, a)\right] \\
    = & (1-\gamma) \eta \left(V^{\tilde{\pi}} - V^{\pi_k}\right) - \frac{\eta^2 \beta}{2}\left\|\hat{w}_{k,T}\right\|_2^2 - \eta \mathrm{err}_k,
\end{align*}
\normalsize
where in the last step we use the performance difference lemma (Lemma 3.2 in \cite{agarwal2021theory}). Telescoping over $k$ gives
\footnotesize
\begin{align*}
    \sum_{k=0}^{K-1} \left\{V^{\tilde{\pi}} - V^{\pi_k}\right\} \leq &\frac{\mathbb{E}_{s\sim d^{\tilde{\pi}}} \left[\mathrm{KL}(\tilde{\pi}(\cdot|s)\|\pi_0(\cdot|s)) )\right]}{1-\gamma}\\
    & + \frac{\eta^2 \beta}{2}\sum_{k=0}^{K-1} \left\|\hat{w}_{k,T}\right\|_2^2 + \frac{1}{1-\gamma}\sum_{k=0}^{K-1} \mathrm{err}_k.
\end{align*}
\normalsize
Note that $\mathrm{KL}(\tilde{\pi}(\cdot|s)\|\pi_0(\cdot|s)) \leq \log |\mathcal{A}|$ and $\left\|\hat{w}_{k,T}\right\| = O\left(\|\tilde{w}\|\right) = O(\revision{\phi_{\sup}} \Upsilon_2^{-1})$, we finish the proof.
\qed
\subsection{Detailed Proof for Theorem~\ref{thm:opt_policy}}
\label{appendix:opt_policy}
With Lemma~\ref{lem:regret} and our choice of $\eta$, we only need to bound the term $\mathrm{err}_k$. We \revision{observe that}
\tiny
\begin{align*}
    & \mathrm{err}_k =  \mathbb{E}_{s\sim d^{\tilde{\pi}}, a\sim\tilde{\pi}(\cdot|s)} \left[A^{\pi_k}(s, a) - \revision{\hat{w}_{k,T}}^\top \left(\phi(s, a) - \mathbb{E}_{a^\prime\sim\pi_k(s)}\left[\phi(s, a^\prime)\right]\right)\right]
\end{align*}
\normalsize
Note that
\footnotesize
\begin{align*}
    & \mathbb{E}_{d^{\tilde{\pi}}, a\sim\tilde{\pi}(\cdot|s)} \left[A^{\pi_k}(s, a) - \revision{\hat{w}_{k,T}}^\top \left(\phi(s, a) - \mathbb{E}_{a^\prime\sim\pi_k(s)}\left[\phi(s, a^\prime)\right]\right)\right]\\
    = & \mathbb{E}_{s\sim d^{\tilde{\pi}}, a\sim \tilde{\pi}(\cdot|s)}\left[Q^{\pi_k}(s, a) - \revision{\hat{Q}_T}^{\pi_k}(s, a)\right] \\
    & - \mathbb{E}_{s\sim d^{\tilde{\pi}}, a\sim \pi_k(\cdot|s)}\left[Q^{\pi_k}(s, a) - \revision{\hat{Q}_T}^{\pi_k}(s, a)\right]\\
    \leq & \sqrt{\mathbb{E}_{s\sim d^{\tilde{\pi}}, a\sim \tilde{\pi}(\cdot|s)}\left[\left(Q^{\pi_k}(s, a) - \revision{\hat{Q}_T}^{\pi_k}(s, a)\right)^2\right]} \\
    & + \sqrt{\mathbb{E}_{s\sim d^{\tilde{\pi}}, a\sim \pi_k(\cdot|s)}\left[\left(Q^{\pi_k}(s, a) - \revision{\hat{Q}_T}^{\pi_k}(s, a)\right)^2\right]}.
\end{align*}
\normalsize
\revision{Plugging in $\pi^*$ as $\tilde{\pi}$}, combining the results in the previous subsection, we have concluded the proof. 
\qed
\revision{
\subsection{Bounding the norms of $\psi_{rf}(s,a)$ and $\psi_{nys}(s,a)$}
\label{appendix:bounding_phi_norms}
We have the following lemma which bounds the norm of $\|\psi_{rf}(s,a)\|$ and $\|\psi_{nys}(s,a) \|$.
\begin{lemma}
\label{lemma:bounding_phi_norms}
    The random features and Nystrom features constructed in Algorithm 2 satisfy the bounds
    \begin{align*}
        & \ \norm*{\psi_{rf}(s,a)} \leq \sqrt{m}\tilde{g}_\alpha \\
        & \ \|\psi_{nys}(s,a)\| \leq \sqrt{\frac{n_{nys} \tilde{g}_\alpha^2}{\lambda_m}}
    \end{align*}
    respectively, where $\lambda_m$ denotes the $m$-th largest eigenvalue of $K^{(n_{nys})}$.
\end{lemma}
\begin{proof}
Recall that 
\begin{align*}
    & \ \psi_{rf}(s,a) := \{\frac{g_\alpha(f(s,a))}{\alpha^d}\cos(\omega_i^\top f(s,a)+b_i)\}_{i=1}^m \\
    & \ \psi_{nys}(s,a) := \{\frac{g_\alpha(f(s,a))}{\alpha^d \sqrt{\lambda_i}}\sum_{\ell=1}^{n_{nys}} U_{\ell,i}k_\alpha\left(x_\ell,\frac{f(s,a)}{1-\alpha^2}\right)\}_{i=1}^m,
\end{align*}
where we recall that $U \in \bbR^{n_{nys} \times n_{nys}}$ is the $U$ that appears in the eigendecomposition of the empirical Nystrom Gram matrix $K^{(nys)} = U\Lambda U^T$. 

For random features, it is not hard to see that $\norm*{\psi_{rf}(s,a)} = O(\sqrt{m}\tilde{g}_\alpha)$, since each random feature takes the form of a randomly shifted cosine $\frac{g_\alpha(f(s,a))}{\alpha^d}$, which in our paper is upper bounded by the expression $\tilde{g}_\alpha$. 

For notational convenience, we denote
\begin{align*}
    \tilde{K}_\alpha := \begin{bmatrix}
        k_\alpha(x_1,f(s,a)) \\ 
        \dots \\ 
        k_\alpha(x_{n_{nys}}, f(s,a))
    \end{bmatrix}.
\end{align*}
Then, for Nystrom features, we note that since $0 \leq k_\alpha(\cdot,\cdot) \leq 1$ and $U$ is an orthonormal matrix, we see that
\small
\begin{align*}
    &  \ \|\psi_{nys}(s,a)\|^2 \\
    \leq & \ \tilde{g}_\alpha^2 \left(\tilde{K}_\alpha^\top U_{:,:m} \Lambda_m^{-1} U_{:,:m}^\top \tilde{K}_\alpha\right) \\
    \leq & \ \tilde{g}_\alpha^2\cdot \frac{1}{\lambda_m} \left(\tilde{K}_\alpha^\top U_{:,:m} U_{:,:m}^\top \tilde{K}_\alpha\right) \\
    \labelrel{\leq}{eq:use_U_col_indep} & \ \frac{1}{\lambda_m}\tilde{g}_\alpha^2\left(\left\|\tilde{K}_\alpha\right\|^2\right) \\
    \leq & \ n_{nys} \frac{1}{\lambda_m}\tilde{g}_\alpha^2
\end{align*}
\normalsize
where we note that $U_{:,:m}$ denotes the first $m$ columns of $U$, $\Lambda_m$ denoted the top $m$ by $m$ block of $\Lambda$,and we used independence of the columns of $U$ to derive (\ref{eq:use_U_col_indep}). Thus, for Nystrom features, we have the bound $\|\psi_{nys}(s,a)\| \leq \sqrt{\frac{n_{nys} \tilde{g}_\alpha^2}{\lambda_m}}$. 
\end{proof}
}


\revision{
\subsection{Proofs of Lemmas \ref{lemma:assumption_2_holds_rf}, \ref{lemma:assumption_2_holds_nys},\ref{lemma:we_satisfy_assumption_3}}
\label{appendix:when_assumptions_hold}
\begin{lemma}[Restatement of Lemma \ref{lemma:assumption_2_holds_rf}]
\label{lemma:assumption_2_holds_rf_appendix}
The random features $\psi_{rf}(s,a)$ are linearly independent almost surely, i.e. the probability of randomly drawing $\{\omega_i,b_i \}_{i=1}^m \sim N(0,\frac{1}{\sigma^2}I_d) \times U([0,2\pi])$ is 0.
\end{lemma}
\begin{proof}
Without loss of generality, we consider the case that $\alpha = 0$, such that the random features take the form $\psi_{rf}(s,a) = \{\cos(\omega_i^\top f(s,a) + b_i) \}_{i=1}^m$, where $\omega_i \sim N(0,\sigma^{-2} I_d)$ and $b_i \sim \mathrm{Unif}([0,2\pi])$, i.e. the random features are randomly shifted cosines; we note that the general case when $\alpha > 0$ can be easily reduced to this case since the factor $g_\alpha(f(s,a))/\alpha^d$ shared by all the features is strictly positive. Suppose the random features are not linearly independent over the space $f(\mathcal{S} \times \mathcal{A})$. Then there exist scalars $\{a_i\}_{i=1}^m$, not all of which are 0, such that 
    $$ S(x) := \sum_{i=1}^m a_i \cos(\omega_i^\top x + b_i) = 0, \quad \forall x \in f(\mathcal{S} \times \mathcal{A}).$$
    Pick some $x^0$ in the interior of $f(\mathcal{S} \times \mathcal{A})$, and denote $y_i := \cos(\omega_i^\top x^0 + b_i)$; the fact that $x^0$ lies in the interior of $f(\mathcal{S} \times \mathcal{A})$ means that all derivatives of $S(x)$ vanish at $x = x^0$. Now, for any natural number $n \leq m-1$, differentiate $S(x)$ with respect to the first coordinate $x_1$ at the point $x^0$, giving us
    \begin{align*}
        \sum_{i=1}^m a_i (\omega_i)_1^{2n} y_i = 0. 
    \end{align*}
    This then yields the following relation
    \footnotesize
    \begin{align*}
        Va = 0, \  \mbox{ where } \!V \!:=\! 
        \begingroup 
\setlength\arraycolsep{2pt}
\begin{pmatrix}
        1 &  \dots & 1 \\
        (\omega_1)_1^2 &  \dots & (\omega_m)_1^2 \\
        \vdots & \dots & \vdots \\
        \left((\omega_1)_1^{2}\right)^{m-1} & \dots & \left((\omega_m)_1^{2}\right)^{m-1}
        \end{pmatrix}, \  a := 
        \begin{pmatrix}
            a_1 \\
            \vdots \\
            a_m
\end{pmatrix}.
\endgroup
    \end{align*}
    \normalsize
    Since $V \in \bbR^{m \times m}$ is a Vandermonde matrix, its determinant is non-zero so long as the $(\omega_i)_1$'s are all distinct, which happens with probability 1 if the $\omega_i$'s are drawn iid from a Gaussian distribution. This implies then that the scalars $\{a_i\}_{i=1}^m$ must all be 0, which is a contradiction. Thus, the random features must be linearly independent over the space $f(\mathcal{S} \times \mathcal{A})$.
\end{proof}

\begin{lemma}[Restatement of Lemma \ref{lemma:assumption_2_holds_nys}]
\label{lemma:assumption_2_holds_nys_appendix}
Consider any feature dimension $0 < m \leq n_{Nys}$. Suppose the Nystrom Gram matrix $K^{(n_{Nys})}$ has rank at least $m$. Then, the Nystrom features $\psi_{nys}(s,a)$ are linearly independent.
\end{lemma}
\begin{proof}
     Without loss of generality, for simplicity we consider the case when $\alpha = 0$, in which case the Nystrom features are constructed as follows. First, let $m$ denote the feature dimension, and let $n_{nys} \geq m$ denote the number of random samples $\left\{x_1,\dots,x_{n_{nys}} \right\}$ drawn from $\mathcal{S}$ following the distribution $\mu_{Nys}$. Recall that the ($n_{nys}$ by $n_{nys}$) Gram matrix $K$ is constructed as follows: $K_{i,j} = k(x_i,x_j) := \exp\left(-\frac{-\|x_i,x_j\|^2}{2\sigma^2} \right).$ Compute the eigendecomposition $KU = \Lambda U$, where $UU^\top = U^\top U = I$, and $\Lambda$ denoting a diagonal matrix with eigenvalues $\lambda_1 \geq \dots \geq \lambda_{n_{nys}} \geq 0$. Suppose that $m$-th largest eigenvalue $\lambda_m > 0$. Then, for any $(s,a)$ pair,the $m$ Nystrom feature functions are given as 
    $$(\psi_{nys})_i(s,a) := \frac{1}{\sqrt{\lambda_i}}\sum_{\ell=1}^{n_{nys}} U_{\ell,i} k(x_\ell, f(s,a)), \quad \quad i \in [m].$$
Note that by construction, the $m$ eigenvectors $U_{:,\ell}$ are linearly independent (so long as $\lambda_m > 0)$. Suppose that the $m$ Nystrom feature functions are not linearly independent. That implies then that there exists $m$ scalars $\{a_i\}_{i=1}^m \subset \bbR$ (at least some of which are nonzero) such that for any $(s,a)$,
\begin{align*}
    & \ \sum_{i=1}^m a_i (\psi_{nys})_i(s,a) = 0 \\
    \iff & \ \sum_{i=1}^m a_i \left( \frac{1}{\sqrt{\lambda_i}}\sum_{\ell=1}^{n_{nys}} U_{\ell,i} k(x_\ell, f(s,a)) \right) = 0 \\
    \iff & \ \sum_{\ell=1}^{n_{nys}} k(x_\ell, f(s,a)) \left(\sum_{i=1}^m a_i \left( \frac{1}{\sqrt{\lambda_i}}\sum_{\ell=1}^{n_{nys}} U_{\ell,i}\right) \right) = 0.
\end{align*}

Note that since $k(\cdot,\cdot) > 0$, the above derivation implies that we must have for every $\ell \in [n_{nys}]$:
\begin{align*}
    & \ \sum_{i=1}^m a_i \left( \frac{1}{\sqrt{\lambda_i}}\sum_{\ell=1}^{n_{nys}} U_{\ell,i}\right) = 0 \implies  \ U_{:,:m} \Lambda_m^{-1/2} a = 0, 
\end{align*}
where $U_{:,:m} \in \bbR^{n_{nys} \times m}$ denotes the first $m$ columns of $U$, $\Lambda_m$ denotes the top $m$ by $m$ block of $\Lambda$, and $a := \mathrm{vec}(a_1,\dots,a_m) \in \bbR^m$. Since the columns in $U$ are linearly independent by design, and $\Lambda_m \succ 0$ (by our assumption that $\lambda_m > 0$), it follows that the vector $a \in \bbR^m$ has to be 0, which contradicts the assumption that the $m$ Nystrom features are linearly dependent. Thus, whenever $\lambda_m > 0$, it follows that the Nystrom features are linearly independent. 

\end{proof}

\begin{lemma}[Restatement of Lemma \ref{lemma:we_satisfy_assumption_3}]
\label{lemma:we_satisfy_assumption_appendix}
Suppose the features $\phi(s,a)$ are linearly independent over the interior of the set $\mathrm{supp}(\nu^\pi)$ for any policy $\pi$.  Then,  for all but finitely many $0 \leq \lambda < 1$, Assumption 3 holds.
\end{lemma}
\begin{proof}
    Fix a policy $\pi$. Suppose the matrix $\mathbb{E}_{\nu^\pi} \big[\psi(s, a)\psi(s, a)^\top \big]$ (which is positive-semidefinite) is not positive definite. Then, there exists some $0 \neq v  \in \bbR^m$ such that 
$$ v^\top\mathbb{E}_{\nu^\pi} \big[\psi(s, a)\psi(s, a)^\top \big] v = 0.$$
This implies $v^\top \psi(s,a) = 0$ for any $(s,a) \in \mathrm{int}(\mathrm{supp}(\nu^\pi))$. However, by linear independence of the features $\{\psi_i(s,a) \}_{i=1}^m$ over $\mathrm{int}(\mathrm{supp}(\nu^\pi))$, it follows that $v$ must be 0. This is a contradiction, and so the matrix $\mathbb{E}_{\nu^\pi} \big[\psi(s, a)\psi(s, a)^\top \big]$ must be positive definite. Next, we show 
$$\sigma_{\min} \left(\mathbb{E}_{\nu^\pi} \big[\psi(s, a) \big(\psi(s, a) - \gamma \mathbb{E}_{\nu^\pi}[\psi(s', a')]\big)^\top \big]\right) > 0$$ 
for all but finitely many $0 \leq \lambda < 1$.  Denote the matrix $M(\lambda) := \mathbb{E}_{\nu^\pi} \big[\psi(s, a) \big(\psi(s, a) - \gamma \mathbb{E}_{\nu^\pi}[\psi(s', a')]\big)^\top \big]$. Observe that the determinant of $M(\lambda)$ is a polynomial in $\lambda$ with degree at most $m$. Since $M(0) \neq 0$ (as we just showed that the matrix $\mathbb{E}_{\nu^\pi} \big[\psi(s, a)\psi(s, a)^\top \big]$ is positive definite), it follows that $\det(M(\lambda))$ is not the zero polynomial. Thus, it has at most $m$ roots, which shows that for all but finitely many $\lambda$, $M(\lambda)$ must be invertible, i.e. its smallest singular value is bounded away from 0.
\end{proof}
}